\documentclass[lettersize,journal]{IEEEtran}
\usepackage{amsmath,amsfonts}
\usepackage{algorithm}
\usepackage{algorithmic}
\usepackage{multirow}
\usepackage{colortbl}
\usepackage{arydshln} 
\usepackage{booktabs}
\usepackage{subcaption}
\usepackage{nicematrix} 
\usepackage{tablefootnote}

\usepackage{array}
\usepackage{textcomp}
\usepackage{stfloats}
\usepackage{url}
\usepackage{verbatim}
\usepackage{graphicx}
\usepackage{cite}
\usepackage{tabularx}
\usepackage{xcolor}
\usepackage[algo2e,linesnumbered,noend]{algorithm2e}
\usepackage{graphicx}
\usepackage{amssymb}
\usepackage{epstopdf}
\usepackage{multicol}
\usepackage{tabularx}
\usepackage{caption}
\usepackage{enumitem}
\usepackage{booktabs}
\usepackage{bm}
\usepackage{wrapfig}
\usepackage{amsmath}
\usepackage{pifont}
\usepackage{lipsum}
\usepackage{dsfont}
\usepackage{mathrsfs}
\usepackage{amsthm}
\usepackage{float}
\usepackage{microtype}
\usepackage{graphicx}
\usepackage{array} 
\usepackage{makecell}
\usepackage{booktabs} 
\usepackage{wrapfig}
\usepackage{blindtext}
\usepackage{enumitem}
\usepackage{adjustbox}
\usepackage{booktabs}
\usepackage[table]{xcolor}
\usepackage{algorithm}
\definecolor{aliceblue}{RGB}{176,223,229}

\definecolor{zzblue}{RGB}{156,183,245}

\usepackage{subcaption}

\usepackage{amsmath}
\usepackage{amssymb}
\usepackage{mathtools}
\usepackage{amsthm}
\usepackage{wrapfig}
\theoremstyle{plain}
\newtheorem{theorem}{Theorem}
\newtheorem{proposition}[theorem]{Proposition}

\theoremstyle{definition}

\theoremstyle{remark}

\newcommand{\CC}{\cellcolor{gray!25}}

\definecolor{iccvblue}{rgb}{0.21,0.49,0.74}

\makeatletter
\renewcommand{\@cite}[2]{\textcolor{blue}{[#1\if@tempswa , #2\fi]}}
\makeatother
\usepackage[T1]{fontenc}
\newlength\savewidth

\renewcommand{\paragraph}[1]{\vspace{1.25mm}\noindent\textbf{#1}}

\newcolumntype{x}[1]{>{\centering\arraybackslash}p{#1pt}}
\newcolumntype{y}[1]{>{\raggedright\arraybackslash}p{#1pt}}
\newcolumntype{z}[1]{>{\raggedleft\arraybackslash}p{#1pt}}

\newcommand{\app}{\raise.17ex\hbox{$\scriptstyle\sim$}}

\definecolor{deemph}{gray}{0.6}

\definecolor{baselinecolor}{gray}{.9}

\definecolor{zeroshotcolor}{gray}{.3}

\definecolor{LightCyan}{rgb}{0.92,1,1}

\definecolor{demphcolor}{RGB}{144,144,144}

\definecolor{LightCyan}{rgb}{0.92,1,1}
\definecolor{darkergreen}{RGB}{21, 152, 56}
\definecolor{red2}{RGB}{252, 54, 65}

\definecolor{bluebell}{rgb}{0.84, 0.84, 0.92}
\newcommand*\colorcmark[1]{%
  \expandafter\newcommand\csname #1cmark\endcsname{\textcolor{#1}{\ding{51}}}%
}
\colorcmark{green}

\newcommand*\colorxmark[1]{%
  \expandafter\newcommand\csname #1xmark\endcsname{\textcolor{#1}{\ding{55}}}%
}

\definecolor{Highlight}{HTML}{39b54a}  

\colorxmark{red}
\usepackage[colorlinks=true, linkcolor=blue, citecolor=blue]{hyperref}

\begin{document}

\title{Reflective Flow Sampling Enhancement}

\author{Zikai~Zhou$^{\star}$,
        Muyao~Wang$^{\star}$,
        Shitong~Shao,
        Lichen~Bai,
        Haoyi~Xiong,
        Bo~Han,
        Zeke~Xie$^{\dagger}$
\thanks{$^{\star}$These authors contributed equally to this work.}
\thanks{Zikai Zhou, Shitong Shao, Lichen Bai, and Zeke Xie are with The Hong Kong University of Science and Technology (Guangzhou), Guangzhou, China.}
\thanks{Muyao Wang is with The University of Tokyo, Japan.}
\thanks{Haoyi Xiong is with Microsoft, Beijing.}
\thanks{Bo Han is with Hong Kong Baptist University, Hong Kong.}
\thanks{$^{\dagger}$Correspondence to: zekexie@hkust-gz.edu.cn}}

\markboth{Journal of \LaTeX\ Class Files,~Vol.~14, No.~8, February~2026}%
{Shell \MakeLowercase{\textit{et al.}}: A Sample Article Using IEEEtran.cls for IEEE Journals}

\maketitle

\begin{abstract}
The growing demand for text-to-image generation has led to rapid advances in generative modeling. Recently, text-to-image diffusion models trained with flow matching algorithms, such as FLUX, have achieved remarkable progress and emerged as strong alternatives to conventional diffusion models. At the same time, inference-time enhancement strategies have been shown to improve the generation quality and text–prompt alignment of text-to-image diffusion models. However, these techniques are mainly applicable to conventional diffusion models and usually fail to perform well on flow models.
To bridge this gap, we propose Reflective Flow Sampling (RF-Sampling), a theoretically-grounded and training-free inference enhancement framework explicitly designed for flow models, especially for the CFG-distilled variants (i.e., models distilled from CFG guidance techniques), like FLUX. Departing from heuristic interpretations, we provide a formal derivation proving that RF-Sampling implicitly performs gradient ascent on the text-image alignment score. By leveraging a linear combination of textual representations and integrating them with flow inversion, RF-Sampling allows the model to explore noise spaces that are more consistent with the input prompt.
Extensive experiments across multiple benchmarks demonstrate that RF-Sampling consistently improves both generation quality and prompt alignment. Moreover, RF-Sampling is also the first inference enhancement method that can exhibit test-time scaling ability to some extent on FLUX.

\end{abstract}


\begin{figure*}[t]
\centering
\includegraphics[width=1.\textwidth,trim={0cm 0cm 0cm 0cm},clip]{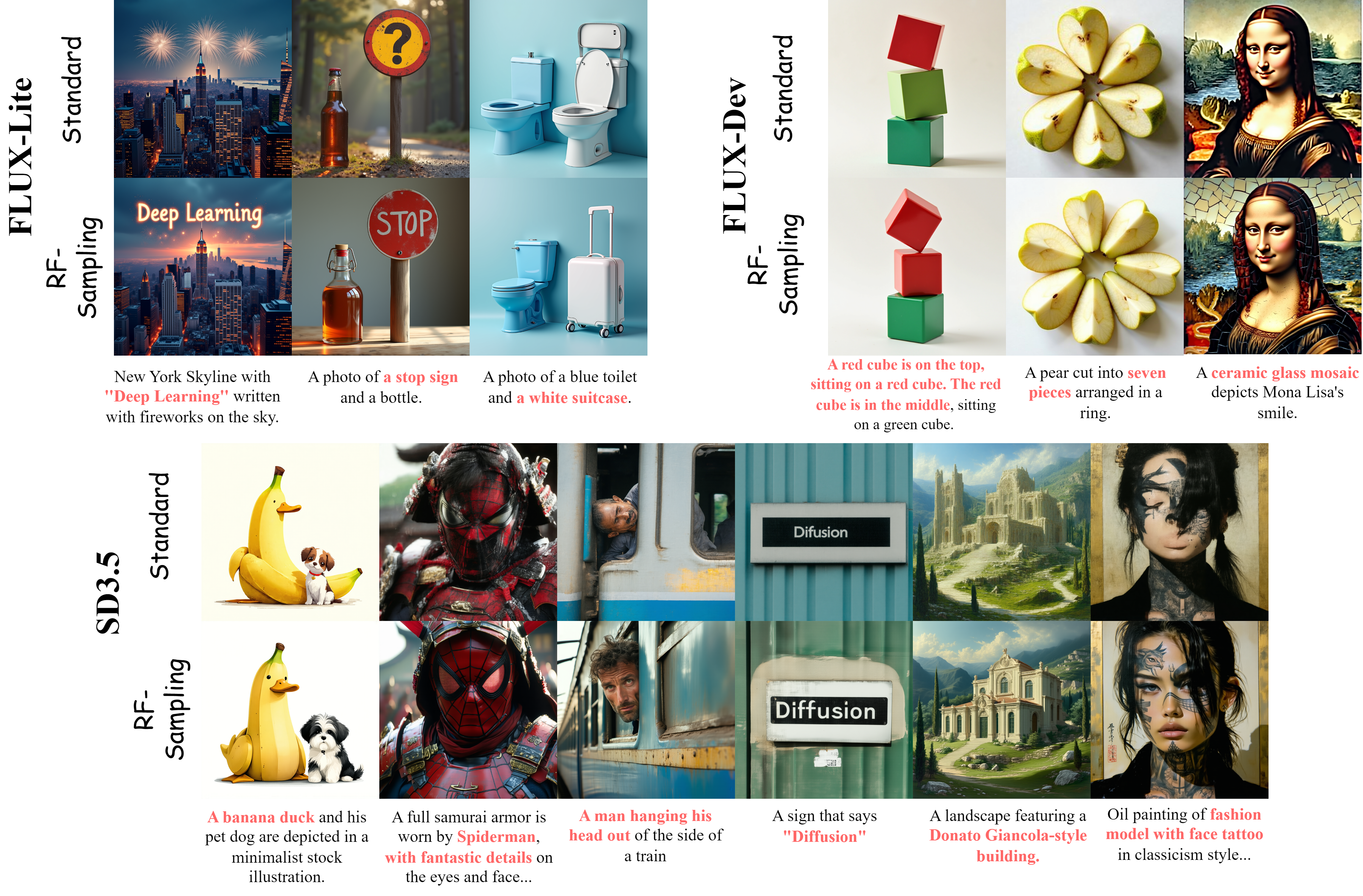}
\caption{Qualitative comparisons with three representative flow models. Images for each prompt are synthesized using the same random seed. 
More visualization results are in Appendix~\ref{sec:more_vis}.
}
\label{figure:opening}
\end{figure*}

\section{Introduction}

Text-to-image (T2I) generation has become one of the most active areas in generative modeling, driven by the growing demand for creating high-quality images from natural language prompts~\cite{Rombach_2022_CVPR,flux2024,flux1-lite,esser2024scalingrectifiedflowtransformers}. 
Recent advances in diffusion models and their training algorithms have led to remarkable progress, enabling strong performance across diverse domains~\cite{yang2023diffusion,esser2024scalingrectifiedflowtransformers,lipman2022flow,liu2022flow,ddpm_begin,zhou2025goldennoisediffusionmodels}. 
To further improve generation quality and prompt alignment, a variety of inference enhancement methods have been proposed for diffusion models~\cite{singhal2025general,ma2025inference,ho2022classifier}. 
Among them, inversion-based techniques such as Z-Sampling~\cite{lichenzigzag} exploit the discrepancy of the Classifier-Free Guidance (CFG)~\cite{ho2022classifier, xieguidance, shao2025core} parameter between denoising and DDIM inversion~\cite{ddim}. While effective, these strategies remain largely heuristic-driven and are primarily optimized for the conventional diffusion models, lacking a unified theoretical foundation to explain their behavior across different generative paradigms. 

At the same time, T2I diffusion models trained with flow matching algorithms~\cite{lipman2022flow}, such as FLUX~\cite{flux2024,flux1-lite}, have recently emerged as promising alternatives to conventional diffusion models, offering both competitive quality and efficient sampling. 
However, the unique geometric properties of flow matching, combined with the prevalence of CFG-distilled archiectures, pose significant challenges for existing inference enhancement.
To mitigate this limitation, recent work such as CFG-Zero*~\cite{fan2025cfg} has proposed optimized scaling and zero-init strategies to adapt CFG-style guidance to flow matching.
Nevertheless, the reliance on CFG-specific techniques still restricts the broader applicability of inference enhancement strategies, especially as CFG-distilled variants~\cite{meng2023distillation}, such as FLUX, continue to gain traction as efficient T2I generators.

To fill this gap, we introduce \textbf{R}eflective \textbf{F}low Sampling (RF-Sampling), a novel training-free inference enhancement framework explicitly designed for flow models that bypasses the reliance on CFG-style guidance entirely. Inspired by the key findings that rich semantic noise latent can improve the generative ability of conventional diffusion model~\cite{wang2024silent,lichenzigzag,zhou2025goldennoisediffusionmodels,po2023synthetic,bai2025weak}, our key idea is to interpolate textual representations and integrate them with flow inversion, which allows the model to explore noise spaces that are more consistent with the input prompt. We refer to such flow inversion as reflective flow. Moving beyond heuristic noise manipulation, we formulate RF-Sampling from the perspective of test-time optimization. We mathematically prove that the latent synthesized by our proposed ``High-Weight Denoising $\to$ Low-Weight Inversion'' mechanism is essentially an approximation of the gradient of the alignment score $\nabla_x \log p(c|x)$~(see formal derivation in Sec.~\ref{method} and Appendix.~\ref{apd:sec:theoretical_analysis}). Building upon this insight, RF-Sampling acts as a gradient ascent process on the latent states: it iteratively updates the trajectory towards regions with higher text-image alignment probability without requiring explicit CFG calculations or backpropagation. This theoretical foundation allows our method to function effectively even on CFG-distilled models like FLUX, where traditional guidance signal are absent or baked into the weights.


We empirically validate the effectiveness of RF-Sampling through comprehensive experiments across multiple benchmarks. Our results demonstrate that RF-Sampling consistently outperforms existing inference enhancement methods, which often struggle to generalize to flow-based architectures. As illustrated in Fig.~\ref{figure:opening}, the images synthesized by our method show noticeable improvements in aesthetic quality and semantic faithfulness. Notably, RF-Sampling is the first inference strategy to exhibit \textit{test-time scaling} properties on FLUX (see Fig.~\ref{figure:time-scaling}), where increased inference computation yields continuous gains in generation quality. Furthermore, we showcase the versatility of our framework by extending it to diverse downstream tasks, including LoRA composition, image editing, and video synthesis. The main contributions of this paper are summarized as follows:

\begin{itemize}
    \item \textbf{Novel Framework for Flow Models:} We propose RF-Sampling, a training-free inference enhancement framework tailored for flow matching models. It effectively addresses the limitations of CFG-distilled variants (e.g., FLUX), where traditional guidance methods often fail.
    
    \item \textbf{Theoretical Grounding:} Departing from heuristic approaches, we provide a rigorous theoretical derivation proving that RF-Sampling implicitly performs gradient ascent on the text-image alignment score. This offers a solid mathematical explanation for its effectiveness in navigating the flow manifold.
    
    \item \textbf{Superior Performance and Scalability:} We demonstrate that RF-Sampling achieves state-of-the-art performance on standard benchmarks and exhibits unique test-time scaling capabilities. Additionally, its robust generalization allows for seamless integration into various generative tasks, ranging from stylized image synthesis to video generation.
\end{itemize}


\begin{figure*}[t]
\centering
\includegraphics[width=.9\textwidth,trim={0cm 0cm 0cm 0cm},clip]{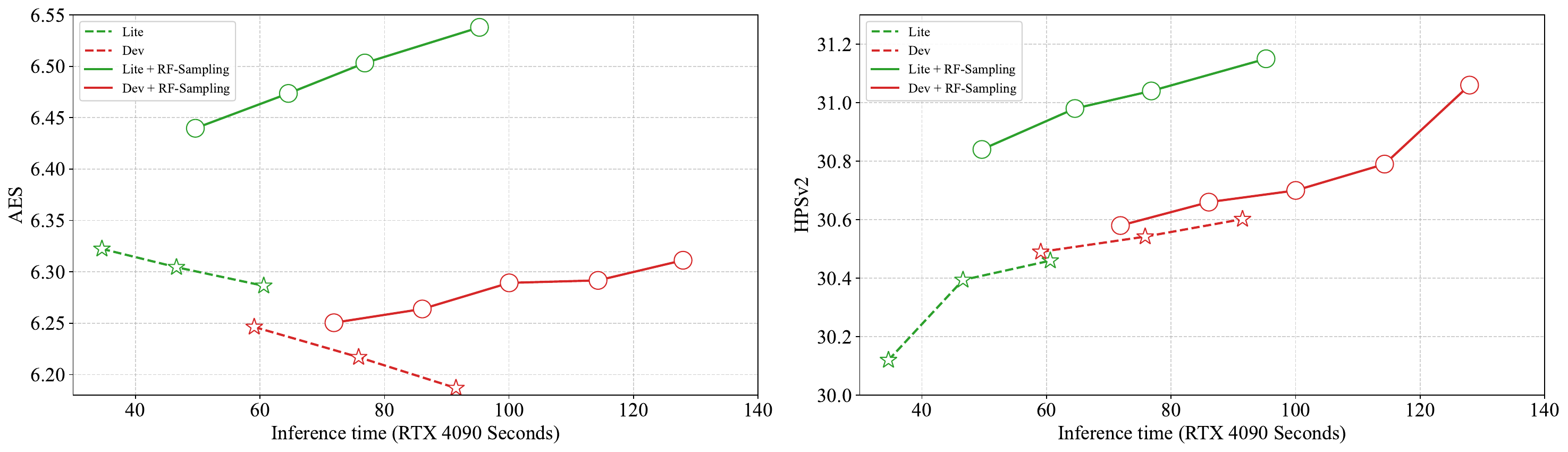}
\caption{RF-Sampling outperforms standard sampling with the same time consumption and significantly enhances the performance of FLUX-Lite and FLUX-Dev. With the increase of inference time, RF-Sampling consistently performs well, validating the scalability of our method.~(Breakdown is shown in Appendix Tab.~\ref{tab:fig2_breakdown}) }
\label{figure:time-scaling}
\end{figure*}
 
\section{Related Work}
\label{related_work}

\subsection{Text-to-Image Generation}
T2I generation is a rapidly evolving branch of generative modeling, aiming to synthesize realistic images that align with given textual descriptions. Early methods primarily relied on autoregressive models~\cite{salimans2017pixelcnn++,chen2020generative} or generative adversarial networks~\cite{Goodfellow2014GenerativeAN,mirza2014conditional}. However, in recent years, Diffusion models~\cite{ddpm_begin,Rombach_2022_CVPR} have emerged as the dominant paradigm in T2I due to their ability to generate high-quality and high-resolution images. These models generate images through a stepwise denoising process, starting from a random noise image and gradually transforming it into a meaningful image. 
In addition to conventional diffusion models, Flow Matching~\cite{lipman2022flow, liu2022flow} is an emerging diffusion model training technique that has rapidly gained traction as a strong alternative. Flow matching learns a continuous transformation that smoothly maps a simple noise distribution to the data distribution via matching the velocity. Unlike conventional diffusion models, which require multiple discrete denoising steps, flow matching models such as FLUX~\cite{flux2024,flux1-lite} can achieve efficient sampling with fewer neural function evaluations (NFEs), significantly reducing inference time while maintaining comparable, even superior generation quality to top conventional diffusion models. This efficiency advantage makes flow matching models particularly attractive for applications requiring fast generation. Our work focuses on developing dedicated inference enhancement strategies for these efficient flow models.

\subsection{CFG-distilled Guidance}

Classifier-Free Guidance~\cite{nips2021_classifier_free_guidance} has become a foundational technique in conditional diffusion models, as it improves alignment between synthesized images and text prompts by blending conditional and unconditional outputs during inference. Despite its effectiveness, CFG doubles inference cost by requiring two forward passes per denoising step.

To mitigate this inefficiency, a class of methods termed \emph{CFG-distilled}~\cite{meng2023distillation,lisnapfusion} techniques has been proposed. These methods aim to replicate the benefits of CFG using a single forward pass, thereby maintaining alignment quality while significantly reducing computational overhead. However, distillation effectively bakes the guidance into the vector field, removing the explicit unconditional branch typically used for inference. This architectural shift renders many enhancement methods inapplicable, which rely on manipulating the scale between conditional and unconditional branches, necessitating new approaches capable of recovering guidance information from the distilled vector field itself.

\subsection{Inference Enhancement for T2I Generation}
To enhance the generation quality and text alignment of conventional diffusion models, researchers have explored a range of inference enhancement strategies, which can be applied to pretrained models without requiring additional training. One key enhancement technique is Z-Sampling~\cite{lichenzigzag}, which leverages differences in the CFG parameters during the denoising process and DDIM inversion~\cite{ddim} to enhance the generation, suggesting that the noise latent space holds rich semantic information crucial for image quality. Other methods, such as~\cite{singhal2025general,ma2025inference,wang2024silent,zhou2025goldennoisediffusionmodels,po2023synthetic}, have also explored improving generation by manipulating the noise or latent space, indicating that intervention at the inference stage is an effective direction. Furthermore, in the context of Flow Matching, CFG-Zero*~\cite{fan2025cfg} mitigates the shortcomings of Flow CFG~\cite{zheng2023guided} by incorporating an optimized scale and zero-init, thereby refining the inference trajectory.
Despite the significant success of these inference enhancement strategies, they are typically tailored to the conventional diffusion models or rely on specific inference mechanisms, such as CFG technique and particular inversion algorithms. As a result, these methods cannot be directly transferred to flow models, especially when dealing with CFG-distilled variants. This limitation is particularly pressing as flow models gain increasing popularity due to their efficiency advantages, making it crucial to address this gap.

\section{Method}
\label{method}
In this section, we formulate RF-Sampling not merely as a heuristic inference trick, but as a principled optimization process~(details see Appendix.~\ref{apd:sec:theoretical_analysis}). We first introduce the preliminaries of flow matching, define the semantic parameterization, and then derive the theoretical connection between our reflective mechanism and the gradient of the text-image alignment score.

\subsection{Flow Matching Models}

Flow matching models represent a new class of generative models that synthesize images by solving an ordinary differential equation(ODE). The core idea is to train a neural network, parameterized as a vector field $v_\theta(x, t)$, to predict the flow that pushes a simple prior distribution $p_0(x)$ (e.g., standard Gaussian) to a complex target data distribution $p_1(x)$. The inference process then involves sampling a point from the prior $x_0 \sim p_0(x)$ and solving the ODE:
\begin{equation}
\frac{dx}{dt} = v_\theta(x, t),
\label{eq:ode}
\end{equation}
from $t=0$ to $t=1$ to obtain the final generated sample $x_1$. For convenience, we refer to this class of models as \emph{flow models} throughout the paper.

\begin{figure*}[t]
\centering
\includegraphics[width=1.\textwidth,trim={0cm 0cm 0cm 0cm},clip]{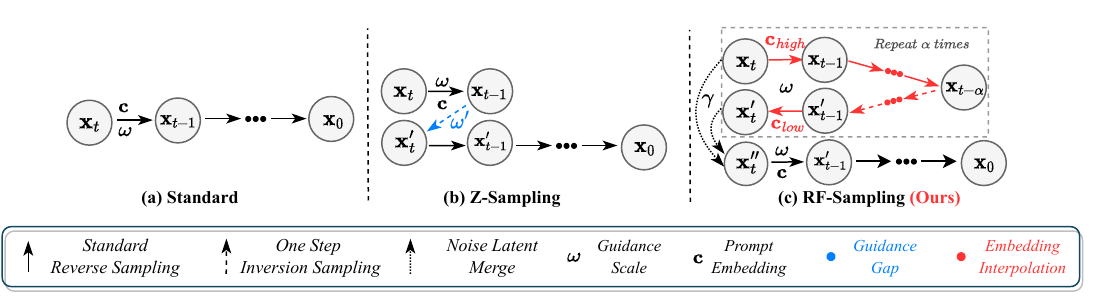}
\caption{Illustration of RF-Sampling. Compared to previous methods, RF-Sampling employs interpolation on text embeddings similar to the traditional CFG, thereby enhancing the model's generation quality and making it more suitable for flow diffusion models, especially CFG-distilled models.}
\label{figure:method}
\end{figure*}

\subsection{Parametrization in Semantic Space}

For T2I generation, the vector field is conditioned on a text embedding $c$, denoted as $v_\theta(x, t, c)$. Unlike conventional diffusion models, where CFG relies on joint training with both conditional and unconditional branches~\cite{nips2021_classifier_free_guidance,fan2025cfg}, Some flow models are typically trained only under conditional settings~\cite{flux2024,flux1-lite}. As a result, directly using CFG techniques or adopting an empty-text embedding as guidance for this kind of CFG-distilled flow models is inappropriate. To address this, we employ a linear interpolation between the conditional text embedding $c_{text}$ and an unconditional empty-text embedding $c_{uncond}$, yielding a mixed text embedding $c_{mix}$. In addition, we introduce a the amplifying weight $s$ to explicitly amplify the semantic discrepancy arising from the different text embeddings used in the denoising and inversion processes. The combination of text embedding can be described as:  
\begin{equation}
\begin{aligned}
c_{mix}(\beta) &= \beta \cdot c_{text} + (1-\beta) \cdot c_{uncond}, \\
c_w(s, \beta)&= c_{text} + {s}\cdot c_{mix}(\beta),
\label{eq:mixed_embedding}
\end{aligned}
\end{equation}
where $\beta$ is the interpolation weight directly controlling the difference between text prompt embeddings. A higher $\beta$ typically leads to a stronger alignment with the prompt. Therefore, the combination of $\beta$ and $s$ enables us to adjust the degree of text guidance throughout the inference process. 

Based on this parametrization, we define two distinct semantic states needed for our method:
\begin{itemize}
    \item \textbf{High-Weight State:} Defined by parameters $\{s_{high}, \beta_{high}\}$, yielding embedding $c_{high}$. This state imposes strong semantic alignment.
    \item \textbf{Low-Weight State:} Defined by parameters $\{s_{low}, \beta_{low}\}$, yielding embedding $c_{low}$. This state approximates the unconditional or weak-alignment flow.
\end{itemize}

\subsection{RF-Sampling as Gradient Ascent}
\label{sec:theory}
Our goal at inference time is to find a latent $x_t$ that maximizes the Alignment Score $J(x_t)$, defined as the log-posterior probability of the text condition given the noisy image latent~($J(x_t) = \log p(c|x_t)$). As established in score-based modeling theory~\cite{song2019generative,ho2022classifier}, the gradient of this score is proportional to the semantic vector field difference~(see Appendix.~\ref{apd:sec:optimzation_score}): 
\begin{equation}
    \nabla_x J(x_t) \propto v_\theta(x_t, c) - v_\theta(x_t, \emptyset),
    \label{eq:objective}
\end{equation}
where $\emptyset$ represents the null prompt. However, obtaining this gradient is challenging in CFG-distilled models~(e.g., FLUX), as they often lack an explicit unconditional branch $v_\theta(x, t, \emptyset)$.

To estimate the gradient $\nabla_x J(x_t)$ without an explicit unconditional branch, we introduce the reflective displacement vector $\Delta_{RF}$, generated by a ``High-Weight Denoising $\to$ Low-Weight Inversion'' operation over a small step $\delta t$~(see Appendix.~\ref{apd:sec:rf-sampling}):
\begin{equation}
    \Delta_{RF} = \delta t \cdot \left[v_\theta(x_t, t, c_{high}) - v_\theta(x_{t-\delta t}, t-\delta t, c_{low}) \right].
\end{equation}
Physically, this vector captures the net displacement caused by the semantic gap between the high and low guidance states. 

\begin{theorem}[\textbf{First-Order Validity}]
\label{prop:first_order}
Let the alignment score $J(x_t)$ be differentiable with gradient $\nabla_x J(x_t) \neq \mathbf{0}$. Assume the vector field $v_\theta(x, t, c)$ is locally Lipschitz continuous with respect to $x$ and differentiable with respect to $c$. Under the first-order Taylor expansion around the null prompt embedding, the reflective displacement $\Delta_{RF}$ satisfies:
\begin{equation}
    \Delta_{RF} = \mathcal{A} \cdot \delta t \cdot \nabla_x J(x_t) + \mathcal{O}(\|\mathbf{u}\|^2),
    \label{eq:thm_first_order}
\end{equation}
where $\mathcal{A} = s_{high}\beta_{high} - s_{low}\beta_{low} > 0$ is the alignment coefficient, and $\mathbf{u}$ is the semantic direction vector~($\mathbf{u} = c_{text} - c_{uncond}$). 
Furthermore, for a sufficiently small step size $\gamma > 0$, the update $x^{\prime\prime}_t = x_t + \gamma \Delta_{RF}$ satisfies the ascent inequality:
\begin{equation}
    J(x^{\prime\prime}_t) > J(x_t) \iff \langle \Delta_{RF}, \nabla_x J(x_t) \rangle > 0.
\end{equation}
Equality in Eq.~\ref{eq:thm_first_order} holds if $v_\theta$ is strictly linear with respect to the text embedding $c$.
\end{theorem}

\begin{proof}[Proof Sketch]
(Full derivation in Appendix~\ref{apd:sec:taylor_derivation}). We decompose $v_\theta(x, c_w)$ using Taylor expansion. By defining $\Delta_{RF} \triangleq \delta t [v_\theta(c_{high}) - v_\theta(c_{low})]$, the zeroth-order terms (unconditional flow) cancel out. The remaining dominant term is linear in the score gradient $\nabla_x J$, scaled by $\mathcal{A}$. Since $\mathcal{A} > 0$, the inner product $\langle \Delta_{RF}, \nabla_x J \rangle$ is positive, ensuring a strictly increasing direction.
\end{proof}

While Theorem~\ref{prop:first_order} guarantees the ascent direction, it does not constrain the magnitude. We address the optimal step size via second-order analysis.

\begin{theorem}[\textbf{Second-Order Optimality}]
\label{prop:second_order}
Assume the alignment score $J(x)$ is twice differentiable ($C^2$) and locally concave along the direction of $\Delta_{RF}$ (i.e., $\Delta_{RF}^\top \nabla^2 J(x_t) \Delta_{RF} < 0$). Let $\gamma$ be the merge ratio. The objective improvement $\Delta J(\gamma) = J(x_t + \gamma \Delta_{RF}) - J(x_t)$ is bounded by the quadratic expansion:
\begin{equation}
    \Delta J(\gamma) \approx \gamma \underbrace{\langle \Delta_{RF}, \nabla_x J(x_t) \rangle}_{G_{\text{linear}} > 0} - \frac{1}{2}\gamma^2 \underbrace{| \Delta_{RF}^\top \mathbf{H}(x_t) \Delta_{RF} |}_{L_{\text{penalty}} > 0},
    \label{eq:thm_second_order}
\end{equation}
where $\mathbf{H}(x_t) = \nabla^2_x J(x_t)$ is the Hessian matrix. The approximation becomes an equality if $J(x)$ is a quadratic function (e.g., Gaussian posterior).
Consequently, there exists a unique optimal step size $\gamma^*$ that maximizes the gain:
\begin{equation}
    \gamma^* = \frac{\langle \Delta_{RF}, \nabla_x J(x_t) \rangle}{| \Delta_{RF}^\top \mathbf{H}(x_t) \Delta_{RF} |}.
\end{equation}
\end{theorem}

\begin{proof}[Proof Sketch]
(See Appendix~\ref{apd:sec:taylor_derivation}). We apply the second-order Taylor expansion to $J(x_t + \gamma \Delta_{RF})$. The linear term $G_{\text{linear}}$ represents the gradient gain derived from Theorem~\ref{prop:first_order}. The quadratic term involves the Hessian $\mathbf{H}$. Since we maximize a log-probability, the local curvature is negative definite, acting as a penalty $L_{\text{penalty}}$. Solving $\frac{d}{d\gamma}\Delta J(\gamma)=0$ yields the closed-form optimal $\gamma^*$.
\end{proof}

\begin{table*}[th]
\begin{center}

\renewcommand\arraystretch{1.35}
\setlength{\tabcolsep}{10pt}

\caption{Main experiments on HPDv2~\cite{wu2023humanpreferencescorev2} dataset across 3 different models. The experiments show the consistent superior performance compared with previous methods, highlighting the effectiveness of our RF-Sampling. Note that other baselines are not applicable to FLUX.} 

\resizebox{1.\linewidth}{!}{\begin{tabular}{cccccccccccc}
\hline
{ }                                                                                 & { }                         & \multicolumn{2}{c}{{ Animation}}           & \multicolumn{2}{c}{{ Concept-art}}         & \multicolumn{2}{c}{{ Painting}}            & \multicolumn{2}{c}{{ Photo}}               & \multicolumn{2}{c}{Average}                                    \\ \cline{3-12} 
\multirow{-2}{*}{{ Model}}                                                          & \multirow{-2}{*}{{ Method}} & AES(↑)                         & HPSv2(↑)                      & AES(↑)                         & HPSv2(↑)                      & AES(↑)                         & HPSv2(↑)                      & AES(↑)                         & HPSv2(↑)                      & AES(↑)                         & HPSv2(↑)                      \\ \hline
{ }                                                                                 & { Standard}                 & 5.9474                         & 30.93                         & 6.1926                         & 28.59                         & 6.4161                         & 28.84                         & 5.4077                         & 27.66                         & 5.9909                         & 29.01                         \\
{ }                                                                                 & GI~\cite{kynkaanniemi2024applying}                                              & 5.9814                         & 26.23                         & \CC6.2188 & 23.48                         & 6.2188                         & 23.61                         & 5.3417                         & 23.81                         & 5.9401                         & 24.28                         \\
{ }                                                                                 & Z-Sampling~\cite{lichenzigzag}                                      & 5.8729                         & 30.58                         & 6.0427                         & 27.58                         & 6.2579                         & 28.21                         & 5.4394                         & 27.92                         & 5.9032                         & 28.57                         \\
{ }                                                                                 & CFG++~\cite{chung2024cfgmanifoldconstrainedclassifierfree}                                           & 5.8329                         & 29.81                         & 6.0969                         & 27.41                         & 6.3206                         & 27.81                         & 5.3969                         & 27.04                         & 5.9118                         & 28.02                         \\
{ }                                                                                 & CFG-Zero*~\cite{fan2025cfg}                                       & 5.9743                         & 31.22                         & 6.2066                         & 29.27                         & \CC6.4280 & 29.22                         & 5.4190                         & 27.65                         & 6.0061                         & 29.34                         \\
\multirow{-6}{*}{{ \begin{tabular}[c]{@{}c@{}}SD3.5\\ (28 steps)\end{tabular}}}     & { RF-Sampling}              & \CC6.0164 & \CC31.71 & 6.2093                         & \CC29.80 & 6.3702                         & \CC29.77 & \CC5.4973 & \CC28.51 & \CC6.0243 & \CC29.95 \\ \hline
{ }                                                                                 & { Standard}                 & 6.2635                         & 31.96                         & 6.5378                         & 30.01                         & 6.7381                         & 30.67                         & 5.8132                         & 29.04                         & 6.3381                         & 30.42                         \\
{ }                                                                                 & Z-Sampling & 6.3850 & 32.18 & 6.5162 & 30.29 & 6.7306  & 30.59 & 5.8084 & 29.21  & 6.3600 & 30.56 \\
\multirow{-2}{*}{{ \begin{tabular}[c]{@{}c@{}}FLUX-Lite\\ (28 steps)\end{tabular}}} & { RF-Sampling}              & \CC6.4350 & \CC32.78 & \CC6.6240 & \CC30.70 & \CC6.7832 & \CC30.95 & \CC5.9864 & \CC29.93 & \CC6.4572 & \CC31.09 \\ \hline
{ }                                                                                 & { Standard}                 & 6.1459                         & 32.26                         & 6.4934                         & 30.56                         & 6.4934                         & 31.27                         & 5.6515                         & 29.64                         & 6.1960                         & 30.93                         \\
{ }                                                                                 & Z-Sampling & 6.1741 & 32.24 & 6.5013 & 30.58 & 6.6510 & 31.29 & 5.6564 & 29.58 & 6.2457 & 30.92 \\
\multirow{-2}{*}{{ \begin{tabular}[c]{@{}c@{}}FLUX-Dev\\ (50 steps)\end{tabular}}}  & { RF-Sampling}              & \CC6.1866 & \CC32.40 & \CC6.5153 & \CC30.80 & \CC6.5153 & \CC31.45 & \CC5.6799 & \CC29.81 & \CC6.2243 & \CC31.12 \\ \hline
\end{tabular}
}
\label{tab:hpd}
\end{center}
\end{table*}
\begin{table*}[th]
\begin{center}

\renewcommand\arraystretch{1.3}
\setlength{\tabcolsep}{10pt}

\caption{Main experiments on Pick-a-Pic~\cite{kirstain2023pickapicopendatasetuser} and DrawBench~\cite{saharia2022photorealistic} datasets across 3 different models. Obviously, our proposed RF-Sampling exhibits superior performance across 4 different metrics. Note that other baselines are not applicable to FLUX.} 
\resizebox{1.\linewidth}{!}{\begin{tabular}{cccccccccc}
\hline
{ }                                                                                 & { }                         & \multicolumn{3}{c}{Pick-a-Pic}                                                                  & { }       & \multicolumn{3}{c}{DrawBench}                                                                   & { }       \\ \cline{3-10} 
\multirow{-2}{*}{{ Model}}                                                          & \multirow{-2}{*}{{ Method}} & PickScore(↑)                  & ImageReward(↑)                 & AES(↑)                         & HPSv2(↑)                      & PickScore(↑)                  & ImageReward(↑)                 & AES(↑)                         & HPSv2(↑)                      \\ \hline
{ }                                                                                 & { Standard}                 & 21.99                         & 85.13                          & 5.9435                         & 29.32                         & 22.60                         & 86.02                          & 5.4591                         & 27.76                         \\
{ }                                                                                 & GI                                              & 21.19                         & 28.94                          & 5.9534                         & 24.63                         & 22.11                         & 47.53                          & 5.4279                         & 23.96                         \\
{ }                                                                                 & Z-Sampling                                      & 21.73                         & 89.03                          & 5.9091                         & 28.84                         & 22.55                         & 92.05                          & 5.4784                         & 28.06                         \\
{ }                                                                                 & CFG++                                           & 21.79                         & 85.17                          & 5.8821                         & 28.50                         & 22.54                         & 81.80                          & 5.3757                         & 27.18                         \\
{ }                                                                                 & CFG-Zero*                                       & 21.88                         & 86.78                          & 5.9536                         & 29.37                         & 22.66                         & 91.90                          & 5.4511                         & 28.10                         \\
\multirow{-6}{*}{{ \begin{tabular}[c]{@{}c@{}}SD3.5\\ (28 steps)\end{tabular}}}     & { RF-Sampling}              & \CC21.99 & \CC101.50 & \CC5.9981 & \CC29.90 & \CC22.64 & \CC94.10  & \CC5.4915 & \CC28.74 \\ \hline
{ }                                                                                 & { Standard}                 & 21.91                         & 86.64                          & 6.3224                         & 30.12                         & 22.59                         & 86.51                          & 5.8969                        & 29.56                         \\
 & Z-Sampling & 21.89 & 94.37 & 6.4123 & 30.14 & 22.57 & 92.97 & 5.9102 & 29.53 \\
\multirow{-2}{*}{{ \begin{tabular}[c]{@{}c@{}}FLUX-Lite\\ (28 steps)\end{tabular}}} & { RF-Sampling}              & \CC22.05 & \CC99.21  & \CC6.5379 & \CC31.16 & \CC22.69 & \CC96.15  & \CC6.0137 & \CC29.75 \\ \hline
{ }                                                                                 & { Standard}                 & 22.06                         & 97.47                          & 6.2464                         & 30.49                         & 22.84                         & 99.73                          & 5.7775                         & 30.14                         \\
{ }                                                                                 & Z-Sampling & 21.95 & 99.89 & 6.2378 & 30.79 & 22.79 & 100.27 & 5.7793 & 30.13 \\
\multirow{-2}{*}{{ \begin{tabular}[c]{@{}c@{}}FLUX-Dev\\ (50 steps)\end{tabular}}}  & { RF-Sampling}              & \CC22.19 & \CC100.90 & \CC6.3113 & \CC31.06 & \CC22.93 & \CC106.21 & \CC5.8051 & \CC30.27 \\ \hline
\end{tabular}
}
\label{tab:pick_and_draw}
\end{center}
\end{table*}

\subsection{RF-Sampling Algorithm}

Guided by the theoretical analysis, we implement RF-Sampling as a three-stage process within each integration step of the ODE solver, as illustrated in Fig.~\ref{figure:method}.

\begin{figure}[h]
    \centering
        \centering
        \includegraphics[width=0.48\textwidth]{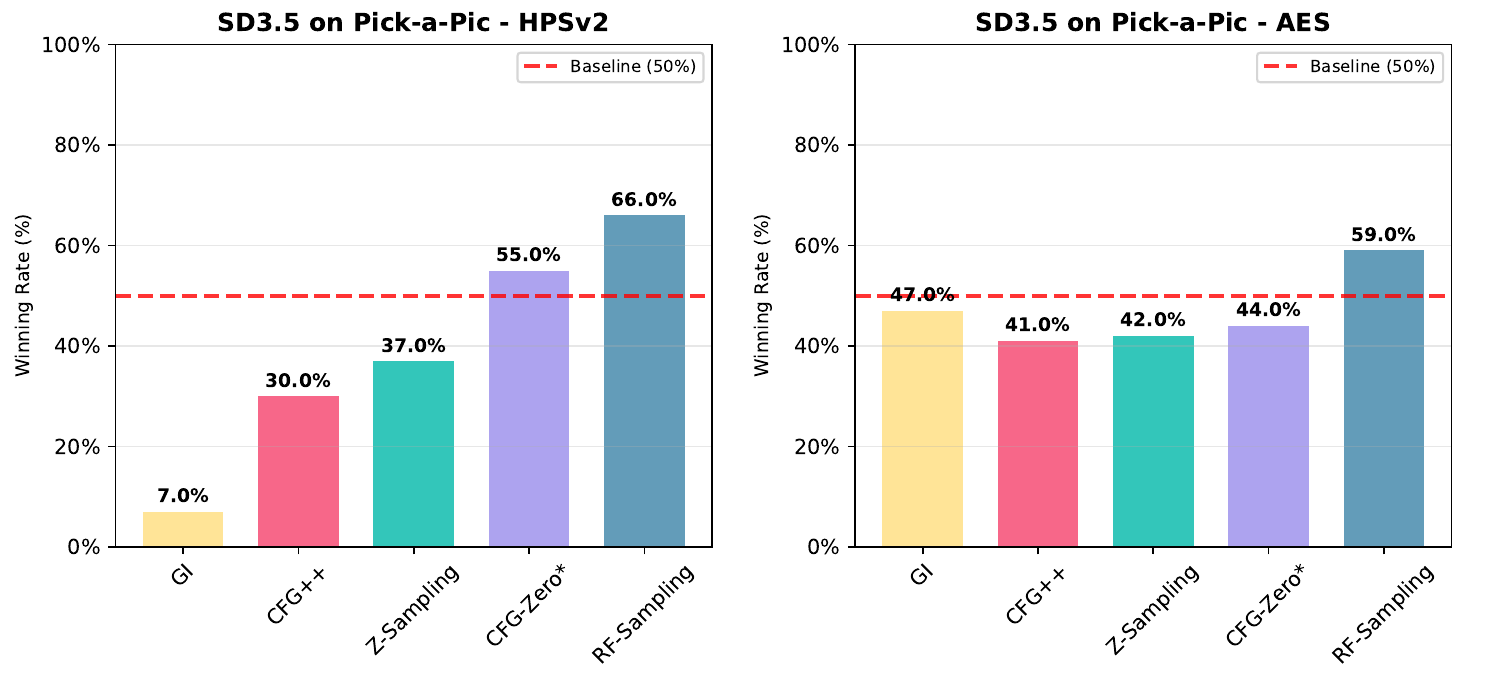}
        \caption{The winning rate of RF-Sampling over other methods on SD3.5. The standard sampling~(baseline) winning rate defaults to 50\%. The results reveal the superiority of RF-Sampling in synthesizing images with good quality.}
        \label{fig:sd3.5-main-winninf}
\end{figure}

\begin{figure}[h]
        \centering
        \includegraphics[width=0.48\textwidth]{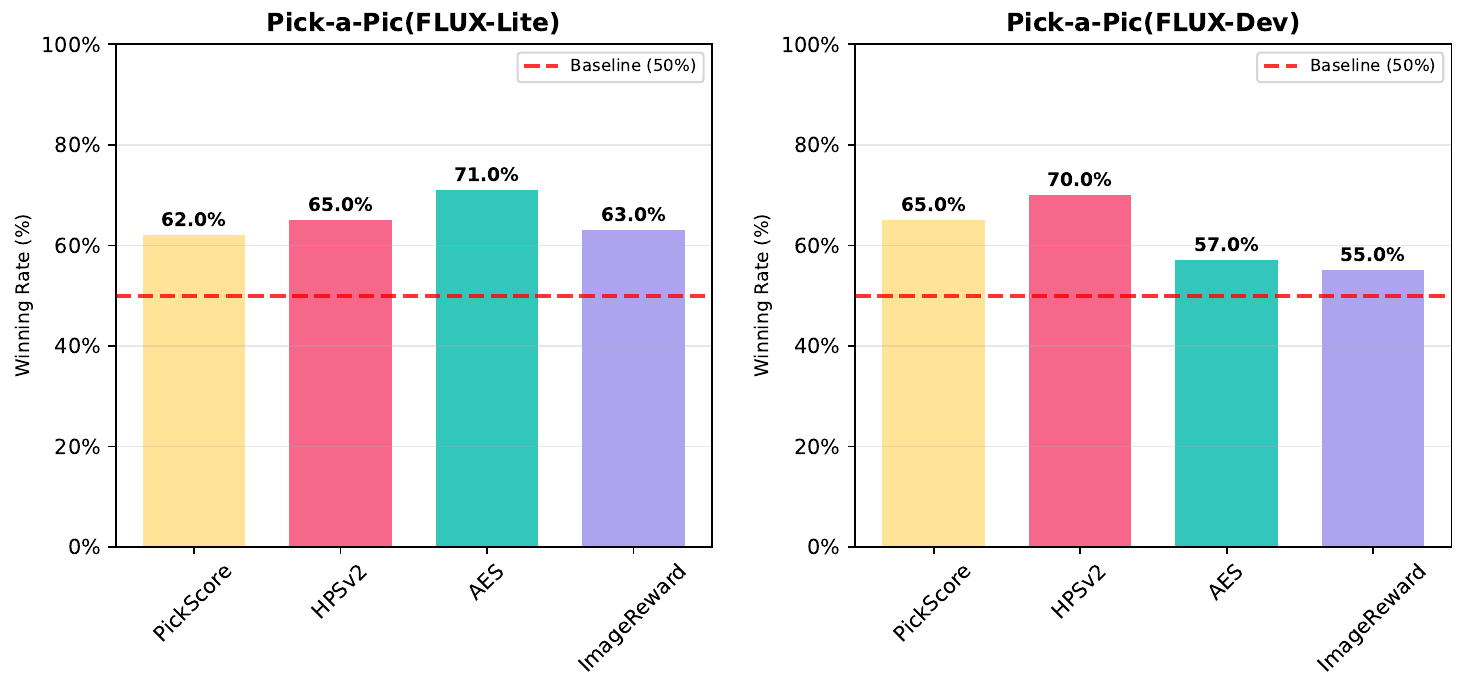}
        \caption{The winning rate of RF-Sampling over other methods on FLUX. The standard sampling~(baseline) winning rate defaults to 50\%. The results reveal the superiority of RF-Sampling in synthesizing images with good quality.}
        \label{fig:flux-main-winning}
\end{figure}

\paragraph{Stage 1: High-Weight Denoising. }
First, we perform a standard denoising step using a relatively \textbf{high interpolation weight} $\beta_{high}$ and a relatively high amplifying weight $s_{high}$ to get the mixed text embedding $c_{high}$, according to Eqn.~\ref{eq:mixed_embedding}. We then take $\alpha$ steps of the ODE solver from $t$ to $t-\alpha$ to obtain the next latent feature $x_{t-\alpha}$:
\begin{equation}
x_{t-\alpha} = x_t + \sum_{i=1}^{\alpha} v_\theta(x_{t-i+1}, t-i+1, c_{high}) \Delta t,
\label{eq:stage1}
\end{equation}
where $v_\theta$ is the conditioned vector field, $\alpha$ is the forward steps, and $\Delta t$ is the integration step size. This stage ensures a rapid and strong alignment with the given text prompt.

\paragraph{Stage 2: Low-Weight Inversion. }
Instead of directly using the newly obtained $x_{t-\alpha}$, we perform a backward-step ODE solving from $x_{t-\alpha}$. Crucially, this inversion uses a low interpolation weight $\beta_{low}$ and a relatively low amplifying weight $s_{low}$ for the mixed text embedding $c_{low}$, according to Eqn.~\ref{eq:mixed_embedding}. The corrected latent feature $x'_t$ is obtained by:
\begin{equation}
x'_t = x_{t-\alpha} - \sum_{i=1}^{\alpha} v_\theta(x_{t-\alpha+i-1}, t-\alpha+i-1, c_{low}) \Delta t,
\label{eq:stage2}
\end{equation}
where $x'_t$ is the corrected latent feature after inversion. This backward step effectively ``reflects'' the high-weight-guided latent feature back towards a more semantically centered region of the latent space. It filters out potential latent that have rich semantic information, providing a more text information starting point for the next forward step. The vector difference $x_t - x^\prime_t$ numerically represents the reflective displacement $\Delta_{RF}$.

\paragraph{Stage 3: Normal-Weight Denoising. }
With the semantically corrected feature $x'_t$, we explicitly perform the gradient ascent update and then proceed with the standard denoising. The latent is updated using the merge ratio $\gamma$ (serving as the learning rate in Prop.~\ref{prop:second_order}). Then we utilize the standard text embedding $c$ and the standard guidance scale $w$ to obtain the final latent feature for the next time step $x''_{t-1}$:
\begin{equation}
\begin{aligned}
    x''_t &= x_t + \gamma \cdot (x_t - x'_t), \quad &&\text{//Gradient Ascent}\\
    x''_{t-1} &=x''_t + v_\theta(x''_t, t, c) \Delta t, \quad &&\text{//Standard Denoising}
\end{aligned}
\label{eq:stage3}
\end{equation}
where $x''_{t-1}$ is the final latent feature for the next time step. This step ensures that the generation process continues to progress towards the target image distribution with an appropriate level of text alignment, building on the refined latent feature from the inversion stage.
 
By repeating this three-stage process for each time step, the latent $x_t$ is optimized to maximize $J(x_t)$ before moving to the next state, achieving a better high-quality and semantically coherent image synthesis. The detail process is shown in Algorithm~\ref{alg:reflective-flow-sampling}.

\begin{figure*}[th]
\centering
\includegraphics[width=1.\textwidth,trim={0cm 0cm 0cm 0cm},clip]{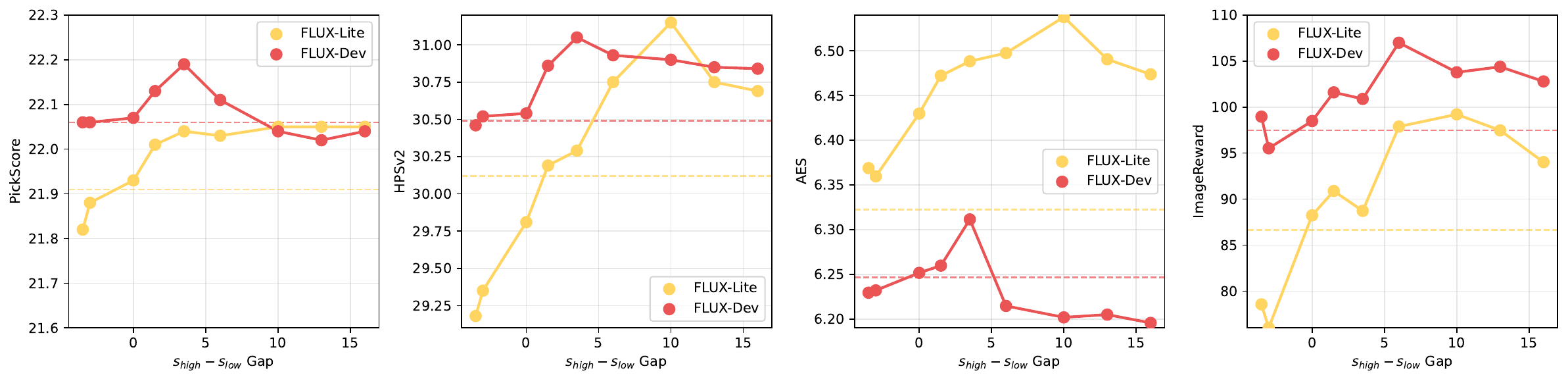}
\caption{Ablation of the gap between $s_{high}$ and $s_{low}$. When the gap of $s_{high}$ - $s_{low}$ increases within a certain range, the quality of synthesized images improves. The dotted lines represents the performance of the standard method. This indicates that within a certain range of values, RF-Sampling perform better than the standard one, demonstrating the robustness of it.}
\label{figure:cfg-gap}
\end{figure*}

\begin{figure*}[th]
\centering
\includegraphics[width=1.\textwidth,trim={0cm 0cm 0cm 0cm},clip]{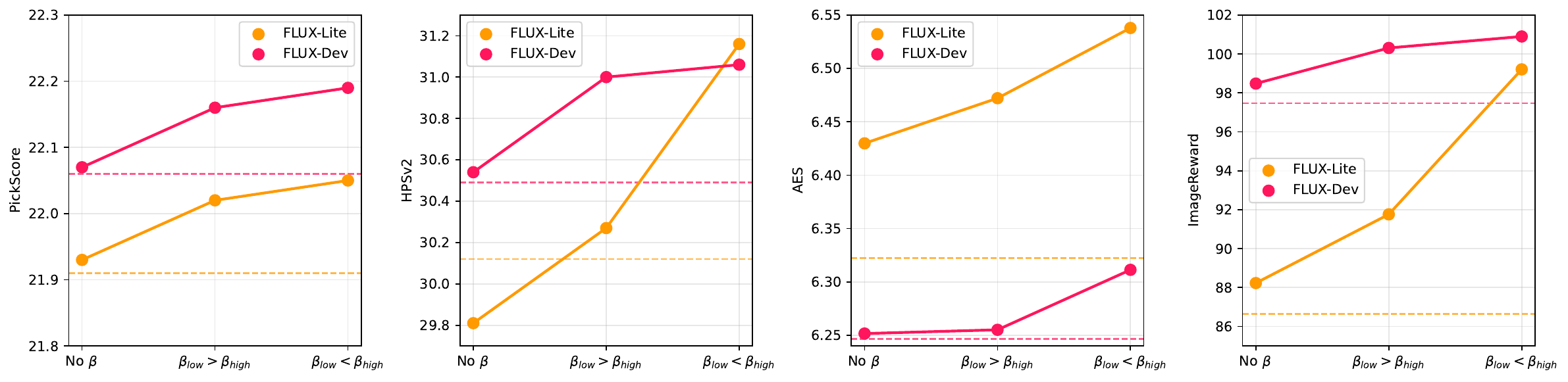}
\caption{Ablation study on the effect of $\beta_{low}$ and $\beta_{high}$. No $\beta$ means that we do not implement the interpolation weight in Eqn.~\ref{eq:mixed_embedding}. The results reveal that following the high-weight denoising $\rightarrow$ low-weight inversion paradigm can enchance the quality of synthesized images. The dotted lines represents the performance of the standard method. This indicates that within a certain range of values, RF-Sampling perform better than the standard one, demonstrating the robustness of it.}
\label{figure:neural-ratio}
\end{figure*}

\begin{algorithm}[!h]
    \caption{Reflective Flow Sampling}
    \label{alg:reflective-flow-sampling}
    \begin{algorithmic}[1]
        \REQUIRE Latent feature $x_t$, time steps $1, \dots, T$, merge ratio $\gamma$, and forward steps $\alpha$. 
        \FOR{$t = T, \dots, 1$}
            \STATE // Stage 1: High-Weight Denoising (\(\alpha\) steps forward)
            \STATE Let $\Delta t$ be the step size for each interval.
            \STATE $x_{fwd} \leftarrow x_t$
            \FOR{$i=1$ to $\alpha$}
                \STATE $c_{mix_{high}} = \beta_{high} \cdot c_{text} + (1-\beta_{high}) \cdot c_{uncond}$
                \STATE $c_{high} = c_{text} + s_{high} \cdot c_{mix_{high}}$
                \STATE $x_{fwd} \leftarrow x_{fwd} + v_\theta(x_{fwd}, t-(i-1), c_{high}) \Delta t$
            \ENDFOR
            \STATE $x_{t-\alpha} \leftarrow x_{fwd}$
            
            \STATE // Stage 2: Low-Weight Inversion (\(\alpha\) steps inversion)
            \STATE $x_{inv} \leftarrow x_{t-\alpha}$
            \FOR{$i=1$ to $\alpha$}
                \STATE $c_{mix_{low}} = \beta_{low} \cdot c_{text} + (1-\beta_{low} ) \cdot c_{uncond}$
                \STATE $c_{low} = c_{text} + s_{low} \cdot c_{mix_{low}}$
                \STATE $x_{inv} \leftarrow x_{inv} - v_\theta(x_{inv}, t-\alpha+(i-1), c_{low}) \Delta t$
            \ENDFOR
            \STATE $x'_t \leftarrow x_{inv}$
            
            \STATE // Stage 3: Normal-Weight Denoising (1 step forward)
            \STATE $x''_t \leftarrow x_t + \gamma(x_t - x'_t)$  \quad // \textit{Gradient Ascent}
            \STATE $x''_{t-1} \leftarrow x''_t + v_\theta(x''_t, t, c) \Delta t$ \quad // \textit{Standard Denoising}
            \STATE $x_{t-1} \leftarrow x''_{t-1}$
        \ENDFOR
        \STATE \textbf{return} $x_0$
    \end{algorithmic}
\end{algorithm}

\section{Experiment}
\label{experiment}


\subsection{Experiment Setting}
We conduct a comprehensive evaluation of several T2I diffusion models. For a detailed description of the benchmarks, evaluation metrics, model architectures and hyperparameter settings, please refer to Appendix~\ref{apd:benchmark_and_dataset}. Below is a summary of our experimental setup.

\paragraph{Benchmarks.}
Our evaluation leverages several established benchmarks to assess a wide range of capabilities. For human preference alignment, we use Pick-a-Pic~\cite{kirstain2023pickapicopendatasetuser} and HPD v2~\cite{wu2023humanpreferencescorev2}. To evaluate compositional reasoning, we employ DrawBench~\cite{saharia2022photorealistic}, GenEval~\cite{ghosh2023genevalobjectfocusedframeworkevaluating}, and T2I-Compbench~\cite{huang2023t2icompbench}. For text-to-video (T2V) and in-context image generation, we utilize ChronoMagic-Bench-150~\cite{yuan2024chronomagic} and FLUX-Kontext-Bench~\cite{labs2025flux1kontextflowmatching}, respectively.

\paragraph{Evaluation Metrics.}
To quantify model performance, we utilize several metrics designed to reflect human perception. These include PickScore~\cite{kirstain2023pickapicopendatasetuser}, HPS v2~\cite{wu2023humanpreferencescorev2}, and ImageReward~\cite{xu2023imagerewardlearningevaluatinghuman} for measuring alignment with human preferences, and the Aesthetic Score (AES)~\cite{AES} for assessing visual appeal. For T2V evaluation on ChronoMagic-Bench-150, we use UMT-FVD, UMTScore, GPT4o-MTScore, and MTScore.

\paragraph{Flow Models.}
Our analysis focuses on five state-of-the-art flow models. For T2I generation, we evaluate FLUX-Dev~\cite{flux2024}, its lightweight variant FLUX-Lite~\cite{flux1-lite}, and StableDiffusion-3.5~\cite{esser2024scalingrectifiedflowtransformers}. For T2V generation, we use Wan2.1-T2V-1.3B~\cite{wan2025}, and for in-context image editing, we evaluate FLUX-Kontext~\cite{labs2025flux1kontextflowmatching}.

\begin{table}[h]
\begin{center}

\renewcommand\arraystretch{1.0}
\setlength{\tabcolsep}{10pt}

\caption{Comparison of FID and IS between standard sampling and RF-Sampling on ImageNet-1K. We use FLUX-Lite with inference steps 28, combining the nunchaku~\cite{li2024svdquant} sampling acceleration framework, to generate 5,000 samples~(5 images per class). RF-Sampling achieves a lower FID and a higher IS than the standard one, demonstrating its ability to better align with the real data distribution while maintaining high-quality and diverse image generation.}
\resizebox{.9\linewidth}{!}{\begin{tabular}{ccc}
\hline
Method      & FID~($\downarrow$)                           & IS~($\uparrow$)                               \\ \hline
Standard    & 35.08                         & 150.07                         \\
RF-Sampling & \CC33.12 & \CC155.21 \\ \hline

\end{tabular}
}
\label{tab:imagenet_5k}
\end{center}
\end{table}

\begin{figure}[h]
\centering
\includegraphics[height=0.35\textwidth,trim={0cm 0cm 0cm 0cm},clip]{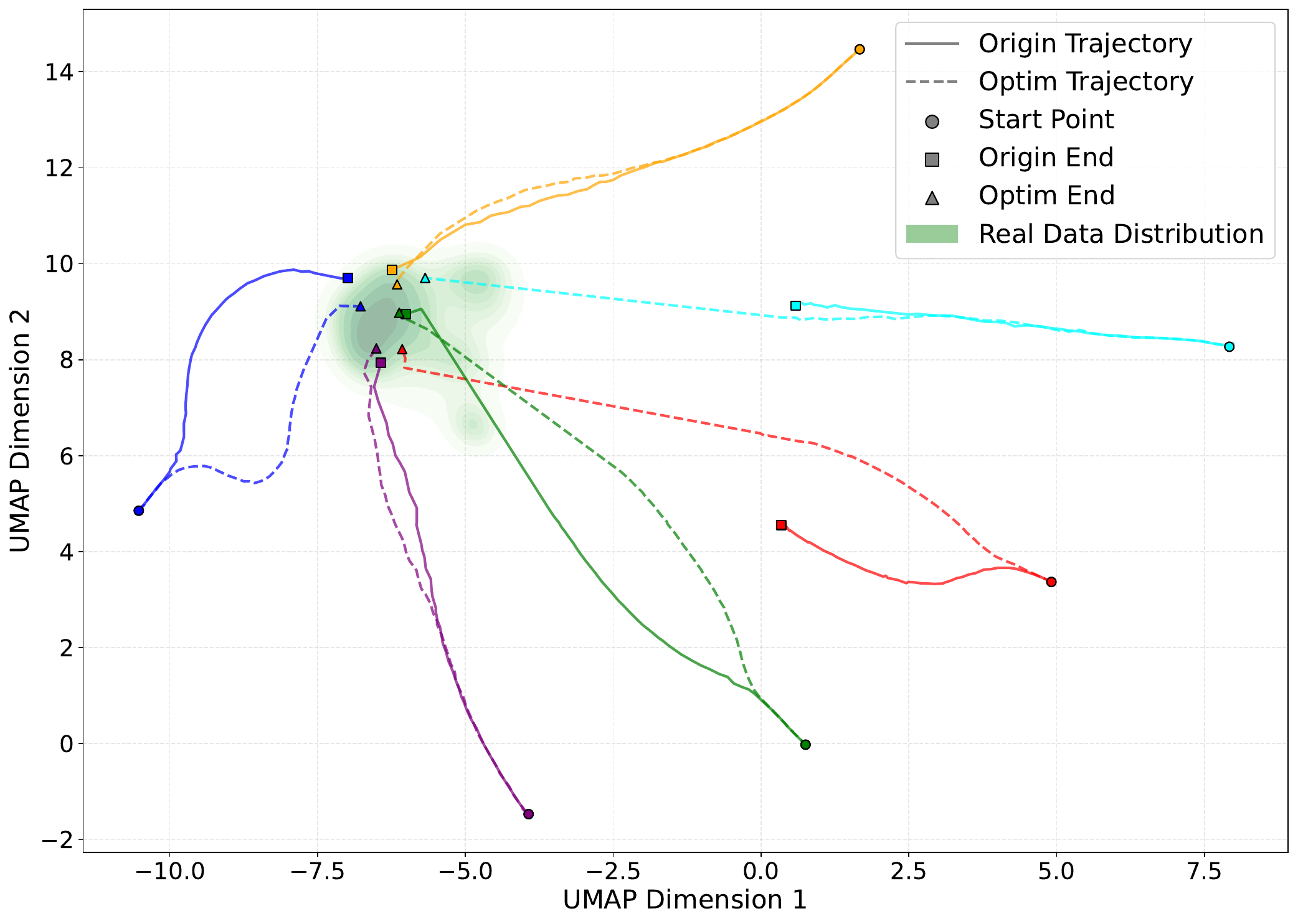}
\caption{Visualization of the sampling trajectories sampled by our method and the standard approach.}
\label{trajector-main}
\end{figure}

\begin{figure*}[h]
\centering
\includegraphics[width=1.\textwidth,trim={0cm 0cm 0cm 0cm},clip]{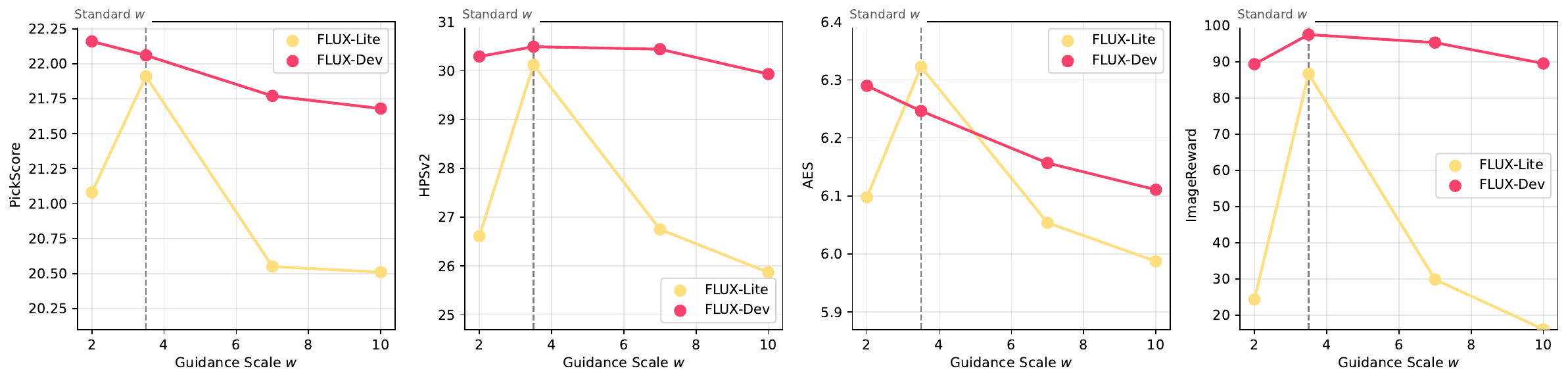}
\caption{Ablation study of standard guidance scale $w$. With the increase of the standard guidance scale $w$, the quality of the synthesized images degrades a lot. The results guarantee that the performance gain of RF-Sampling is not introduced by the increase of the amplifying weight $s$.}
\label{figure:cfg-effect}
\end{figure*}

\begin{figure*}[h]
\centering
\includegraphics[width=1.\textwidth,trim={0cm 0cm 0cm 0cm},clip]{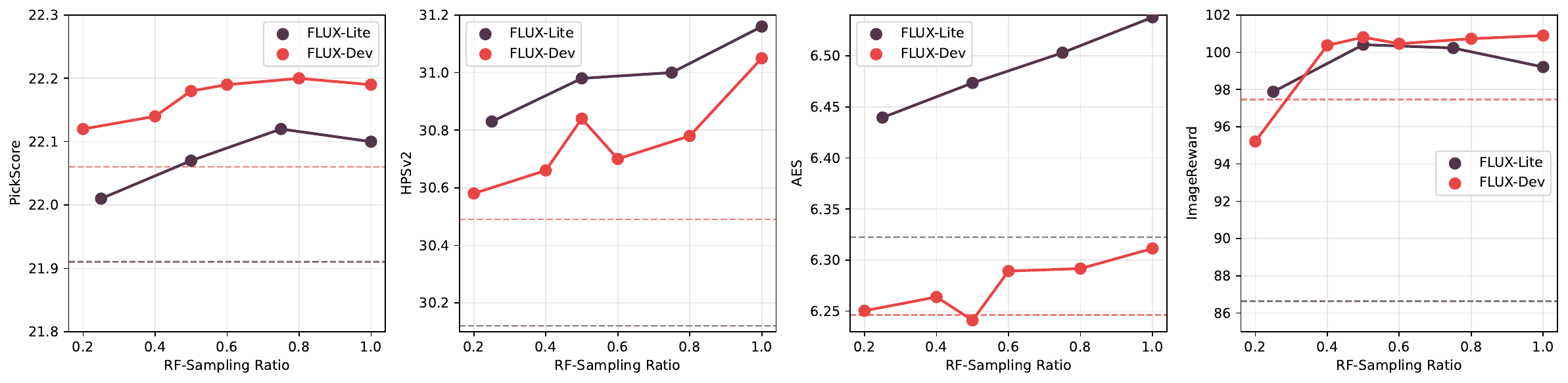}
\caption{Robustness to the RF-Sampling steps. The horizontal axis shows the ratio of RF-Sampling operations during the whole inference steps. As the ratio increases, generation quality improves, indicating effective semantic information gain throughout the whole path. The dotted lines represents the performance of the standard method. This indicates that within a certain range of values, RF-Sampling perform better than the standard one, demonstrating the robustness of it.}
\label{figure:step-operation}
\end{figure*}

\begin{figure*}[h]
\centering
\includegraphics[width=1.\textwidth,trim={0cm 0cm 0cm 0cm},clip]{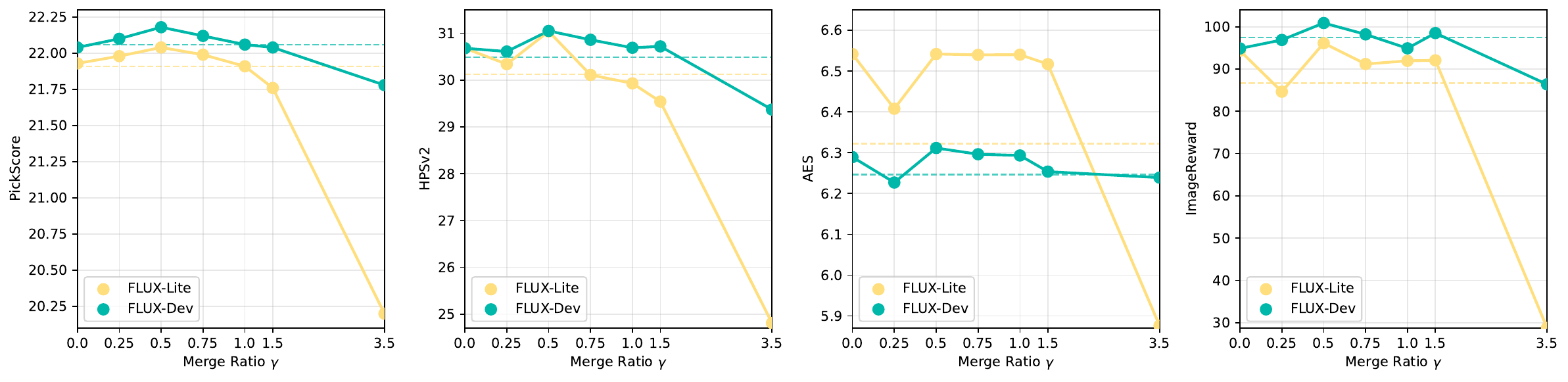}
\caption{We explore the influence of merge ratio $\gamma$ on Pick-a-Pic dataset. The results across 4 metrics reveal that $\gamma = 0.5$ is a better choice, where the synthesized images are the best. The dotted lines represents the performance of the standard method. This indicates that within a certain range of values, RF-Sampling perform better than the standard one, demonstrating the robustness of it.}
\label{figure:merge-ratio}
\end{figure*}

\subsection{Main Experiment}
To validate the effectiveness of our method, we conduct evaluations using multiple human preference models that score the images generated by our approach. The results in Tab.~\ref{tab:hpd} and Tab.~\ref{tab:pick_and_draw} prove that our method consistently achieves top-1 performance across most metrics. Additionally, Tab.~\ref{tab:random_seed} in the appendix demonstrates the robustness of our method, showing that RF-Sampling consistently outperforms the standard method across varying random seeds. We further report preference-winning rate experiments among different human preference models in Fig.~\ref{fig:sd3.5-main-winninf} and Fig.~\ref{fig:flux-main-winning}, where our method achieves up to $70\%$ winning rate under certain expert preferences. Moreover, we evaluate our method on the T2I and GenEval benchmarks to demonstrate its effectiveness. The corresponding results are provided in the appendix, as shown in Tab.~\ref{tab:t2i} and Tab.~\ref{tab:geneval_full}. To highlight the advantages of our approach, we provide qualitative visualizations in Fig.~\ref{figure:opening}, with additional synthesized examples in Appendix Sec.~\ref{sec:more_vis}. These visualizations further highlight the enhanced inference capability of our method.

\subsection{Ablation Studies and Additional Analysis.}

To better highlight the characteristics of our method, we conducted extensive quantitative and qualitative experiments, as presented below. More results are provided in Appendix~\ref{sec:add_analysis}.

\paragraph{High denoising and low inversion.} 
To validate the rationale behind the choice of the interpolation parameter $\beta$, we conduct experiments with different settings of $\beta$. The results, shown in Fig.~\ref{figure:neural-ratio}, confirm the effectiveness of interpolation and justify assigning higher weights to the forward process while using lower weights for the inverse process. As a complement, to provide a more intuitive understanding of the effect of varying $\beta$, we present the corresponding visualizations in Fig.~\ref{figure:high-denoise} and Fig.~\ref{figure:low-inversion}.
In addition, to examine the effectiveness of parameter $s$ in amplifying the semantic gap, we perform experiments as illustrated in Fig.~\ref{figure:cfg-gap}. The results indicate that an appropriately larger gap can better guide the model to generate high-quality images. Besides, these experiments validate the theoretical analysis Eq.~\ref{eq:thm_first_order}. 

\begin{figure*}[!t]
\centering
\includegraphics[width=.9 \textwidth,trim={0cm 0cm 0cm 0cm},clip]{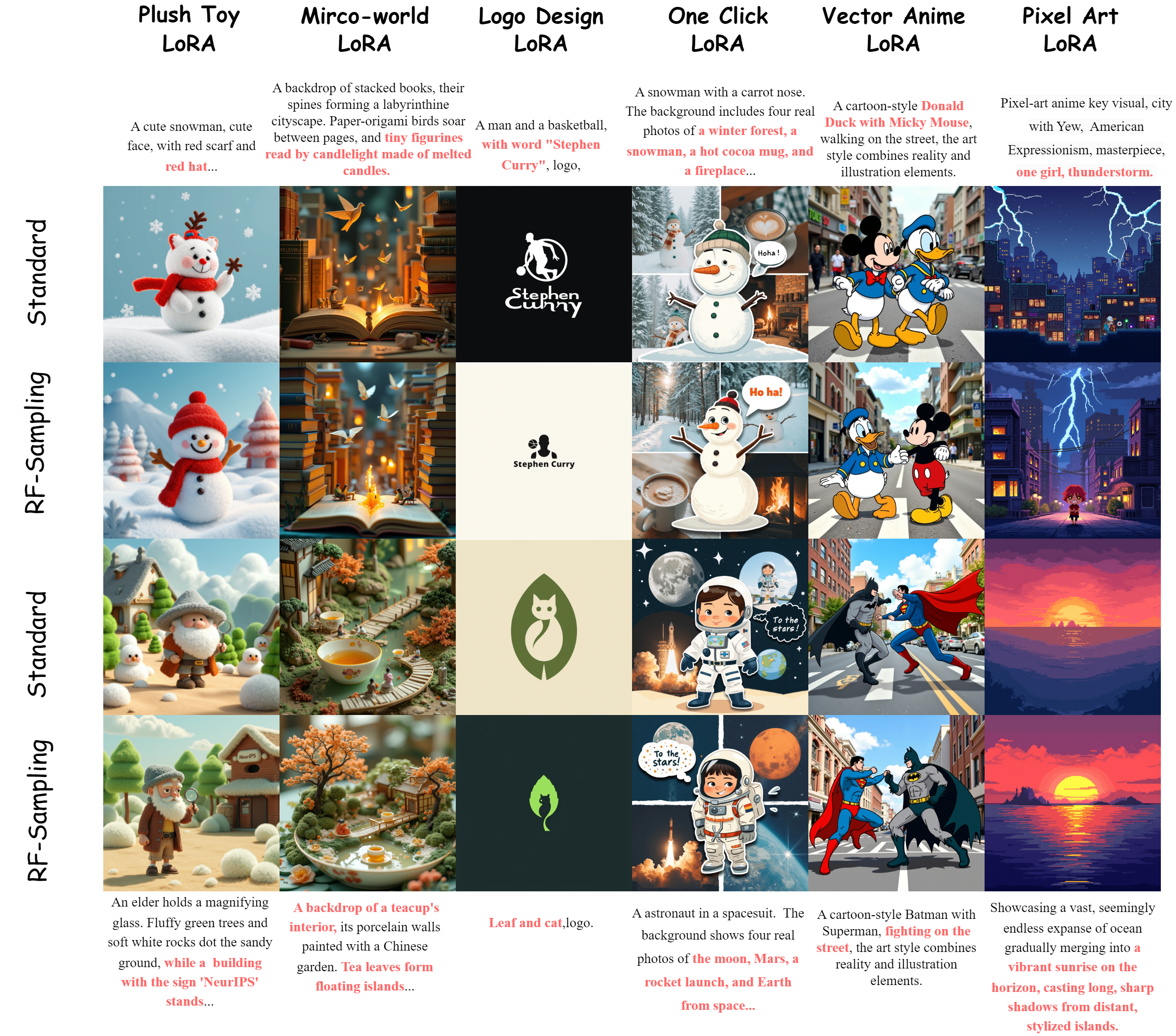}
\caption{We combine our proposed methods with existing LoRAs in FLUX community. Our RF-Sampling can be directly applied to the corresponding downstream tasks. The synthesized images validate the effectiveness and generalizability of our method.}
\label{figure:lora-main}
\end{figure*}

\paragraph{Optimal Steps.} 
To further validate the contribution of our method at each inference step, we evaluate the proportion of steps performing reflection relative to the total number of inference steps. The results, presented in Fig.~\ref{figure:step-operation}, demonstrate that, in general, increasing the number of reflection-enhanced steps leads to higher image generation quality.

\begin{table*}[h]
\begin{center}

\renewcommand\arraystretch{1.25}
\setlength{\tabcolsep}{10pt}

\caption{Orthogonal experiments with Nunchaku~\cite{li2024svdquant}, a sampling acceleration method for FLUX. The results demonstrate the generalizability of RF-Sampling to sampling acceleration. } 
\resizebox{.9\linewidth}{!}{\begin{tabular}{cccccc}
\hline
Model                                                                                                   & Method                                     & PickScore(↑)                  & ImageReward(↑)                 & AES(↑)                         & HPSv2(↑)                      \\ \hline
{ }                                                                                 & { Standard}            & 21.91                         & 86.64                          & 6.3224                         & 30.12                         \\
{ }                                                                                 & RF-Sampling                                & \CC22.05                         & \CC99.21                          & \CC6.5379 & \CC31.16 \\
\cline{2-6}
{ }                                                                                 & { Standard + Nunchaku} & 22.07                         & 95.94                          & 6.2303                         & 30.47                         \\ 
\multirow{-4}{*}{{ \begin{tabular}[c]{@{}c@{}}FLUX-Lite\\ (28 steps)\end{tabular}}} & RF-Sampling + Nunchaku                     & \CC22.23 & \CC102.35 & \CC6.4171                         & \CC30.86                         \\ \hline
{ }                                                                                 & { Standard}            & 22.06                         & 97.47                          & 6.2464                         & 30.49                         \\
{ }                                                                                 & RF-Sampling                                & \CC22.19                         & \CC100.90                         & \CC6.3113 & \CC31.06 \\
\cline{2-6}
{ }                                                                                 & { Standard + Nunchaku} & 22.18                         & 102.23                         & 6.2203                         & 30.73                         \\ 
\multirow{-4}{*}{{ \begin{tabular}[c]{@{}c@{}}FLUX-Dev\\ (50 steps)\end{tabular}}}  & RF-Sampling + Nunchaku                     & \CC22.22 & \CC107.46 & \CC6.2672                         & \CC30.90   \\ \hline
\end{tabular}
}
\label{tab:nunchaku}
\end{center}
\end{table*}
\begin{table*}[h]
\begin{center}

\renewcommand\arraystretch{1.1}
\setlength{\tabcolsep}{10pt}

\caption{We conduct ablation studies of the value of $\alpha$ and the form of $c_{uncond}$ on Pick-a-Pic dataset. It is noticed that as $\alpha$ increases, the quality of synthesized images improves, albeit at an inevitable cost to computational time. The form of $c_{uncond}$ may be either the null prompt embedding $\varnothing$ or the zero-padded prompt embedding $\mathbf{0}$. In relation to the form of $c_{uncond}$, the null prompt embedding $\varnothing$ yields superior results, as it carries a greater amount of unconditional semantic information.} 
\resizebox{0.9\linewidth}{!}{\begin{tabular}{ccccccc}
\hline
Model                                                                                                   & \multicolumn{2}{c}{Method}                                & PickScore(↑)                  & ImageReward(↑)                & AES(↑)                         & HPSv2(↑)                       \\ \hline
{\color[HTML]{000000} }                                                                                 & {\color[HTML]{000000} }                        & $\mathbf{0}$   & 21.95 & 88.58                         & 6.4346                         & 29.90                          \\
{\color[HTML]{000000} }                                                                                 & \multirow{-2}{*}{{\color[HTML]{000000} $\alpha = 1$}} & $\varnothing$ & 21.95 & 87.64                         & 6.4576                         & 30.05                          \\ \cline{2-3}
{\color[HTML]{000000} }                                                                                 &                                                & $\mathbf{0}$   & 21.88                         & 91.86                         & 6.4667                         &  29.61                      \\
\multirow{-4}{*}{{\color[HTML]{000000} \begin{tabular}[c]{@{}c@{}}FLUX-Lite\\ (28 steps)\end{tabular}}} & \multirow{-2}{*}{$\alpha = 2$}                        & $\varnothing$ & \CC22.05                         & \CC99.21 & \CC6.5379 & \CC31.16  \\ \hline
{\color[HTML]{000000} }                                                                                 & {\color[HTML]{000000} }                        & $\mathbf{0}$   & 22.11                         & 98.82                         & 6.2566                         &  30.00                     \\
{\color[HTML]{000000} }                                                                                 & \multirow{-2}{*}{{\color[HTML]{000000} $\alpha = 1$}} & $\varnothing$ & \CC22.19 & 100.90                        & 6.3113                         & \CC31.06                           \\ \cline{2-3}
{\color[HTML]{000000} }                                                                                 &                                                & $\mathbf{0}$   & 22.08                         & \CC101.19                         & 6.3168                         & 30.74 \\
\multirow{-4}{*}{{\color[HTML]{000000} \begin{tabular}[c]{@{}c@{}}FLUX-Dev\\ (50 steps)\end{tabular}}}  & \multirow{-2}{*}{$\alpha = 2$}                        & $\varnothing$ & 22.14                         & 100.52 & \CC6.3342 & 30.88 \\ \hline
\end{tabular}
}
\label{tab:null_and_alpha}
\end{center}
\end{table*}

\begin{figure*}[h]
\centering
\includegraphics[width=.9\textwidth,trim={0cm 0cm 0cm 0cm},clip]{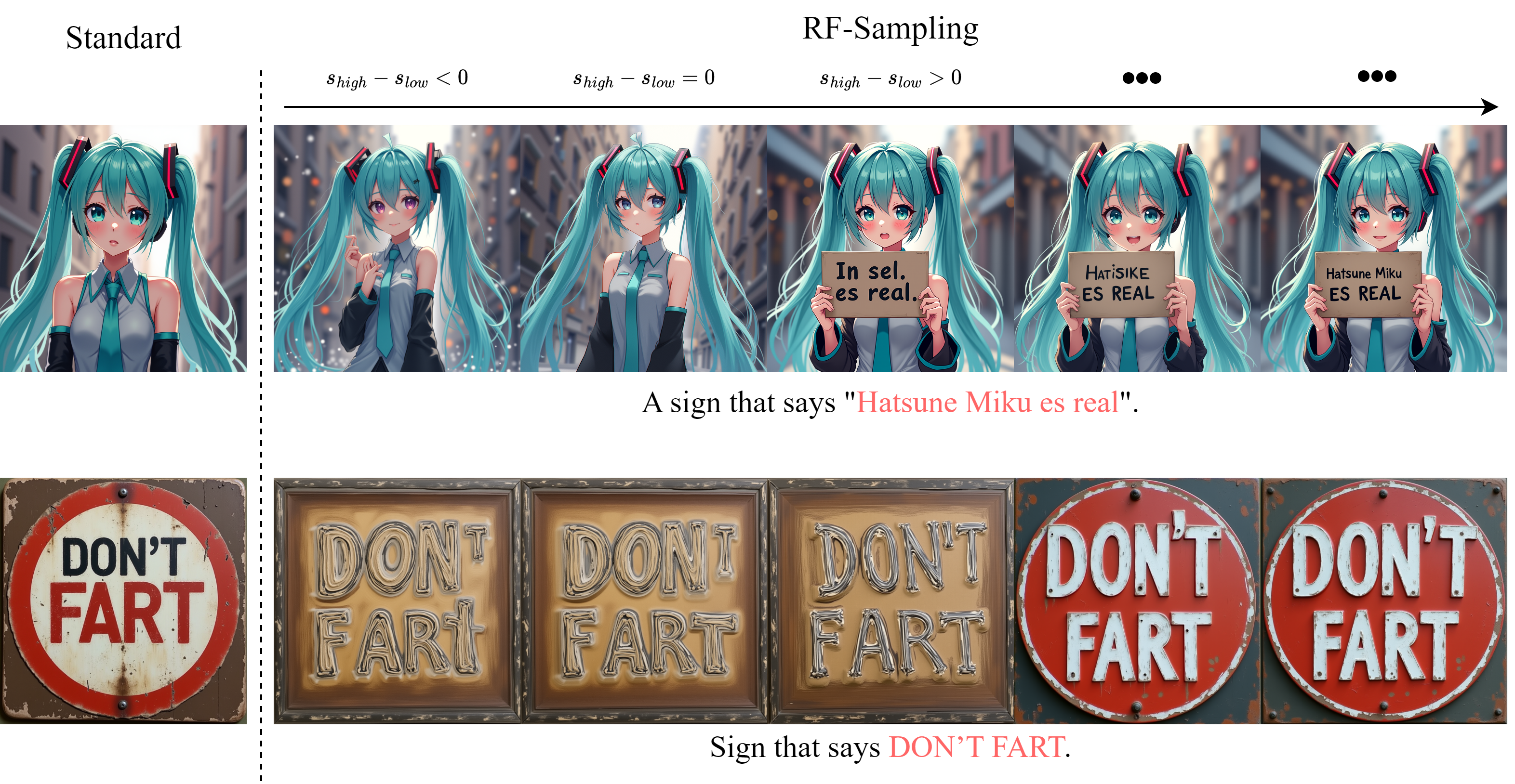}
\caption{Visualization of synthesized images with different $s$ scales. $s_{high} > s_{low}$ serves as a necessary condition for achieving superior image synthesis quality. When $s_{high} - s_{low} < 0$, RF-Sampling shows minimal advantage over the standard method with poor text generation accuracy and blurred details. However, as $s_{high} - s_{low}$ increases to positive values, RF-Sampling exhibits dramatic improvements in text rendering precision, visual detail clarity, and overall image coherence, ultimately generating high-fidelity outputs that significantly outperform Standard approaches. The results indicate that the parameter relationship $s_{high} > s_{low}$ acts as a crucial control mechanism that enables RF-Sampling to leverage its full potential for complex text-to-image generation tasks.}
\label{figure:high-denoise}
\end{figure*}

\begin{figure*}[h]
\centering
\includegraphics[width=.9\textwidth,trim={0cm 0cm 0cm 0cm},clip]{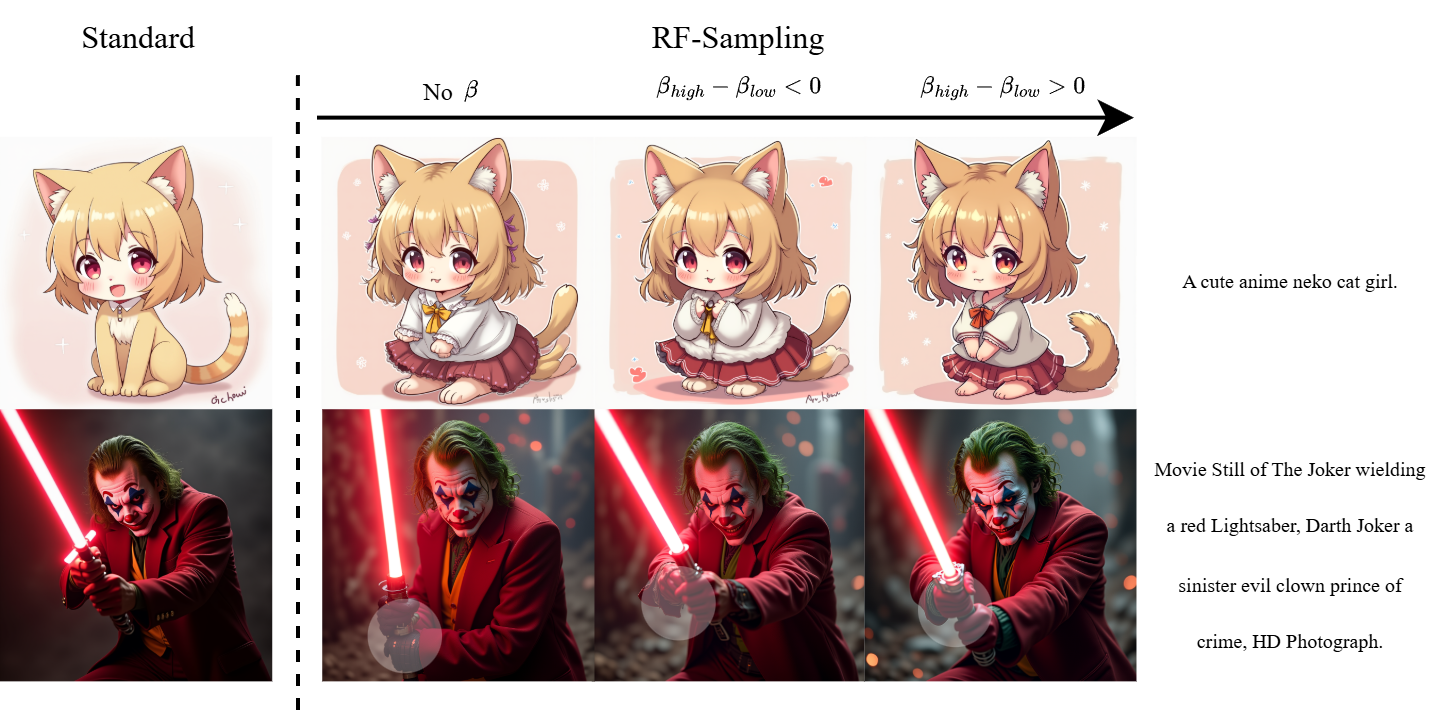}
\caption{Visualization of synthesized images with different $\beta$ scales. By applying the interpolation weight $\beta$, the model can synthesize higher-quality, more detailed, and visually appealing images that better align with user expectations for complex prompts, especially when $\beta_{high} > \beta_{low}$.}
\label{figure:low-inversion}
\end{figure*}

\begin{figure*}[!h]
\centering
\includegraphics[width=.9\textwidth,trim={0cm 0cm 0cm 0cm},clip]{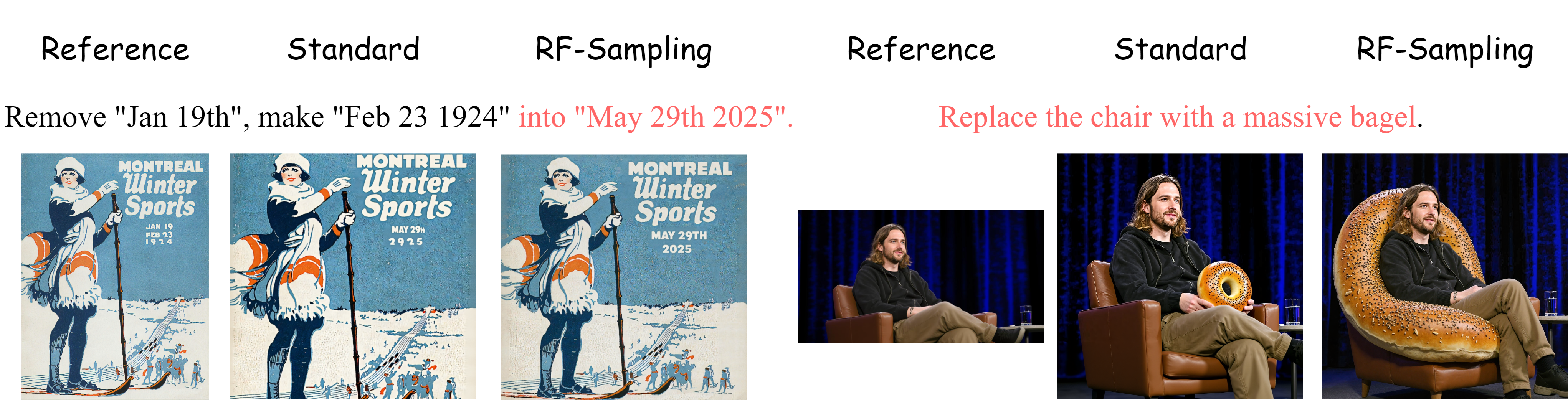}
\caption{Image editing experiments on FLUX-Kontext Bench~\cite{labs2025flux1kontextflowmatching}. Compared to the standard sampling, RF-sampling enables a more precise understanding of the given instruction, thereby achieving accurate image editing. For more examples, please see Appendix Fig.~\ref{figure:kontext}.}
\label{figure:kontext-main-paper}
\end{figure*}

\paragraph{The form of Null prompt.} Since the null-text representation can be either implemented as zero padding $\mathbf{0}$ or as an explicit null token $\varnothing$, we conduct experiments under different values of the parameter $\alpha$. The results are reported in Tab.~\ref{tab:null_and_alpha}, which demonstrate that using a null-text representation provides our method with stronger unconditional information.

\paragraph{Effect of parameter $\alpha$.} We evaluated the impact of the parameter $\alpha$ on inference performance, with results shown in Tab.~\ref{tab:null_and_alpha}. Considering that $\alpha$ also affects inference speed, we set $\alpha = 2$ for FLUX Lite and $\alpha = 1$ for FLUX Dev in our final configuration.

\paragraph{Effect of merge ratio $\gamma$.} 
To determine the optimal merge ratio $\gamma$ for our method, we conduct a quantitative study on the Pick-a-Pic dataset shown in Fig.~\ref{figure:merge-ratio}. We systematically vary the value of $\gamma$ and evaluate the quality of the synthesized images using four distinct metrics. Across all four metrics, the highest scores are usually achieved when $\gamma$ was set to $0.5$. This suggests that an equal balance in the merge operation is critical for producing the highest-quality images. Lower or higher values of $\gamma$ led to a noticeable degradation in performance, indicating an imbalanced fusion. 

\paragraph{Effect of Guidance Scale.} Traditional T2I diffusion models can enhance the quality of synthesized images by increasing the inference steps and guidance scale. In RF-Sampling, we adopt High-Weight Denoising $\rightarrow$ Low-Weight Inversion, which implicitly increases the inference steps and guidance scales. To validate our method, We further conduct an ablation study on the standard guidance scale $w$ and inference steps as shown in Fig~\ref{figure:cfg-effect}, Fig.~\ref{figure:time-scaling}, respectively. As $w$ increases, we observe a clear degradation in the quality of the synthesized images. In addition, with increasing inference steps, the performance gains of RF-Sampling. This finding confirms that the performance improvement of RF-Sampling does not originate from simply amplifying the guidance through a larger weight $s$, but rather from the reflective mechanism itself.

\paragraph{Efficiency Analysis.} 
To demonstrate the efficiency of our method, we conduct performance comparison experiments under the same number of inference steps. As shown in Fig.~\ref{figure:time-scaling}, the results indicate that our method achieves better performance within the same inference steps. 

\begin{table*}[h]
\begin{center}

\caption{We extend our method to the video generation task. Due to the computational budget, we utilize Wan2.1-T2V-1.3B~\cite{wan2025}. The results on ChronoMagic-Bench-150~\cite{yuan2024chronomagic} across 4 metrics show the promising scalability of our method to the video generation task.}

\renewcommand\arraystretch{1.2}
\setlength{\tabcolsep}{10pt}

\resizebox{.9\linewidth}{!}{%
\begin{tabular}{ccccc}
\hline 
            Method
            & UMT-FVD (↓)                    & UMTScore (↑)                   & GPT4o-MTScore (↑)              & MTScore(↑)                      \\ \hline
Standard    & 264.84                         & 2.7053                         & 3.4797                         & 0.41497                         \\
RF-Sampling & \CC229.49 & \CC2.9095 & \CC3.5302 & \CC0.43671\\
\bottomrule
\end{tabular}
}
\label{table:wanx}
\end{center}
\end{table*}

To ensure a fair comparison, we evaluate our method against baselines under an equivalent computational budget in Tab.~\ref{tab:time_cost}. With the total inference time matched (e.g., 84 steps), RF-Sampling consistently outperforms methods like GI, CFG++, and Z-Sampling. Furthermore, Tab.~\ref{tab:best_of_n} shows that our method offers a superior trade-off compared to Best-of-N strategies, surpassing Best-of-3 in metrics such as PickScore and AES while being approximately $1.5\times$ faster. Notably, comparisons with \cite{ma2025inference} on DrawBench and T2I-CompBench (Tab.~\ref{tab:saining_scaling_draw}, Tab.~\ref{tab:saining_scaling_t2i}) highlight a substantial efficiency advantage: RF-Sampling achieves competitive or top-tier results using only 150 NFEs, far fewer than the 2880 NFEs required by the baseline. Finally, to further improve efficiency, we conduct orthogonal experiments with Nunchaku~\cite{li2024svdquant}, a sampling acceleration method for FLUX. The results, presented in Tab.~\ref{tab:nunchaku}, show that our method can be effectively combined with such acceleration techniques, highlighting its potential for speedup.


\paragraph{Distribution Analysis.} 
To explore the RF-Sampling trajectories, we select two ImageNet classes~\cite{ILSVRC15} (Ambulance and Zebra) and visualize their respective data distributions as green shaded regions in the UMAP space. For each class, we randomly sample 6 Gaussian noises and process them through both standard diffusion sampling and RF-Sampling methods on FLUX-Dev using the prompt format "a photo of {class} in ImageNet." The results shown in Fig.~\ref{figure:real-distribution} reveal that RF-Sampling trajectories consistently demonstrate stronger convergence towards the real data distribution compared to standard method trajectories, as evidenced by the optimized endpoints (triangles) being more tightly clustered within or closer to the dense real data regions than their corresponding standard endpoints (squares). This convergence pattern indicates that RF-Sampling successfully refines the generation process by moving latent representations closer to the manifold of real images, thereby enhancing the fidelity and realism of generated samples while maintaining the semantic coherence of the target ImageNet classes. To validate the effectiveness of RF-Sampling, we conduct the experiments on ImageNet-1k, as shown in Tab.~\ref{tab:imagenet_5k}. The results indicate that RF-Sampling can synthesize image samples that better align with the real data distribution while maintaining high-quality and diverse image generation.

\subsection{Generalization to Other Tasks}



To further validate the generality and robustness of our approach, we extend its application beyond the standard text-to-image generation task to image editing, video generation, and LoRA fine-tuning.

\paragraph{Image Editing.} As shown in Fig.~\ref{figure:kontext-main-paper} and Appendix Fig.~\ref{figure:kontext}, our method achieves a winning rate of $57\%$ when evaluated under editing scenarios, highlighting its ability to preserve semantic alignment and generate coherent modifications guided by textual instructions.

\paragraph{Video Generation.} We further apply our method to the challenging task of video generation. The results, presented in Appendix Fig.~\ref{figure:wan} and Tab.~\ref{table:wanx}, indicate that our approach consistently enhances video quality, confirming that the reflective mechanism generalizes well to sequential data.

\paragraph{LoRA Combination.} Finally, we examine the compatibility of our method with lightweight fine-tuning techniques. As shown in Fig.~\ref{figure:lora-main} and Appendix Fig.~\ref{figure:lora-appendix}, our method remains effective when combined with LoRA-based models, demonstrating that inference enhancements are orthogonal and complementary to parameter-efficient adaptation strategies.




\subsection{Theoretical Discussion}
Based on the theoretical proofs in Sec.~\ref{method}, we can provide rigorous theoretical explanations for Fig.~\ref{figure:time-scaling} and Fig.~\ref{figure:merge_fit}.

Fig.~\ref{figure:time-scaling} provides empirical support for our theoretical derivation. As inference time increases (corresponding to a decrease in the discretization step size $\delta t$), the approximation error of the reflective direction $\Delta_{RF}$ diminishes. Consequently, the estimated gradient becomes more accurate, allowing RF-Sampling to consistently improve generation quality  with increased compute. This contrasts with standard sampling, which often saturates or degrades, highlighting that RF-Sampling effectively leverages finer temporal discretization to achieve better text-image alignment.

Eqn.~\ref{eq:thm_second_order} explains the inverted U-shaped curves observed in Fig.~\ref{figure:merge_fit}. When $\gamma < \gamma^*$, the linear term dominates, leading to quality improvement. When $\gamma > \gamma^*$, the quadratic penalty term ($-\frac{1}{2} C_t \gamma^2$) grows faster than the linear gain, causing the image quality to degrade. Although there are some local fluctuation in Fig.~\ref{figure:merge_fit}, which may be caused by non-smoothness of the evaluation metrics, the polynomial fit~(dotted lines) robustly demonstrates the overall inverted U-shaped trend, confirming the consistency between our theoretical second-order analysis and the empirical results.

\section{Conclusion}

In this work, we introduced RF-Sampling, a novel training-free inference enhancement method tailored for flow models, particularly those CFG-distilled variants. Moving beyond heuristic interpretations, we established a rigorous theoretical foundation for our method, mathematically proving that the proposed reflective mechanism serves as an accurate proxy for the gradient of the text-image alignment score. This formulation effectively bridges the gap between inference-time intervention and principled gradient ascent optimization. Our experiments demonstrate that RF-Sampling significantly improves both generation quality and text-prompt alignment, outperforming existing methods and achieving top-1 performance in various evaluations. This theoretical groundedness also unlocks the capability for test-time scaling, a property largely absent in previous methods, where increasing the inference compute consistently yields higher generation quality. Future work could explore adaptive learning rate schedules or higher-order optimization methods to further push the boundaries of flow-based generation.


{
    \small
    \bibliography{main}

}

\clearpage
\onecolumn
\appendix

\subsection{Benchmark}
\label{apd:benchmark_and_dataset}

\paragraph{Pick-a-Pic.} Pick-a-Pic~\cite{kirstain2023pickapicopendatasetuser} is an open dataset curated to capture user preference for T2I-synthesized images. Collected through an intuitive web application, it contains over 500,000 examples based on 35,000 unique prompts, providing a large-scale foundation for studying user preferences.

\paragraph{DrawBench.} DrawBench~\cite{saharia2022photorealistic}\footnote{https://huggingface.co/datasets/shunk031/DrawBench} is a benchmark dataset introduced to enable comprehensive evaluation of T2I models. It consists of 200 meticulously designed prompts, categorized into 11 groups to assess model capabilities across various semantic dimensions. These dimensions include compositionality, numerical reasoning, spatial relationships, and the ability to interpret complex textual instructions. DrawBench is specifically designed to provide a multidimensional analysis of model performance, facilitating the identification of both strengths and weaknesses in T2I synthesis.

\paragraph{HPD v2.} The human preference dataset v2 (HPD v2)~\cite{wu2023humanpreferencescorev2} is an extensive dataset featuring clean and precise annotations. With 798,090 binary preference labels across 433,760 image pairs, it addresses the limitations of conventional evaluation metrics that fail to accurately reflect human preferences. Following the methodologies in ~\cite{wu2023humanpreferencescorev2,shao2025bagdesignchoicesinference}, we employed four distinct subsets for our analysis: Animation, Concept-art, Painting, and Photo, each containing 800 prompts.

\paragraph{GenEval.} GenEval~\cite{ghosh2023genevalobjectfocusedframeworkevaluating} is an evaluation framework specifically designed to assess the compositional properties of synthesized images, such as object co-occurrence, spatial positioning, object count, and color. By leveraging state-of-the-art detection models, GenEval provides a robust evaluation of T2I generation tasks, ensuring strong alignment with human judgments. Additionally, the framework allows for the integration of other advanced vision models to validate specific attributes. The benchmark comprises 550 prompts, all of which are straightforward and easy to interpret.

\paragraph{T2I-Compbench.} T2I-Compbench~\cite{huang2023t2icompbench} is a comprehensive benchmark for evaluating open-world compositional T2I synthesis. It includes 6,000 compositional text prompts, systematically categorized into three primary groups: attribute binding, object relationships, and complex compositions. These groups are further divided into six subcategories: color binding, shape binding, texture binding, spatial relationships, non-spatial relationships, and intricate compositions.

\paragraph{ChronoMagic-Bench-150.} Chronomagic-Bench-150, introduced in~\cite{yuan2024chronomagic} serves as a comprehensive benchmark for metamorphic evaluation of timelapse T2V synthesis. This benchmark includes 4 main categories of time-lapse videos: biological, human-created, meteorological, and physical, further divided into 75 subcategories. Each subcategory contains two challenging prompts, leading to in a total of 150 prompts. We consider three distinct metrics in Chronomagic-Bench-150: {UMT-FVD (↓)}, {UMTScore (↑)}, {GPT4o-MTScore (↑)} and {MTScore (↑)}.

\paragraph{FLUX-Kontext-Bench.}
FLUX-Kontext-Bench, introduced in~\cite{labs2025flux1kontextflowmatching}, is a comprehensive benchmark for evaluating in-context image generation and editing models. It consists of 1026 image-prompt pairs derived from 108 base images. The benchmark spans five core task categories: local editing, global editing, text editing, style reference, and character reference. Designed to reflect real-world usage, FLUX-Kontext-Bench addresses limitations of prior synthetic or narrow-scope benchmarks and supports holistic evaluation of both single-turn quality and multi-turn consistency.

\subsection{Evaluation Metric}

\paragraph{PickScore.} PickScore is a CLIP-based scoring model, developed using the Pick-a-Pic dataset, which captures user preferences for synthesized images. This metric demonstrates performance surpassing that of typical human benchmarks in predicting user preferences. By aligning effectively with human evaluations and leveraging the diverse range of prompts in the Pick-a-Pic dataset, PickScore offers a more relevant and insightful assessment of T2I models compared to traditional metrics like FID~\cite{heusel2018ganstrainedtimescaleupdate} on datasets such as MS-COCO~\cite{lin2015microsoftcococommonobjects}.

\paragraph{HPS v2.} The human preference score version 2 (HPS v2) is an improved model to predict user preferences, created by fine-tuning the CLIP model~\cite{radford2021learningtransferablevisualmodels} on the HPD v2. This refined metric is designed to align T2I generation outputs with human tastes by estimating the likelihood that a synthesized image will be preferred, thereby serving as a reliable benchmark for evaluating the performance of T2I models across diverse image distributions.

\paragraph{AES.} The Aesthetic Score (AES)~\cite{AES} is a metric that evaluates the visual appeal of images. It is calculated using a model bulit on CLIP embeddings and enhanced with multilayer perceptron~(MLP) layers. This metric provides a quantitative measure of the aesthetic quality of synthesized images, offering valuable insights into their alignment with human aesthetic standards.

\paragraph{ImageReward.}ImageReward~\cite{xu2023imagerewardlearningevaluatinghuman} is a specialized reward model designed to evaluate T2I synthesis based on human preferences. Trained on a large-scale dataset of human comparisons, the model effectively captures user inclinations by assessing multiple aspects of synthesized images, including their alignment iwth text prompts and their aesthetic quality. ImageReward has shown superior performance compared to traditional metrics such as the Inception Score (IS)~
\cite{is} and Fréchet Inception Distance (FID), establishing it as a highly promising tool for automated evaluation in T2I tasks.

\subsection{Flow Models}
In the main paper, we totally use 3 flow-based T2I diffusion models, including FLUX-Dev~\cite{flux2024}, FLUX-Lite~\cite{flux1-lite}, and 
StableDiffusion-3.5~\cite{esser2024scalingrectifiedflowtransformers}, 1 flow-based T2V diffusion, Wan2.1-T2V-1.3B~\cite{wan2025}, and 1 flow-based TI2I diffusion model, FLUX-Kontext~\cite{labs2025flux1kontextflowmatching}.

\paragraph{FLUX-Dev.} FLUX-Dev~\cite{flux2024} is a family of T2I diffusion models built upon a transformer-based architecture, departing from the conventional U-Net design. Its core components include a dual text encoder system~(CLIP and T5~\cite{chung2022scalinginstructionfinetunedlanguagemodels}) for robust prompt understanding and a joint attention mechanism. This mechanism facilitates a bidirectional information flow between image and text representations within the transformer blocks, significantly enhancing prompt fidelity. The models are trained using a rectified flow formulation~\cite{iclr22_rect}, which enables high-quality image synthesis with fewer sampling steps compared to traditional diffusion models.

\paragraph{FLUX-Lite.} FLUX-Lite is a lightweight and highly efficient version derived from the FLUX models, optimized for faster inference. This 8B parameter model achieving a 23\% reduction in latency and a 7GB decrease in RAM usage. Its robustness is enhanced by a refined distillation process, trained on a diverse dataset and optimized for a broad range of guidance values~(2.0-5.0) and step counts~(20-32). The model's efficiency stems from an architectural insight that its transformer blocks contribute heterogeneously. An analysis revealed that intermediate blocks possess a degree of redundancy, unlike the critical initial and final blocks. The property allows for effective distillation and optimization without significant degradation in generative performance.

\paragraph{Stable-Diffusion-3.5.} StableDiffusion-3.5 marks a significant architectural shift in the StableDiffusion series to a Diffusion Transformer~(DiT)~\cite{peebles2023scalablediffusionmodelstransformers} model, aligning with the principles of rectified flow. As described by~\cite{esser2024scalingrectifiedflowtransformers}, this model processes text and image modalities using separate transformer weights before fusing them with a joint attention mechanism. This approach enables a sophisticated, bidirectional interaction between the two modalities, leading to well performance in prompt adherence, typographic generation, and overall image coherence. Its design demonstrates predictable scaling, where improvements in training loss directly translate to superior synthesis quality.

\paragraph{Wan2.1.} Wan2.1, introduced in~\cite{wan2025}, is an open-source video generation model developed by Alibaba, based on a Diffusion Transformer~(DiT) architecture and flow matching framework. It supports multiple tasks including text-to-video~(T2V) and image-to-video~(I2V). The model is available in two versions: a 14B-parameter variant for high-quality 720p generation and a lightweight 1.3B variant suitable for consumer-grade GPUs. Due to the resource limits, in this paper, we utilize the Wan2.1-T2V-1.3B.

\paragraph{FLUX-Kontext.} FLUX-Kontext, introduced in~\cite{labs2025flux1kontextflowmatching}, is a unified flow matching model for in-context image generation and editing in latent space. It combines text and image conditioning through a simple sequence concatenation mechanism, enabling both local editing and generative tasks within a single architecture. The model excels in preserving character and object consistency across multiple iterative edits, supports high-resolution output at interactive speeds, and facilitates iterative workflows.

\subsection{Hyperparameter Settings}
For all the experiments in the main paper, the inference steps are default to 28, 28, 50, 50, and 50, corresponding to SD3.5, FLUX-Lite, FLUX-Dev, FLUX-Kontext and Wan-2.1-T2V-1.3B. For the standard guidance scales $w$ are default to 4.5, 3.5, 3.5, 3.5, and 5.0, corresponding to SD3.5, FLUX-Lite, FLUX-Dev, FLUX-Kontext and Wan-2.1-T2V-1.3B.

For the hyperparameters, the interpolation weight $\beta_{high} = 0.7$, $\beta_{low} = 0.3$, and the merge ratio $\gamma = 0.5$~(for Wan-2.1-T2V-1.3B, $\gamma = 0.03$) across all the experiments. The amplifying weight $s_{high} = 3.5$, $s_{low} = 0$ for FLUX-Dev, and $s_{high} = 9$, $s_{low} = -1$ for FLUX-Lite, FLUX-Kontext, and SD3.5. The repeat time $\alpha = 1$ for SD3.5, FLUX-Dev, FLUX-Kontext, and Wan-2.1-T2V-1.3B, and $\alpha = 2$ for FLUX-Lite. For experiments in SD3.5, FLUX-Lite and FLUX-Dev, we execute RF-Sampling operations through all the inference steps, for FLUX-Kontext and Wan-2.1-T2V-1.3B, due to the time budgets, we only perform RF-Sampling operations in the first two steps.

\begin{figure*}[!ht]
\centering
\includegraphics[width=1.\textwidth,trim={0cm 0cm 0cm 0cm},clip]{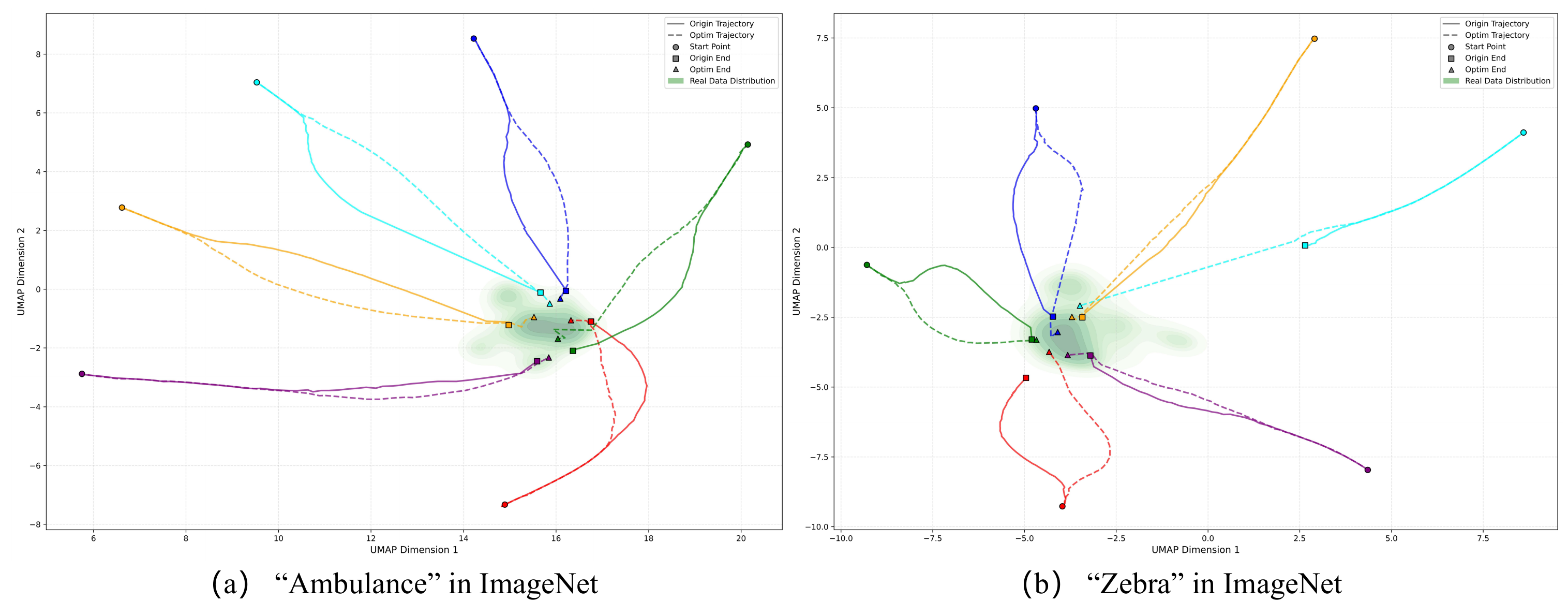}
\caption{Visualizations of the sampling trajectories of RF-Sampling and the standard method. we randomly select two ImageNet classes~\cite{ILSVRC15} (Ambulance and Zebra) and visualize their respective data distributions. For each class, we randomly sample \textbf{6} Gaussian noises and process them through both standard diffusion sampling and RF-Sampling methods using the prompt format \textbf{"a photo of {class} in ImageNet."} The results reveal that RF-Sampling trajectories consistently demonstrate stronger convergence towards the real data distribution compared to standard method trajectories.}
\label{figure:real-distribution}
\end{figure*}


\subsection{Additional Analysis}
\label{sec:add_analysis}

To further understand the effect of the parameter in our method, we conduct additional parameter analysis experiments as shown follows:

\paragraph{Additional Comparison Results and Efficiency Analysis.}
To comprehensively validate the effectiveness and efficiency of RF-Sampling, we conducted a series of comparative experiments. First, Tab.~\ref{tab:best_of_n} highlights the superior trade-off between performance and computational cost. RF-Sampling not only outperforms standard sampling across all metrics but also surpasses the Best-of-3 method in PickScore, AES, and ImageReward while being approximately $1.5\times$ faster. To ensure a fair comparison, we evaluated the methods under an equivalent computational budget in Tab.~\ref{tab:time_cost}. With the total inference time matched (84 steps), RF-Sampling consistently outperforms baseline methods, including GI, CFG++, and Z-Sampling, across multiple metrics. Furthermore, comparisons with \cite{ma2025inference} on DrawBench (Tab.~\ref{tab:saining_scaling_draw}) and T2I-CompBench (Tab.~\ref{tab:saining_scaling_t2i}) demonstrate the significant efficiency advantage of our approach; RF-Sampling achieves competitive or top-tier results using only 150 NFEs, a substantial reduction compared to the 2880 NFEs required by the baseline. Finally, we analyzed the scalability and robustness of our method. Tab.~\ref{tab:random_seed} shows that RF-Sampling consistently outperforms standard sampling at equivalent time consumption levels on both FLUX-Lite and FLUX-Dev.

\paragraph{Experiments on Large-scale Dataset.} To further validate the effectiveness of RF-Sampling, we conduct experiments on popular, large-scale benchmarks like GenEval~\cite{ghosh2023genevalobjectfocusedframeworkevaluating} and T2I-CompBench~\cite{huang2023t2icompbench} across 3 different flow models, shown in Tab.~\ref{tab:t2i} and Tab.~\ref{tab:geneval_full}. The large-scale experiments demonstrate the generalizability and robustness of our proposed methods.

\begin{table*}[!ht]
\begin{center}

\renewcommand\arraystretch{1.5}
\setlength{\tabcolsep}{10pt}

\caption{We evaluate the effectiveness of RF-Sampling on T2I-CompBench~\cite{huang2023t2icompbench}  across 3 diffusion models. The results validate the effectiveness and generalizability of our method.} 
\resizebox{1.\linewidth}{!}{\begin{tabular}{ccccccccccc}
\hline
                                                                                 &                          & \multicolumn{3}{c}{Attribute Binding}                                                                                                                                 & \multicolumn{4}{c}{Object Relationship}                                                                                                                                                                                       &                                                       &                                                       \\ \cline{3-9}
\multirow{-2}{*}{Model}                                                          & \multirow{-2}{*}{Method} & Color(↑)                                              & Shape(↑)                                              & Texture(↑)                                            & 2D-Spatial(↑)                                         & 3D-Spatial(↑)                                         & Non-Spatial(↑)                                        & numeracy(↑)                                           & \multirow{-2}{*}{Complex(↑)}                          & \multirow{-2}{*}{Overall(↑)}                          \\ \hline
                                                                                 & Standard                 & 0.7511                                               & 0.5709                                                & 0.7119                                                & \CC{ 0.2927} & 0.3751                                                & 0.3166                                                & 0.6078                                                & \CC{ 0.3846} & 0.5013                                                \\
\multirow{-2}{*}{\begin{tabular}[c]{@{}c@{}}SD3.5\\ (28 steps)\end{tabular}}     & RF-Sampling              & \CC{ 0.7817} & \CC{ 0.5885} & \CC{ 0.7241} & 0.2864                                                & \CC{ 0.3974} & \CC{ 0.3174} & \CC{ 0.6121} & 0.3844                                                & \CC{ 0.5119} \\ \hline
                                                                                 & Standard                 & 0.7030                                                & 0.4154                                                & 0.4887                                                & 0.2258                                                & 0.3710                                                & 0.3030                                                & 0.5564                                                & 0.3365                                                & 0.4249                                                \\
\multirow{-2}{*}{\begin{tabular}[c]{@{}c@{}}FLUX-Lite\\ (28 steps)\end{tabular}} & RF-Sampling              & \CC{ 0.7613} & \CC{ 0.4725} & \CC{ 0.5970} & \CC{ 0.2420} & \CC{ 0.4042} & \CC{ 0.3070} & \CC{ 0.6090} & \CC{ 0.3649} & \CC{ 0.4698} \\ \hline
                                                                                 & Standard                 & 0.7535                                                & 0.5018                                                & 0.6167                                                & \CC{ 0.2783} & 0.3866                                                & 0.3078                                                & 0.6052                                                & 0.3706                                                & 0.4775                                                \\
\multirow{-2}{*}{\begin{tabular}[c]{@{}c@{}}FLUX-Dev\\ (50 steps)\end{tabular}}  & RF-Sampling              & \CC{ 0.7761} & \CC{ 0.5323} & \CC{ 0.6422} & { 0.2687}                         & \CC{ 0.3943} & \CC{ 0.3080} & \CC{ 0.6082} & \CC{ 0.3733} & \CC{ 0.4887} \\ \hline
\end{tabular}
}
\label{tab:t2i}
\end{center}
\end{table*}
\begin{table*}[!ht]
\begin{center}

\renewcommand\arraystretch{1.2}
\setlength{\tabcolsep}{10pt}

\caption{We evaluate the effectiveness of RF-Sampling on GenEval~\cite{ghosh2023genevalobjectfocusedframeworkevaluating} across 3 diffusion models. The results show the superiority over the standard. } 
\resizebox{1.\linewidth}{!}{\begin{tabular}{ccccccccc}
\hline
{ Model}                                                                            & { Method}      & { Single(↑)}                    & { Two(↑)}                       & { Counting(↑)}                  & { Colors(↑)}                    & { Positions(↑)}                 & { Color Attribution(↑)}         & { Overall(↑)}                   \\ \hline
{ }                                                                                 & { Standard}    & { 0.97}                         & \CC{ 0.91} & \CC{ 0.75} & { 0.85}                         & \CC{ 0.21} & { 0.53}                         & { 0.70}                         \\
\multirow{-2}{*}{{ \begin{tabular}[c]{@{}c@{}}SD3.5\\ (28 steps)\end{tabular}}}     & { RF-Sampling} & \CC{ 0.99} & \CC{ 0.91} & { 0.72}                         & \CC{ 0.89} & { 0.19}                         & \CC{ 0.54} & \CC{ 0.71} \\ \hline
{ }                                                                                 & { Standard}    & { 0.90}                         & { 0.57}                         & { 0.52}                         & { 0.71}                         & { 0.11}                         & { 0.36}                         & { 0.53}                         \\
\multirow{-2}{*}{{ \begin{tabular}[c]{@{}c@{}}FLUX-Lite\\ (28 steps)\end{tabular}}} & { RF-Sampling} & \CC{ 0.93} & \CC{ 0.62} & \CC{ 0.59} & \CC{ 0.73} & \CC{ 0.18} & \CC{ 0.42} & \CC{ 0.58} \\ \hline
{ }                                                                                 & { Standard}    & \CC{ 0.99} & { 0.80}                         & \CC{ 0.78} & { 0.77}                         & { 0.23}                         & \CC{ 0.50} & { 0.68}                         \\
\multirow{-2}{*}{{ \begin{tabular}[c]{@{}c@{}}FLUX-Dev\\ (50 steps)\end{tabular}}}  & { RF-Sampling} & \CC{ 0.99} & \CC{ 0.82} & { 0.76}                         & \CC{ 0.80} & \CC{ 0.25} & \CC{ 0.50} & \CC{ 0.69}  \\  \hline
\end{tabular}
}
\label{tab:geneval_full}
\end{center}
\end{table*}


\begin{figure*}[!ht]
\centering
\includegraphics[width=1.\textwidth,trim={0cm 0cm 0cm 0cm},clip]{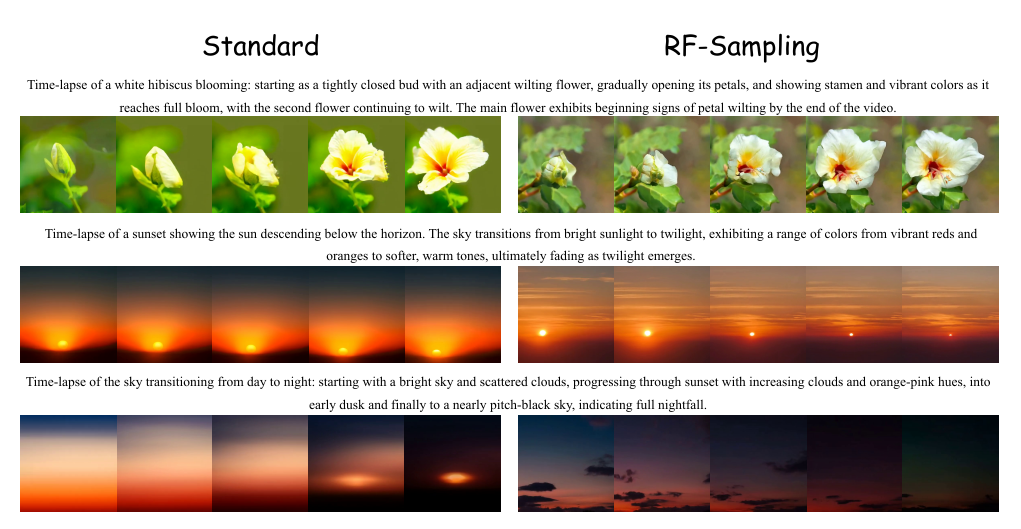}
\caption{We directly extend our proposed method to video generation task on Wan2.1-T2V-1.3B. The visualizations show the superiority of our proposed method compared with standard sampling.}
\label{figure:wan}
\end{figure*}

\begin{figure*}[!ht]
\centering
\includegraphics[width=1.\textwidth,trim={0cm 0cm 0cm 0cm},clip]{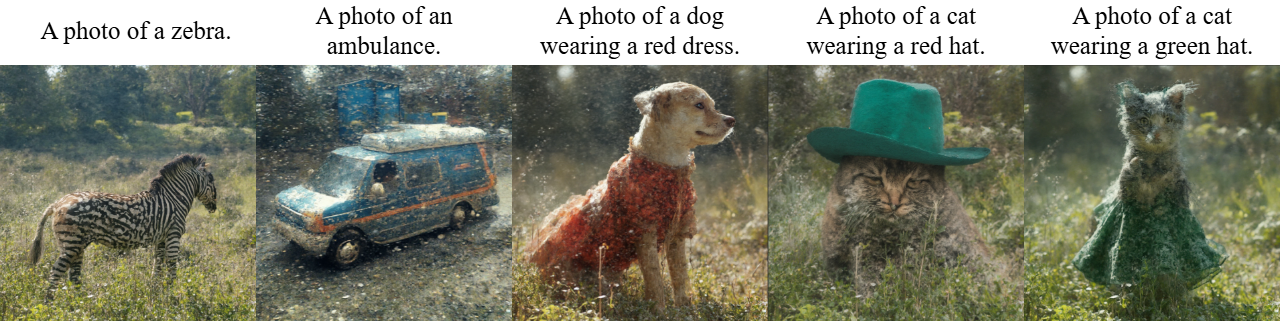}
\caption{\small Visual results of FLUX-Lite with guidance scale $w=1$. The generated images remain semantically aligned with the input text prompts, demonstrating that the model's output is still conditionally generated even at the minimum guidance scale. This empirically verifies that CFG-distilled models like FLUX do not possess a true unconditional generation mode, and setting $w=1$ does not produce unconditional outputs.}
\label{figure:w_1_condition}
\end{figure*}

\begin{figure*}[h]
\centering
\includegraphics[width=1.\textwidth,trim={0cm 0cm 0cm 0cm},clip]{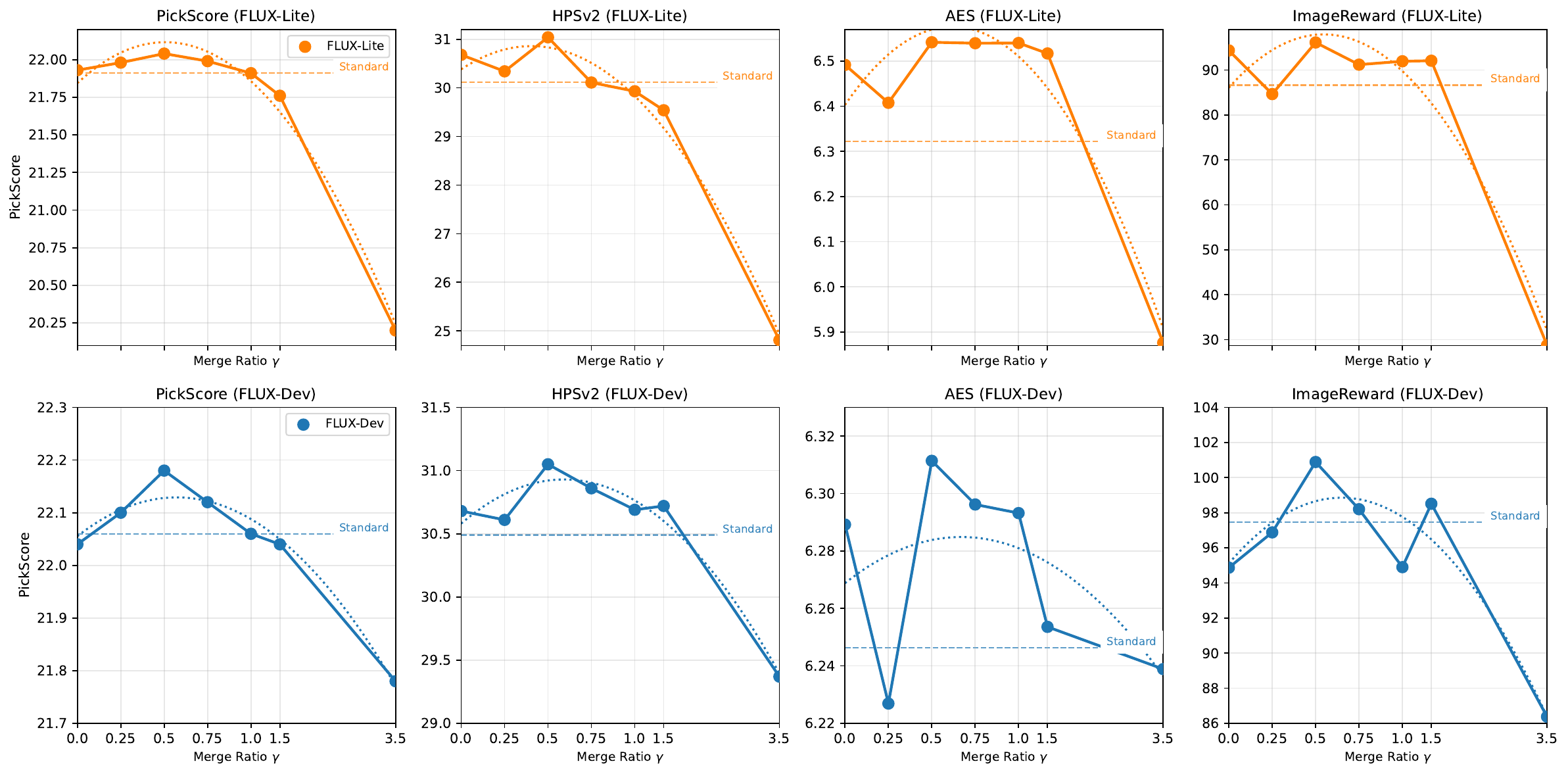}
\caption{Impact of Merge Ratio $\gamma$ on generation quality. The inverted U-shaped curves across all metrics confirm the existence of an optimal step size, balancing gradient alignment and manifold constraints. FLUX-Dev shows significantly higher robustness to large $\gamma$ values than FLUX-Lite, attributed to the smoother latent manifold of the larger model. Dotted curve lines represent quadratic fits to the data.}
\label{figure:merge_fit}
\end{figure*}


\subsection{Theoretical Discussion of RF-Sampling}
\label{apd:sec:theoretical_analysis}

\paragraph{Optimization Objective. }
\label{apd:sec:optimzation_score}
In text-to-image generation task, our goal at inference time is to find a latent $x$ that maximizes the probability of the text condition $c$. We define the optimization objective function $J(x)$ as the log-posterior, referred to as the alignment score~\cite{hessel2021clipscore,kirstain2023pickapicopendatasetuser,xu2023imagerewardlearningevaluatinghuman,wu2023humanpreferencescorev2}:

\begin{equation} 
\label{eq:objective_apd}
    J(x) = \log p(c | x) = \log p(x | c) - \log p(x) + \text{const}
\end{equation}

According to the score-based generative modeling theory~\cite{song2019generative,sde} and Classifier-Free Guidance~(CFG)~\cite{ho2022classifier}, the gradient of this objective function~(the score) can be approximated by the difference between conditional and unconditional noise/velocity predictions. In flow matching models~\cite{flux2024,lipman2022flow,liu2022flow}, the vector field $v_\theta(x, t, c)$ predicts the velocity that points to the data distribution. The gradient of the log-likelihood with respect to $x$ is proportional to the difference in velocity fields:

\begin{equation} 
\label{eq:gradient_proxy}
    \nabla_x J(x) \propto v_\theta(x, t, c) - v_\theta(x, t, \emptyset)
\end{equation}

where $\emptyset$ represents the null prompt. Standard CFG modifies the velocity as $v = v_{uncond} + w (v_{cond} - v_{uncond})$. For the CFG-distilled models, our goal is to perform gradient ascent on $J(x)$ without explicit CFG calculations at every step, utilizing the flow trajectory itself.

\paragraph{RF-Sampling. }
\label{apd:sec:rf-sampling}
RF-Sampling introduces a heuristic:``Denoise with high weight, invert with low weight.'' We now mathematically prove that this operation creates a displacement vector $\Delta_{RF}$ that is equivalent the gradient term in Eq.~\ref{eq:gradient_proxy}.

Consider the latent $x_t$ at timestep $t$. We define two text embeddings via embedding interpolation:
\begin{itemize}
    \item High Weight Embedding ($c_{high}$~(or $c'$ in the main paper)): Strong alignment requirement.
    \item Low Weight Embedding ($c_{low}$~(or $c''$ in the main paper)): Weak alignment requirement.
\end{itemize}

The RF-Sampling process consists of a high-weight denoising $\rightarrow$ low-weight inversion with step size $\delta t$.

\begin{itemize}
    \item High-Weight Denoising: Moving from $t$ to $t - \delta t$:
        \begin{equation}
            x_{t-\delta t} = x_t + v_\theta(x_t, t, c_{high}) \cdot (-\delta t)
        \end{equation}
    \item Low-Weight Inversion (Backward): Moving from $t-\delta t$ back to $t$:
        \begin{equation}
            x'_t = x_{t-\delta t} + v_\theta(x_{t-\delta t}, t-\delta t, c_{low}) \cdot (+\delta t)
        \end{equation}
\end{itemize}

Substituting $x_{t-\delta t}$ into the inversion equation:

\begin{equation}
    x'_t = \left[ x_t - v_\theta(x_t, t, c_{high}) \delta t \right] + v_\theta(x_{t-\delta t}, t-\delta t, c_{low}) \delta t
\end{equation}

We define the reflective displacement vector  $\Delta_{RF}$ as the difference between the original latent $x_t$ and the reflected latent $x_t'$:

\begin{equation}
    \Delta_{RF} = x_t - x'_t = \delta t \cdot \left[ v_\theta(x_t, t, c_{high}) - v_\theta(x_{t-\delta t}, t-\delta t, c_{low}) \right]
\end{equation}

Assuming the step size $\delta t$ is sufficiently small and the vector field $v_\theta$ is locally Lipschitz continuous~(smooth),  we can approximate $v_\theta(x_{t-\delta t}) \approx v_\theta(x_t)$~(support evidence see Fig.~\ref{figure:time-scaling}). The expression simplifies to:
\begin{equation} \label{eq:delta_sim}
    \Delta_{RF} \approx \delta t \cdot \left( v_\theta(x_t, t, c_{high}) - v_\theta(x_t, t, c_{low}) \right)
\end{equation}

\paragraph{Embedding Taylor Expansion. }
\label{apd:sec:taylor_derivation}
We define the semantic direction vector $\mathbf{u}$ as the difference between conditional and unconditional embeddings: $\mathbf{u} = c_{text} - c_{uncond}$.
Any input embedding $c_w$ in our method can be decomposed into a base component aligned with the unconditional embedding and a directional component aligned with the semantic vector:
\begin{align}
    c_w(s, \beta) &= (c_{text} + s \cdot c_{mix}) \\
                  &= c_{text} + s (\beta c_{text} + (1-\beta) c_{uncond}) \\
                  &= (1+s\beta) c_{text} + s(1-\beta) c_{uncond} \\
                  &= (1+s\beta) (c_{uncond} + \mathbf{u}) + s(1-\beta) c_{uncond} \\
                  &= \underbrace{(1+s) c_{uncond}}_{c_{base}(s)} + \underbrace{(1+s\beta)}_{\lambda(s, \beta)} \mathbf{u}
\end{align}
Here, $c_{base}(s)$ represents the baseline embedding scaling with $s$, and $\lambda(s, \beta)$ is the effective magnitude along the semantic direction. We treat the velocity field $v_\theta$ as a function of the embedding. Performing a first-order Taylor expansion around the unconditional embedding $c_{uncond}$:
\begin{equation}
    v_\theta(x, c_w) \approx v_\theta(x, c_{uncond}) + (1+s\beta) \left( \nabla_c v_\theta \cdot \mathbf{u} \right)
\end{equation}
Note that we approximate $v_\theta(x, (1+s)c_{uncond}) \approx v_\theta(x, c_{uncond})$ by simply treating it as the unconditional flow component $v_{uncond}$, supported by Tab.~\ref{tab:ablation_stage}. 
To rigorously justify why the term $(\nabla_c v_\theta \cdot \mathbf{u})$ corresponds to the score difference $\nabla_x J(x)$, we can analyze the relationship between the semantic directional derivative and the image-space gradient. By performing a first-order Taylor expansion of the conditional vector field $v_\theta(x, c_{text})$ around the unconditional embedding $c_{uncond}$ (where the semantic direction $\mathbf{u} = c_{text} - c_{uncond}$), we have:
\begin{equation}
    v_\theta(x, c_{text}) = v_\theta(x, c_{uncond} + \mathbf{u}) \approx v_\theta(x, c_{uncond}) + \nabla_c v_\theta(x, c_{uncond}) \cdot \mathbf{u}.
\end{equation}
By rearranging the terms, we observe that the directional derivative along $\mathbf{u}$ mathematically approximates the difference between the conditional and unconditional velocity fields:
\begin{equation}
    \nabla_c v_\theta(x, c_{uncond}) \cdot \mathbf{u} \approx v_\theta(x, c_{text}) - v_\theta(x, c_{uncond}).
    \label{eq:taylor_diff}
\end{equation}
Crucially, according to the theoretical formulation of Classifier-Free Guidance (CFG) and score-based generative modeling, the gradient of the log-likelihood (the alignment score $J(x) = \log p(c|x)$) is strictly proportional to this exact velocity difference:
\begin{equation}
    v_\theta(x, c_{text}) - v_\theta(x, c_{uncond}) \propto \nabla_x \log p(c|x) = \nabla_x J(x).
    \label{eq:cfg_diff}
\end{equation}
Therefore, combining Eq.~\ref{eq:taylor_diff} and Eq.~\ref{eq:cfg_diff}, we establish the fundamental equivalence that taking a small step in the text embedding space acts as a direct proxy for the score gradient in the image space:
\begin{equation}
    (\nabla_c v_\theta \cdot \mathbf{u}) \propto \nabla_x J(x).
\end{equation}

Now we calculate the difference between High-Weight Denoising ($s_{high}, \beta_{high}$) and Low-Weight Inversion ($s_{low}, \beta_{low}$). Note that $s_{high} > s_{low}$ and $\beta_{high} > \beta_{low}$.

\begin{align}
    \Delta v &= v_{high} - v_{low} \\
             &\approx \left[v_{uncond} + (1+s_{high} \beta_{high}) \nabla_x J(x) \right] \nonumber \\
             &\quad - \left[v_{uncond} + (1+s_{low} \beta_{low}) \nabla_x J(x) \right]
\end{align}

Grouping the terms by components:

\begin{equation} \label{eq:two_terms}
    \Delta v \approx \underbrace{(s_{high} \beta_{high} - s_{low} \beta_{low})}_{\text{Alignment Coefficient } \mathcal{A}} \cdot \nabla_x J(x)
\end{equation}

The coefficient $\mathcal{A} = s_{high} \beta_{high} - s_{low} \beta_{low}$. Since we configure $s_{high} > s_{low}$ and $\beta_{high} > \beta_{low}$ in our experiments, $\mathcal{A}$ is guaranteed to be positive and large. This confirms that $\Delta_{RF}$ provides a strong gradient ascent direction for text-image alignment.

\paragraph{Substitution into Optimization Object. }
Now we link $\Delta_{RF}$ back to the optimization objective $J(x)$.

From Eq.~\ref{eq:gradient_proxy} and Eq.~\ref{eq:two_terms}, we know that the difference between two vector fields with different text embedding is proportional to the gradient of the alignment score. Specifically, since $c_{high}$ is more aligned with the prompt than $c_{low}$:

\begin{equation}
    v_\theta(x_t, c_{high}) - v_\theta(x_t, c_{low}) \propto v_\theta(x_t, c) - v_\theta(x_t, \emptyset) \propto \nabla_x \log p(c | x_t) = \nabla_x J(x_t)
\end{equation}

Substituting this into Eq.~\ref{eq:delta_sim}, we obtain:

\begin{equation}
    \Delta_{RF} \approx \delta t \cdot \nabla_x J(x_t)
\end{equation}

where $\delta t$ is the step size, which is positive value.

Our goal is to maximize the alignment score $J(x)$, we utilize the gradient ascent rule:

\begin{equation}
    x_{new} = x_t + \lambda \cdot \nabla_x J(x_t)
\end{equation}

By substituting the reflective displacement $\Delta_{RF}$ as the proxy for the gradient, we arrive at the final RF-Sampling update rule:

\begin{equation} \label{eq:final_update}
    x''_t = x_t + \gamma \cdot \underbrace{(x_t - x'_t)}_{\Delta_{RF}}
\end{equation}

Here, the merge ratio $\gamma$ acts as the learning rate~(step size) for this single-step gradient ascent optimization.

\subsection{Why RF-Sampling works?}
\label{apd:sec:reason}
Let $x_t$ denote the latent at timestep $t$ generated by a Flow Matching model $\theta$. Our goal is to maximize the alignment between the image latent $x_t$ and the given text prompt $c$. We define the objective function~(Alignment Score) as the log-likelihood of the condition given the latent:

\begin{equation}
    J(x) = \log p(c | x)
\end{equation}

RF-Sampling aims to update the current latent state $x_t$ to a refined state $x''_t$ such that \textbf{$J(x''_t) > J(x_t)$}. The update rule is defined as:

\begin{equation}
    x''_t = x_t + \gamma \cdot \Delta_{RF} \approx x_t + \gamma\delta t \cdot \nabla_x J(x_t) \approx x_t + \gamma \cdot \nabla_x J(x_t)
\end{equation}

where $\gamma > 0$ is the merge ratio~(gradient ascent step size). $\Delta_{RF} = x_t - x'_t$  is the reflective direction, derived from the difference between the $x_t$ and $x'_t$. 

\paragraph{First-Order Analysis: Validity of RF-Sampling. }
\label{apd:first_order}
We first demonstrate why RF-Sampling improves generation quality using the first-order Taylor expansion.

\begin{proposition}[First-Order Validity]
    For a sufficiently small merge ratio $\gamma > 0$, if the reflective direction $\Delta_{RF}$ forms an acute angle with the true gradient $\nabla_x J(x_t)$, then RF-Sampling strictly increases the objective function value.
\end{proposition}

\begin{proof}
    Consider the first-order Taylor expansion of $J(x)$ around $x_t$:
    
    \begin{equation}
        J(x''_t) = J(x_t + \gamma \Delta_{RF}) \approx J(x_t) + \gamma \Delta_{RF}^\top \nabla_x J(x_t)
    \end{equation}
    
    The change in the objective function is:
    
    \begin{equation}
        \delta J \approx \gamma \underbrace{\Delta_{RF}^\top \nabla_x J(x_t)}_{\text{Directional Alignment}} = \gamma    \delta t \cdot \nabla_x^\top J(x_t) \cdot \nabla_x J(x_t)
    \end{equation}
    
    Since RF-Sampling distills semantic information from the guidance difference, $\Delta_{RF}$ is aligned with the gradient direction, implying $\Delta_{RF}^\top \nabla_x J(x_t) \geq 0$. Therefore, for small $\gamma$, we have $\delta J > 0$, proving that the update direction is valid.
\end{proof}

\paragraph{Second-Order Analysis: Optimality and Constraints. }
\label{apd:second_order}
While the first-order analysis suggests that larger $\gamma$ yields better results, experimental observations (Fig.~\ref{figure:merge_fit}) show an inverted U-shaped performance curve. We explain this phenomenon using the second-order Taylor expansion.

\begin{proposition}[Second-Order Optimality]
    The relationship between the objective improvement and the merge ratio $\gamma$ is parabolic. There exists an optimal merge ratio $\gamma^*$ determined by the local curvature of the semantic manifold.
\end{proposition}

\begin{proof}
    We expand $J(x)$ to the second order around $x_t$:
    
    \begin{equation}
        J(x''_t) \approx J(x_t) + \gamma \Delta_{RF}^\top \nabla_x J(x_t) + \frac{1}{2} \gamma^2 \Delta_{RF}^\top \mathbf{H}(x_t) \Delta_{RF}
    \end{equation}

    where $\mathbf{H}(x_t) = \nabla^2_x J(x_t)$ is the Hessian matrix representing the local curvature. 
    In maximization problems, we assume the objective function is locally concave, implying that the quadratic term is negative~(negative definite or negative semi-definite). Let us define the directional curvature penalty $C_{RF}$:

    \begin{equation}
        C_{RF} = - \Delta_{RF}^\top \mathbf{H}(x_t) \Delta_{RF} \qquad \text{(assuming $C_{RF} > 0$)}
    \end{equation}

    The improvement $\Delta J(\gamma) = J(x''_t) - J(x_t)$ can be rewritten as:

    \begin{equation} 
    \label{eq:parabola}
        \Delta J(\gamma) \approx \gamma \cdot \underbrace{(\Delta_{RF}^\top \nabla_x J)}_{\text{Linear Gain (Slope } b)}  -  \gamma^2 \cdot \underbrace{\frac{1}{2} C_{RF}}_{\text{Curvature Penalty (Coeff } a)} 
    \end{equation}

    Equation (\ref{eq:parabola}) describes a downward-opening parabola with respect to $\gamma$. To find the optimal merge ratio $\gamma^*$, we take the derivative with respect to $\gamma$ and set it to zero:
    
    \begin{equation}
        \frac{d}{d\gamma} \Delta J(\gamma) = (\Delta_{RF}^\top \nabla_x J) - \gamma C_{RF} = 0
    \end{equation}
    
    Yielding the optimal step size:
    
    \begin{equation}
        \gamma^* = \frac{\Delta_{RF}^\top \nabla_x J(x_t)}{C_{RF}} = \frac{\Delta_{RF}^\top \nabla J}{|\Delta_{RF}^\top \mathbf{H} \Delta_{RF}|}
    \end{equation}
\end{proof}

\begin{table*}[!ht]
\begin{center}

\renewcommand\arraystretch{1.2}
\setlength{\tabcolsep}{10pt}

\caption{To demonstrate the robustness of our method, we conducted repeated experiments on the Pick-a-Pic using FLUX-Lite with different random seeds. The results show that our approach consistently outperformed the standard method across varying random seeds, highlighting the robustness of RF-Sampling.}
\resizebox{1.\linewidth}{!}{\begin{tabular}{cccccc}
\hline
                          & Method      & PickScore(↑)                     & HPSv2(↑)                         & AES(↑)                            & ImageReward(↑)                    \\ \hline
                          & Standard    & 21.91                         & 30.12                         & 6.3224                         & 86.84                          \\
\multirow{-2}{*}{Round 1} & RF-Sampling & \CC22.05 & \CC31.16 & \CC6.5379 & \CC99.21  \\ \hline
                          & Standard    & 21.95                         & 30.33                         & 6.3473                         & 93.73                          \\
\multirow{-2}{*}{Round 2} & RF-Sampling & \CC22.04 & \CC30.82 & \CC6.5231 & \CC100.81 \\ \hline
                          & Standard    & 21.94                         & 30.20                         & 6.3608                         & 99.42                          \\
\multirow{-2}{*}{Round 3} & RF-Sampling & \CC21.99 & \CC30.63 & \CC6.5133 & \CC103.45 \\ \hline
                          & Standard    & 21.96                         & 30.23                         & 6.3365                         & 96.22                          \\
\multirow{-2}{*}{Round 4} & RF-Sampling & \CC22.02 & \CC30.83 & \CC6.5243 & \CC109.37 \\ \hline
                          & Standard    & 21.94 $\pm$ 0.02                         & 30.22 $\pm$ 0.08                        & 6.3418 $\pm$ 0.0163                        & 94.00 $\pm$ 5.43                          \\
\multirow{-2}{*}{Average} & RF-Sampling & \CC22.03 $\pm$ 0.03 & \CC30.86 $\pm$ 0.22 & \CC6.5247 $\pm$ 0.0101 & \CC103.21 $\pm$ 4.46 \\ \hline
\end{tabular}
}
\label{tab:random_seed}
\end{center}
\end{table*}
\begin{table*}[!ht]
\begin{center}

\renewcommand\arraystretch{1.2}
\setlength{\tabcolsep}{10pt}

\caption{Ablation on reflection component. We replace the full reflection step with a simple linear interpolation between embeddings~(using Eqn.~\ref{eq:mixed_embedding}) under two distinct mixing weights. The results show that both linear variants fail to improve over the standard baseline, performing identically across all metrics. This demonstrates that the model-driven reflection is essential, as simpler heuristics cannot achieve the performance gains of our full RF-Sampling.}
\resizebox{1.\linewidth}{!}{\begin{tabular}{ccccc}
\hline
Method             & PickScore(↑)                     & HPSv2(↑)                         & AES(↑)                            & ImageReward(↑)                    \\ \hline
Standard           & 21.99                         & 29.32                         & 5.9435                         & 85.13                          \\ \hline
High Embedding Mix~$(s = 9, \beta = 0.7)$ & 21.99                         & 29.32                         & 5.9435                         & 85.13                          \\
Low Embedding Mix~$(s = -1, \beta = 0.3)$  & 21.99                         & 29.32                         & 5.9435                         & 85.13                          \\ \hline
RF-Sampling        & \CC21.99 & \CC29.90 & \CC5.9981 & \CC101.50 \\ \hline
\end{tabular}
}
\label{tab:ablation_stage}
\end{center}
\end{table*}
\begin{table*}[!ht]
\begin{center}

\renewcommand\arraystretch{1.2}
\setlength{\tabcolsep}{10pt}

\caption{Comparison of our RF-Sampling with Best-of-N method. RF-Sampling achieves a better trade-off between performance and efficiency: it outperforms standard sampling in all metrics and is competitive with Best-of-N methods. While Best-of-5 achieves the highest performance, it requires more than double the time per image compared to RF-Sampling. RF-Sampling outperforms Best-of-3 in PickScore, AES and ImageReward with approximately 1.5 times faster. These results demonstrate the effectiveness of our method in achieving high performance with reduced computational cost.}
\resizebox{1.\linewidth}{!}{\begin{tabular}{cccccc}
\hline
Method                                                                                     & PickScore(↑)                     & HPSv2(↑)                         & AES(↑)                            & ImageReward(↑)                    & s/img($\downarrow$)                         \\ \hline
\begin{tabular}[c]{@{}c@{}}Standard~(28 steps)\end{tabular}                              & 21.99                         & 29.32                         & 5.9435                         & 85.13                          & 29.93                         \\
\begin{tabular}[c]{@{}c@{}}Standard~(28 $\times$ 3 = 84 steps)\end{tabular} & 21.96                         & 29.60                         & 5.9109                         & 89.87                          & 67.06                         \\ \hline
Best-of-5                                                                                  & 22.21                         & 30.58                         & 5.9849                         & 106.69                         & 154.17                        \\ \hline
Best-of-3                                                                                  & 21.94                         & \CC30.14                         & 5.9642                         & 100.40                         & 97.63                         \\
RF-Sampling                                                                                & \CC21.99 & 29.90 & \CC5.9981 & \CC101.50 & \CC65.04 \\ \hline

\end{tabular}
}
\label{tab:best_of_n}
\end{center}
\end{table*}
\begin{table*}[!ht]
\begin{center}

\renewcommand\arraystretch{1.2}
\setlength{\tabcolsep}{10pt}

\caption{Comparison under equivalent computational budget. To demonstrate the effectiveness of our method, we compare RF-Sampling against baselines using 84 steps (28×3), matching the total inference time. The results show that RF-Sampling almost outperforms all baseline methods across different metrics while maintaining comparable time per image, demonstrating its effectiveness under a fair computational setting.}
\resizebox{1.\linewidth}{!}{\begin{tabular}{cccccc}
\hline
Method                                                                          & PickScore(↑)                       & HPSv2(↑)                           & AES(↑)                              & ImageReward(↑)                      & s/img~($\downarrow$)                         \\ \hline
\begin{tabular}[c]{@{}c@{}}Standard~(28 steps)\end{tabular}                   & 21.99                         & 29.32                         & 5.9435                         & 85.13                          & 29.93                         \\
\begin{tabular}[c]{@{}c@{}}GI~(28 steps)\end{tabular}                         & 21.19                         & 24.63                         & 5.9534                         & 28.94                          & 31.33                         \\
\begin{tabular}[c]{@{}c@{}}CFG++~(28 steps)\end{tabular}                      & 21.79                         & 28.50                         & 5.8821                         & 85.17                          & 32.46                         \\
\begin{tabular}[c]{@{}c@{}}CFG-Zero*~(28 steps)\end{tabular}                  & 21.88                         & 29.37                         & 5.9536                         & 86.78                          & 28.91                         \\ \hline
\begin{tabular}[c]{@{}c@{}}Standard~(28 $\times$ 3 = 84 steps)\end{tabular}   & 21.96                         & 29.60                         & 5.9109                         & 89.87                          & 67.06                         \\
\begin{tabular}[c]{@{}c@{}}GI~(28 $\times$ 3 = 84 steps)\end{tabular}         & 21.25                         & 25.27                         & 5.9335                         & 28.16                          & 67.04                         \\
\begin{tabular}[c]{@{}c@{}}Z-Sampling~(28 $\times$ 3 = 84 steps)\end{tabular} & 21.73                         & 28.84                         & 5.9091                         & 89.03                          & \CC65.00 \\
\begin{tabular}[c]{@{}c@{}}CFG++~(28 $\times$ 3 = 84 steps)\end{tabular}      & 20.98                         & 27.02                         & 5.6144                         & 64.73                          & 68.07                         \\
\begin{tabular}[c]{@{}c@{}}CFG-Zero*~(28 $\times$ 3 = 84 steps)\end{tabular}  & \CC22.01 & 29.48                         & 5.8949                         & 97.22                          & 65.47                         \\ \hline
RF-Sampling                                                                     & 21.99                         & \CC29.90 & \CC5.9981 & \CC101.50 & 65.04                         \\ \hline

\end{tabular}
}
\label{tab:time_cost}
\end{center}
\end{table*}
\begin{table*}[!ht]
\begin{center}

\renewcommand\arraystretch{1.0}
\setlength{\tabcolsep}{10pt}

\caption{Quantitative comparisons with ~\cite{ma2025inferencetimescalingdiffusionmodels} on DrawBench. RF-Sampling requires only 150 NFEs, far fewer than the baseline methods~(2880 NFEs), yet still achieves the top results in both ImageReward and AES, demonstrating the dual advantages of our method in both efficiency and effectiveness.}
\resizebox{1.\linewidth}{!}{\begin{tabular}{c|ccccc}
\hline
                         & \multicolumn{5}{c}{Method}                                                                                                   \\ \cline{2-6} 
\multirow{-2}{*}{Metric} & \multicolumn{1}{c|}{Standard} & Aesthetic + Random & + ZO-2 & \multicolumn{1}{c|}{+ Path-2} & RF-Sampling                    \\ \hline
NFEs                     & \multicolumn{1}{c|}{50}       & 2880               & 2880   & \multicolumn{1}{c|}{2880}     & 50 $\times$ 3 = 150             \\ \hline
ImageReward              & \multicolumn{1}{c|}{99.73}    & 101.21             & 98.42  & \multicolumn{1}{c|}{97.13}    & \CC106.21 \\ \hline
\end{tabular}
}

\vspace{2mm}

\resizebox{1.\linewidth}{!}{\begin{tabular}{c|ccccc}
\hline
                         & \multicolumn{5}{c}{Method}                                                                                                   \\ \cline{2-6} 
\multirow{-2}{*}{Metric} & \multicolumn{1}{c|}{Standard} & CLIPScore + Random & + ZO-2 & \multicolumn{1}{c|}{+ Path-2} & RF-Sampling                    \\ \hline
NFEs                     & \multicolumn{1}{c|}{50}       & 2880               & 2880   & \multicolumn{1}{c|}{2880}     & 50 $\times$ 3 = 150             \\ \hline
AES              & \multicolumn{1}{c|}{6.1459}    & 6.0323             & 6.0512  & \multicolumn{1}{c|}{6.0452}    & \CC6.1866 \\ \hline

\end{tabular}
}

\vspace{2mm}
\resizebox{1.\linewidth}{!}{\begin{tabular}{c|ccccc}
\hline
                         & \multicolumn{5}{c}{Method}                                                                                                   \\ \cline{2-6} 
\multirow{-2}{*}{Metric} & \multicolumn{1}{c|}{Standard} & ImageReward + Random & + ZO-2 & \multicolumn{1}{c|}{+ Path-2} & RF-Sampling                    \\ \hline
NFEs                     & \multicolumn{1}{c|}{50}       & 2880               & 2880   & \multicolumn{1}{c|}{2880}     & 50 $\times$ 3 = 150             \\ \hline
AES              & \multicolumn{1}{c|}{6.1459}    & 6.1459           & 6.1265   & \multicolumn{1}{c|}{6.0945}    & \CC6.1966 \\ \hline
\end{tabular}
}

\label{tab:saining_scaling_draw}
\end{center}
\end{table*}
\begin{table*}[!ht]
\begin{center}

\renewcommand\arraystretch{1.2}
\setlength{\tabcolsep}{10pt}

\caption{Quantitative comparisons with \cite{ma2025inferencetimescalingdiffusionmodels} on T2I-CompBench. RF-Sampling requires only 150 NFEs, far fewer than the baseline methods~(1920 NFEs), yet almost achieves the top results across different dimensions, demonstrating the dual advantages of our method in both efficiency and effectiveness.}
\resizebox{1.\linewidth}{!}{\begin{tabular}{cccccccc}
\hline
Method                                                                           & Color(↑)                          & Shape(↑)                          & Texture(↑)                        & Spatial(↑)                        & Numeracy(↑)                       & Complex(↑)                        & Overall(↑)                        \\ \hline
Standard                                                                         & 0.7535                         & 0.5018                         & 0.6167                         & 0.2783                         & 0.6052                         & 0.3706                         & 0.5210                         \\
\begin{tabular}[c]{@{}c@{}}Aesthetic + Random~(1920 NFEs)\end{tabular}         & 0.7518                         & 0.5219                         & 0.5926                         & \CC0.2893 & 0.6059                         & 0.3572                         & 0.5199                         \\ \hline
\begin{tabular}[c]{@{}c@{}}RF-Sampling~(50 $\times$ 3 = 150 NFEs)\end{tabular} & \CC0.7761 & \CC0.5323 & \CC0.6422 & 0.2687                         & \CC0.6082 & \CC0.3733 & \CC0.5335 \\ \hline

\end{tabular}
}
\label{tab:saining_scaling_t2i}
\end{center}
\end{table*}
\begin{table*}[!ht]
\begin{center}

\renewcommand\arraystretch{1.2}
\setlength{\tabcolsep}{10pt}

\caption{Detailed breakdown of Fig.~\ref{figure:time-scaling}, including step counts~(NFEs) and wall time. As shown in the table below, RF-Sampling outperforms standard sampling with the same time consumption and significantly enhances the performance of FLUX-Lite and FLUX-Dev. With the increase of inference time, RF-Sampling consistently performs well, validating the scalability of our method.}
\resizebox{1.\linewidth}{!}{\begin{tabular}{cccccc}
\hline
Model                       & Method                                                                                 & NFEs                     & HPSv2(↑)                      & AES(↑)                         & s/img~($\downarrow$)  \\ \hline
                            &                                                                                        & 28                       & 30.12                         & 6.3224                         & 34.63  \\
                            &                                                                                        & 50                       & 30.39                         & 6.3045                         & 46.60  \\
                            & \multirow{-3}{*}{Standard}                                                             & 75                       & 30.46                         & 6.2864                         & 60.61  \\ \cline{2-6} 
                            &                                                                                        & 7 $\times$ 5 + 21 = 56   & 30.84                         & 6.4397                         & 49.63  \\
                            &                                                                                        & 14 $\times$ 5 + 14 = 84  & 30.98                         & 6.4736                         & 64.57  \\
                            &                                                                                        & 21 $\times$ 5 + 7 = 112  & 31.04                         & 6.5032                         & 76.84  \\
\multirow{-7}{*}{FLUX-Lite} & \multirow{-4}{*}{\begin{tabular}[c]{@{}c@{}}RF-Sampling\\ ($\alpha = 2$)\end{tabular}} & 28 $\times$ 5 = 140      & \CC31.16 & \CC6.5379 & 95.26  \\ \hline
                            &                                                                                        & 50                       & 30.49                         & 6.2464                         & 59.09  \\
                            &                                                                                        & 75                       & 30.54                         & 6.2170                         & 75.85  \\
                            & \multirow{-3}{*}{Standard}                                                             & 100                      & 30.60                         & 6.1869                         & 91.48  \\ \cline{2-6} 
                            &                                                                                        & 10 $\times$ 3 + 40 = 70  & 30.58                         & 6.2505                         & 71.87  \\
                            &                                                                                        & 20 $\times$ 3 + 30 = 90  & 30.66                         & 6.2639                         & 86.07  \\
                            &                                                                                        & 30 $\times$ 3 + 20 = 110 & 30.70                         & 6.2893                         & 100.03 \\
                            &                                                                                        & 40 $\times$ 3 + 10 = 130 & 30.79                         & 6.2917                         & 114.30 \\
\multirow{-8}{*}{FLUX-Dev}  & \multirow{-5}{*}{\begin{tabular}[c]{@{}c@{}}RF-Sampling\\ ($\alpha = 1$)\end{tabular}}   & 50 $\times$ 3 = 150      & \CC31.06 & \CC6.3113 & 127.95 \\ \hline

\end{tabular}
}
\label{tab:fig2_breakdown}
\end{center}
\end{table*}

\begin{figure*}[!ht]
\centering
\includegraphics[width=1.\textwidth,trim={0cm 0cm 0cm 0cm},clip]{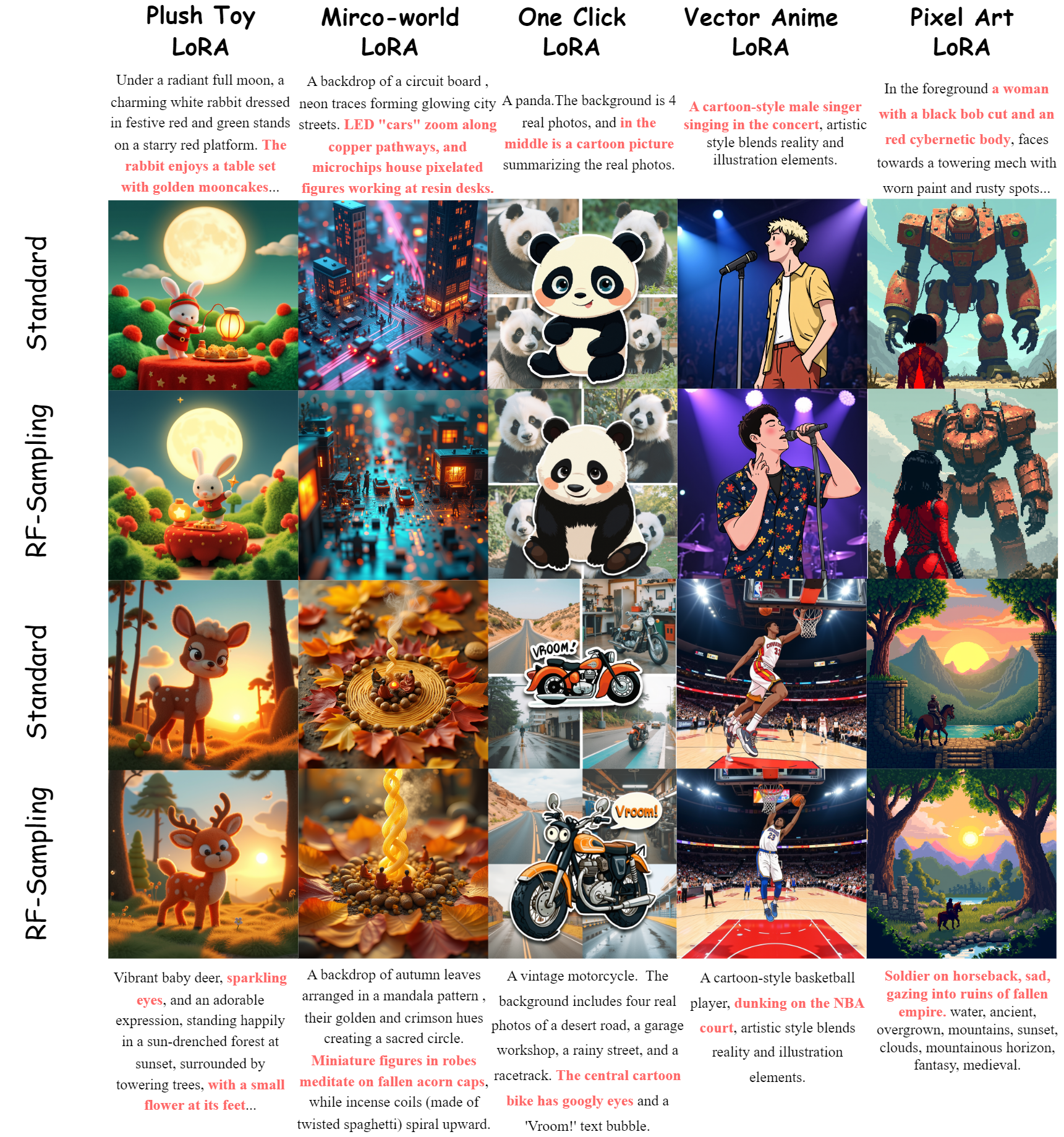}
\caption{\small We combine our proposed methods with existing LoRAs in FLUX community. Our RF-Sampling can be directly applied to the corresponding downstream tasks, validating the generalizability of our method.}
\label{figure:lora-appendix}
\end{figure*}

\begin{figure*}[!ht]
\centering
\includegraphics[width=1.\textwidth,trim={0cm 0cm 0cm 0cm},clip]{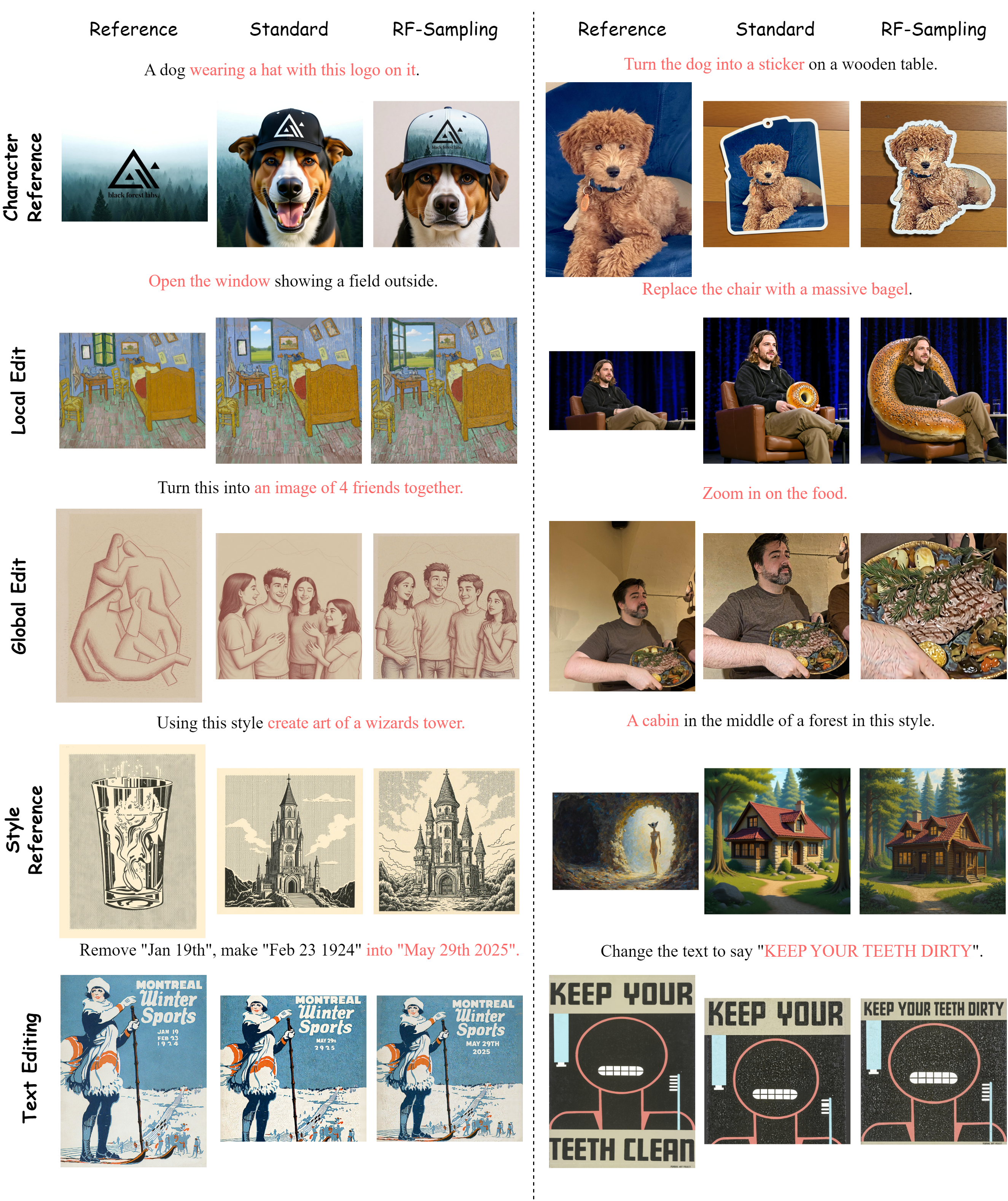}
\caption{\small We extend our proposed methods to image editing tasks on FLUX-Kontext. Our RF-Sampling can be directly applied to the corresponding downstream tasks, validating the effectiveness of our method.}
\label{figure:kontext}
\end{figure*}

\begin{figure*}[!ht]
\centering
\includegraphics[width=1.\textwidth,trim={0cm 0cm 0cm 0cm},clip]{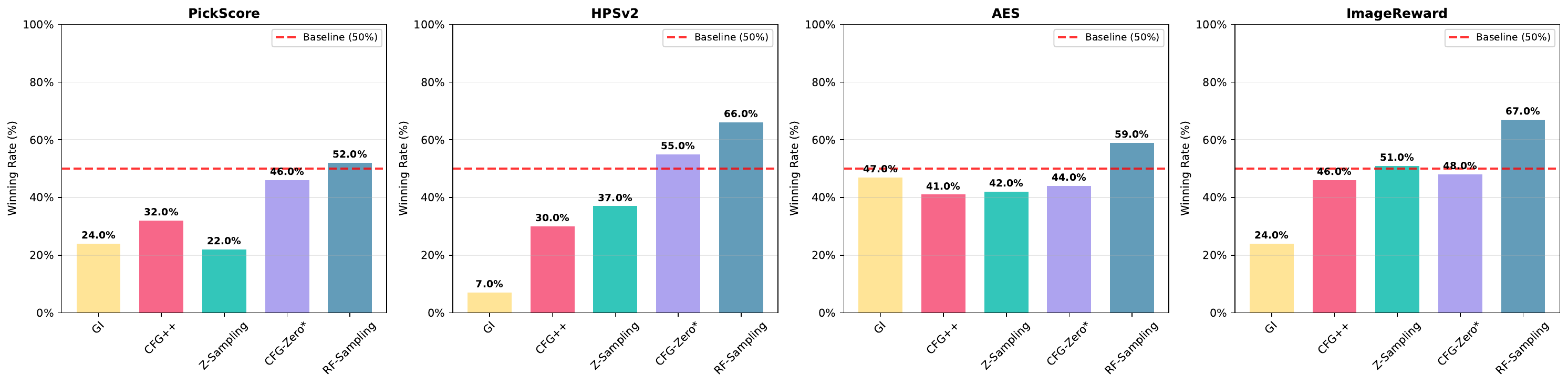}
\caption{The winning rate of RF-Sampling over other methods on SD3.5 on Pick-a-Pic dataset. The standard sampling~(baseline) winning rate defaults to 50\%}
\label{figure:winning-rate-sd3-pick}
\end{figure*}

\begin{figure*}[!ht]
\centering
\includegraphics[width=1.\textwidth,trim={0cm 0cm 0cm 0cm},clip]{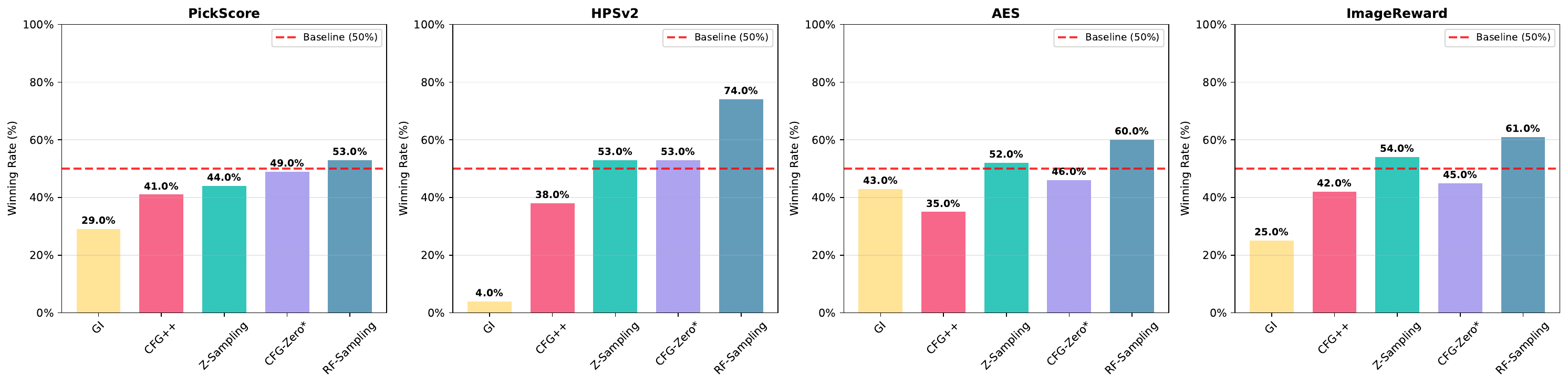}
\caption{The winning rate of RF-Sampling over other methods on SD3.5 on DrawBench dataset. The standard sampling~(baseline) winning rate defaults to 50\%.}
\label{figure:winning-rate-sd3-draw}
\end{figure*}

\begin{figure*}[!ht]
\centering
\includegraphics[width=1.\textwidth,trim={0cm 0cm 0cm 0cm},clip]{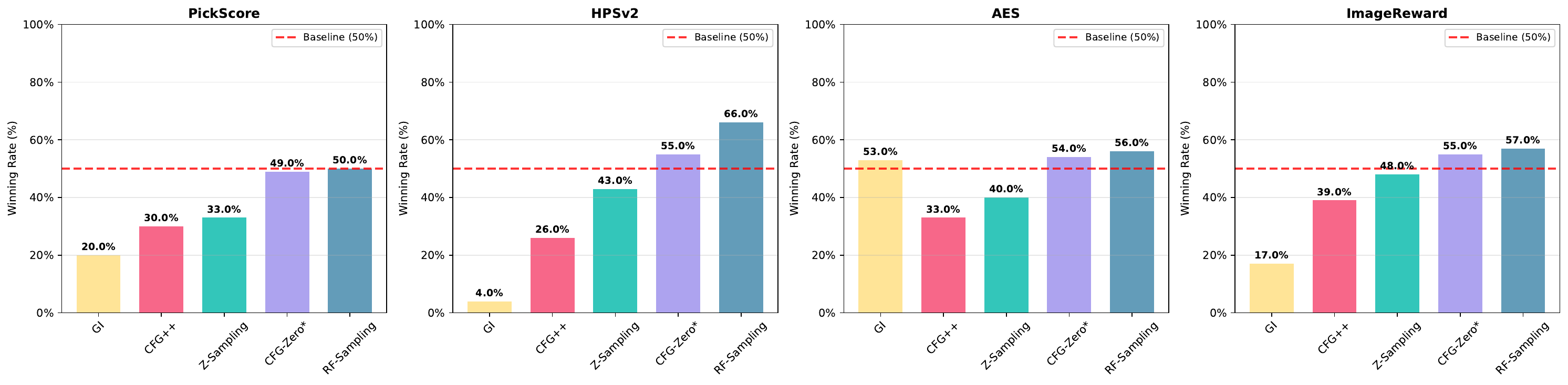}
\caption{The winning rate of RF-Sampling over other methods on SD3.5 on the animation subset of HPD v2 dataset. The standard sampling~(baseline) winning rate defaults to 50\%.}
\label{figure:winning-rate-sd3-anime}
\end{figure*}

\begin{figure*}[!ht]
\centering
\includegraphics[width=1.\textwidth,trim={0cm 0cm 0cm 0cm},clip]{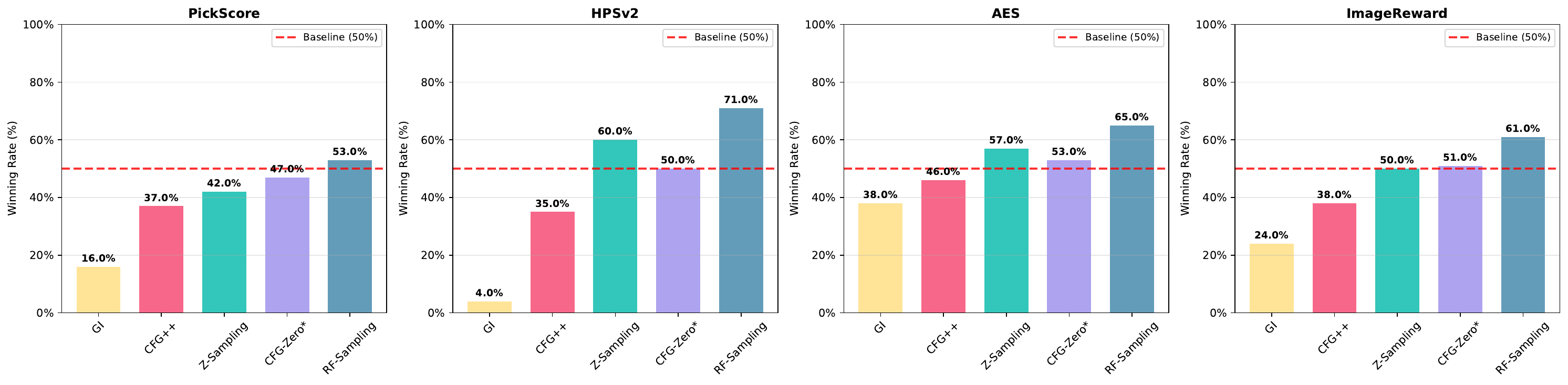}
\caption{The winning rate of RF-Sampling over other methods on SD3.5 on the photo subset of HPD v2 dataset. The standard sampling~(baseline) winning rate defaults to 50\%.}
\label{figure:winning-rate-sd3-photo}
\end{figure*}

\begin{figure*}[!ht]
\centering
\includegraphics[width=1.\textwidth,trim={0cm 0cm 0cm 0cm},clip]{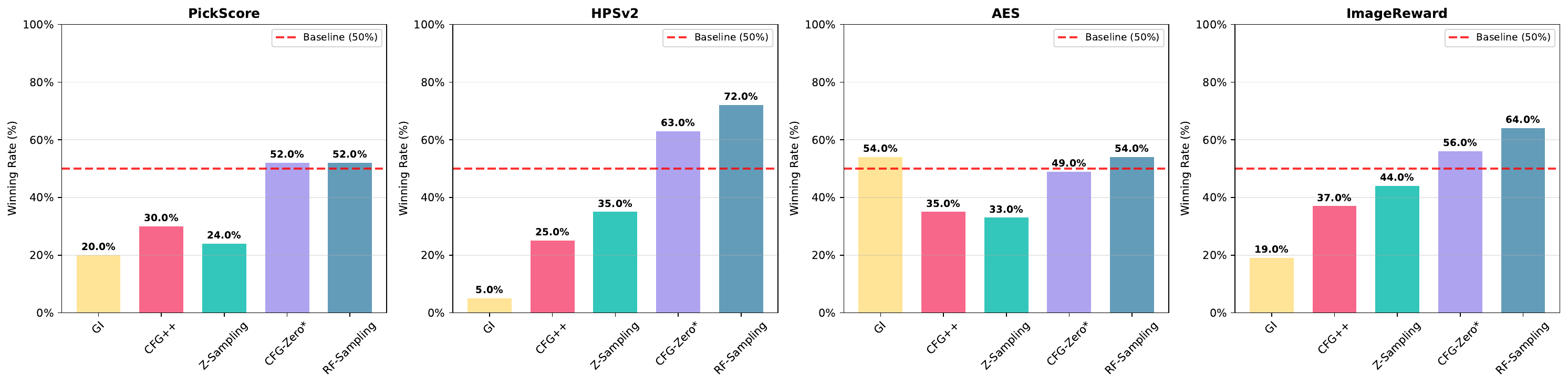}
\caption{The winning rate of RF-Sampling over other methods on SD3.5 on the concept-art subset of HPD v2 dataset. The standard sampling~(baseline) winning rate defaults to 50\%.}
\label{figure:winning-rate-sd3-concept}
\end{figure*}

\begin{figure*}[!ht]
\centering
\includegraphics[width=1.\textwidth,trim={0cm 0cm 0cm 0cm},clip]{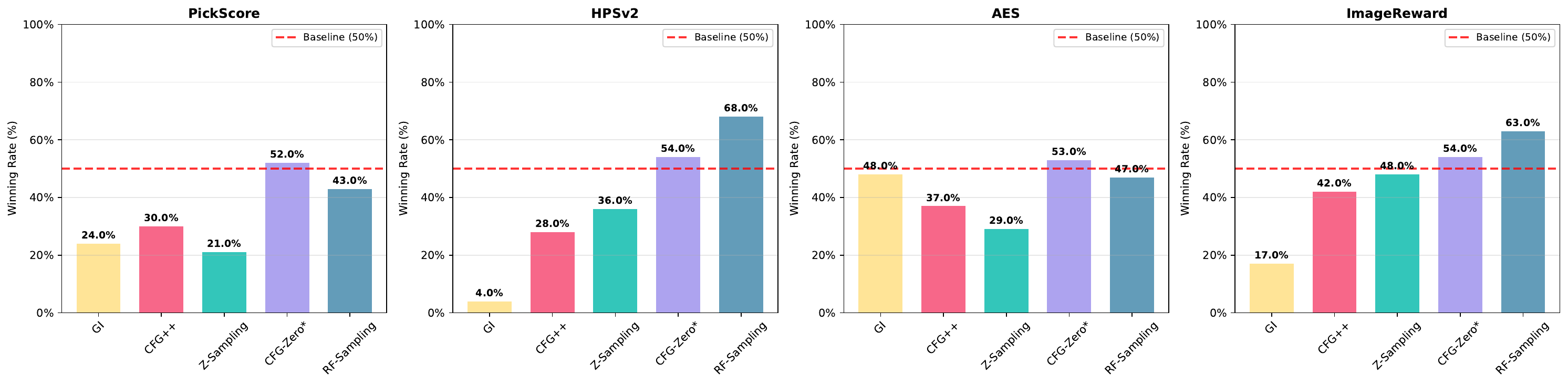}
\caption{The winning rate of RF-Sampling over other methods on SD3.5 on the painting subset of HPD v2 dataset. The standard sampling~(baseline) winning rate defaults to 50\%..}
\label{figure:winning-rate-sd3-paint}
\end{figure*}

\begin{figure*}[!ht]
\centering
\includegraphics[width=1.\textwidth,trim={0cm 0cm 0cm 0cm},clip]{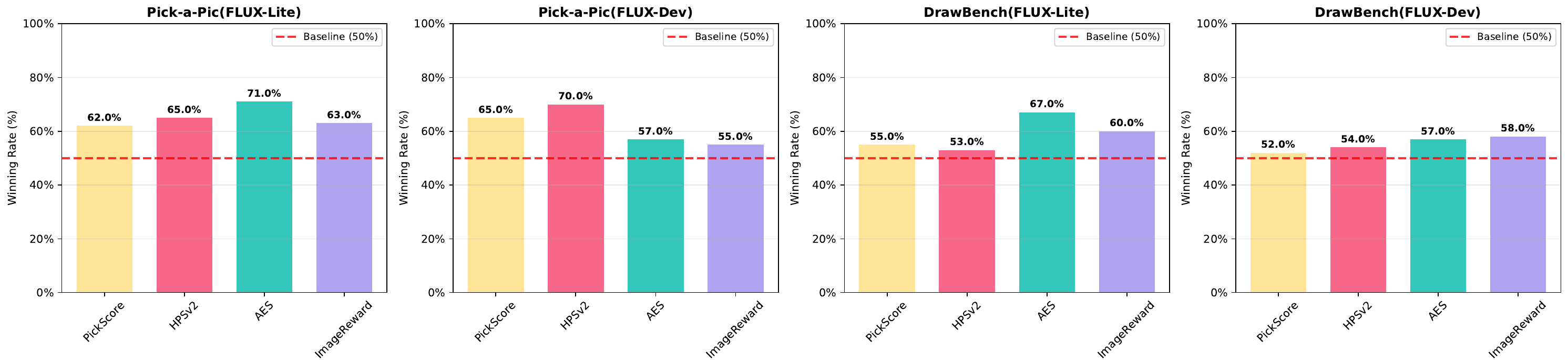}
\caption{The winning rate of RF-Sampling over the standard one on FLUX-Lite and FLUX-Dev on Pick-a-Pic and DrawBench datasets. The standard sampling~(baseline) winning rate defaults to 50\%.}
\label{figure:winning-rate-pick-draw}
\end{figure*}

\begin{figure*}[!ht]
\centering
\includegraphics[width=1.\textwidth,trim={0cm 0cm 0cm 0cm},clip]{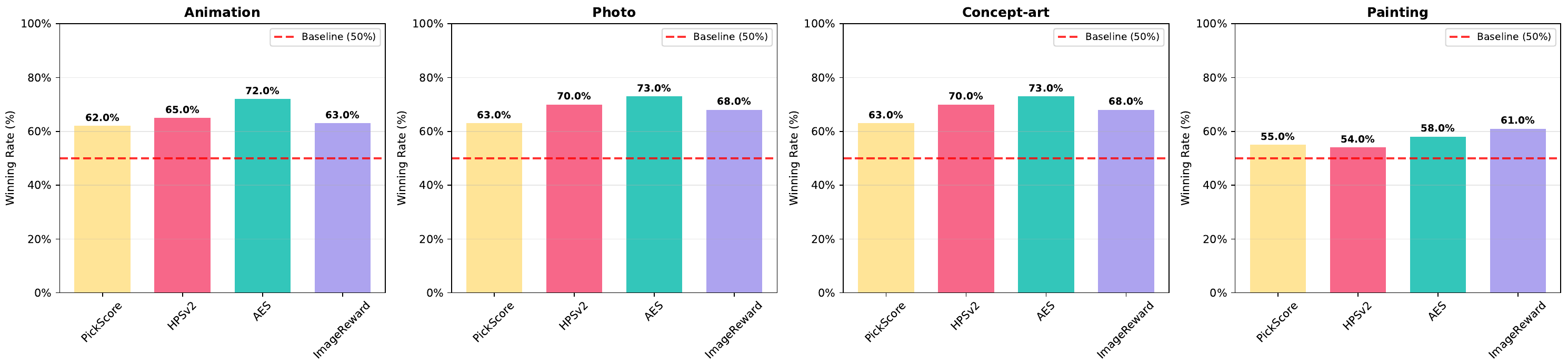}
\caption{The winning rate of RF-Sampling over the standard one on FLUX-Lite on the 4 subsets of HPD v2 datasets. The standard sampling~(baseline) winning rate defaults to 50\%.}
\label{figure:winning-rate-lite-hpd}
\end{figure*}

\begin{figure*}[!ht]
\centering
\includegraphics[width=1.\textwidth,trim={0cm 0cm 0cm 0cm},clip]{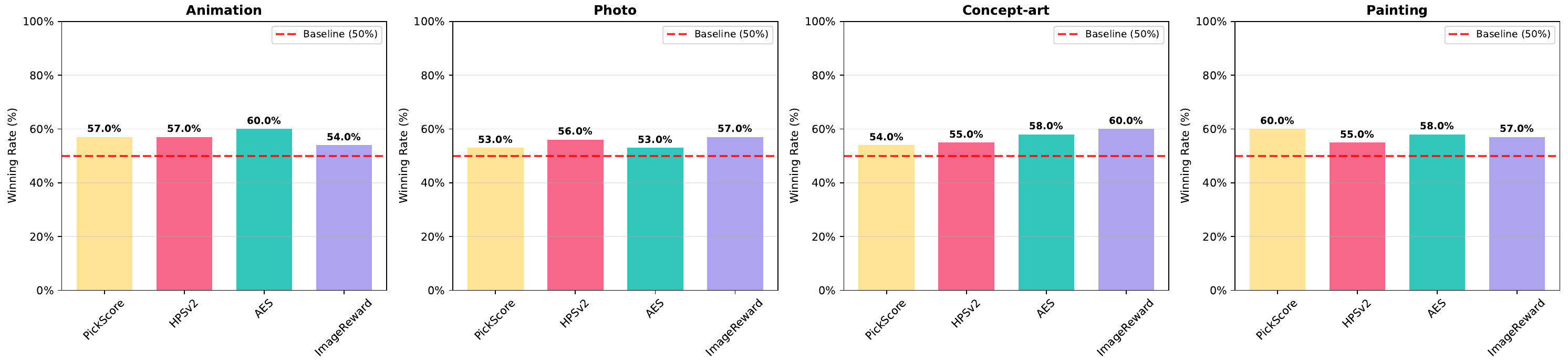}
\caption{The winning rate of RF-Sampling over the standard one on FLUX-Dev on the 4 subsets of HPD v2 datasets. The standard sampling~(baseline) winning rate defaults to 50\%.}
\label{figure:winning-rate-dev-hpd}
\end{figure*}

\subsection{More Visualizations}
\label{sec:more_vis}
We provide more visualizations of synthesized images on FLUX-Dev and FLUX-Lite, across HPD v2, Pick-a-Pic, DrawBench and GenEval datasets as shown in Fig.~\ref{figure:lite-anime}, \ref{figure:lite-photo}, \ref{figure:lite-paint}, \ref{figure:lite-concept}, \ref{figure:lite-geneval}, \ref{figure:lite-draw-pick}, \ref{figure:dev-anime}, \ref{figure:dev-photo}, \ref{figure:dev-paint}, \ref{figure:dev-concept}, \ref{figure:dev-geneval}, and \ref{figure:dev-draw-pick}.

\begin{figure*}[!ht]
\centering
\includegraphics[width=1.\textwidth,trim={0cm 0cm 0cm 0cm},clip]{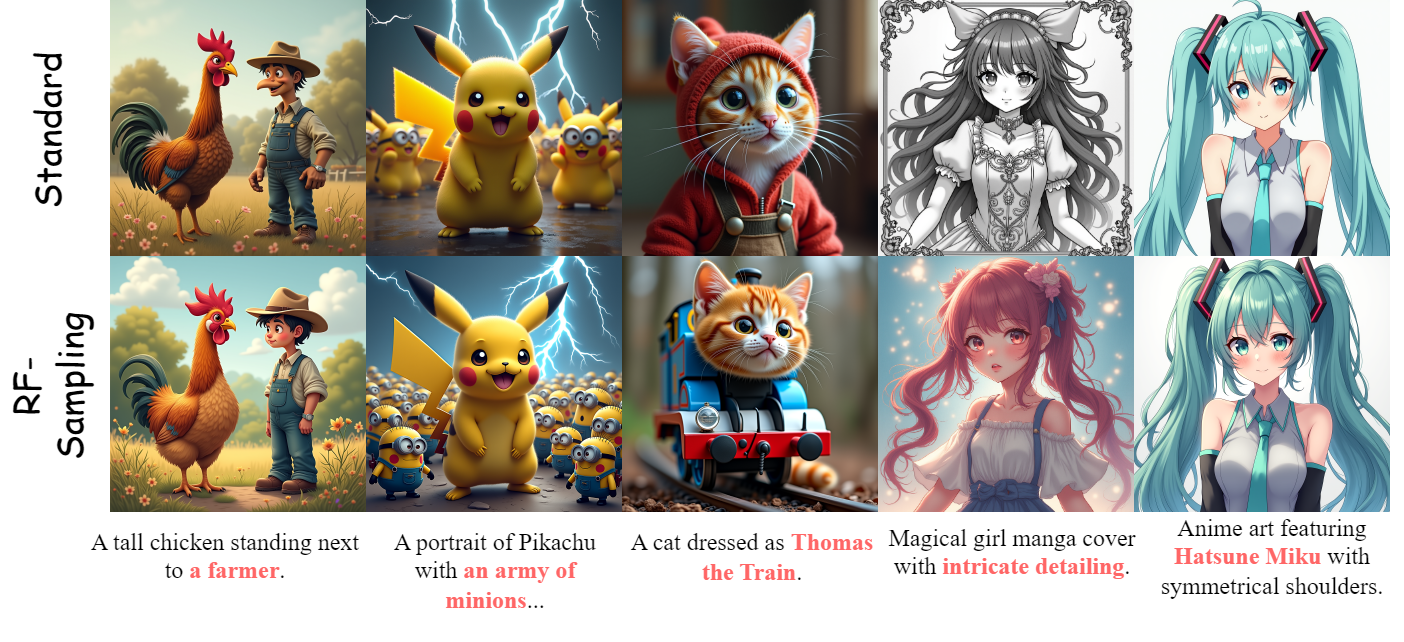}
\caption{\small Synthesized images of FLUX-Lite on anime subset of HPD v2.}
\label{figure:lite-anime}
\end{figure*}

\begin{figure*}[!ht]
\centering
\includegraphics[width=1.\textwidth,trim={0cm 0cm 0cm 0cm},clip]{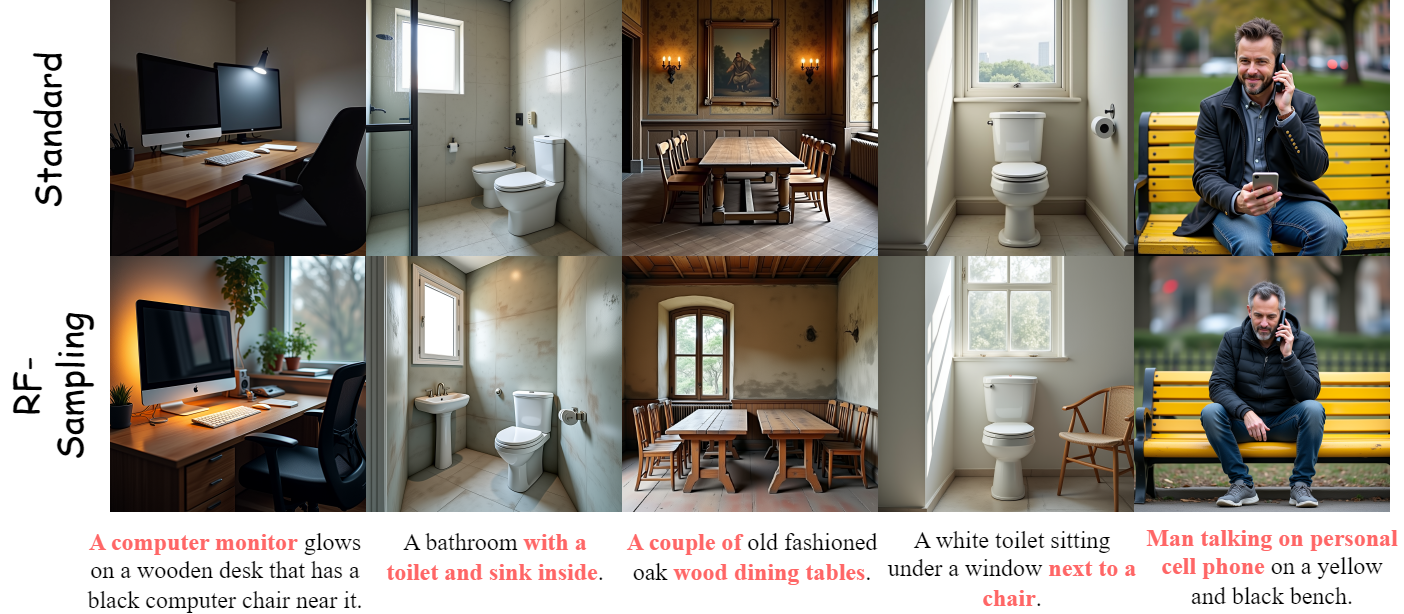}
\caption{\small Synthesized images of FLUX-Lite on photography subset of HPD v2.}
\label{figure:lite-photo}
\end{figure*}

\begin{figure*}[!ht]
\centering
\includegraphics[width=1.\textwidth,trim={0cm 0cm 0cm 0cm},clip]{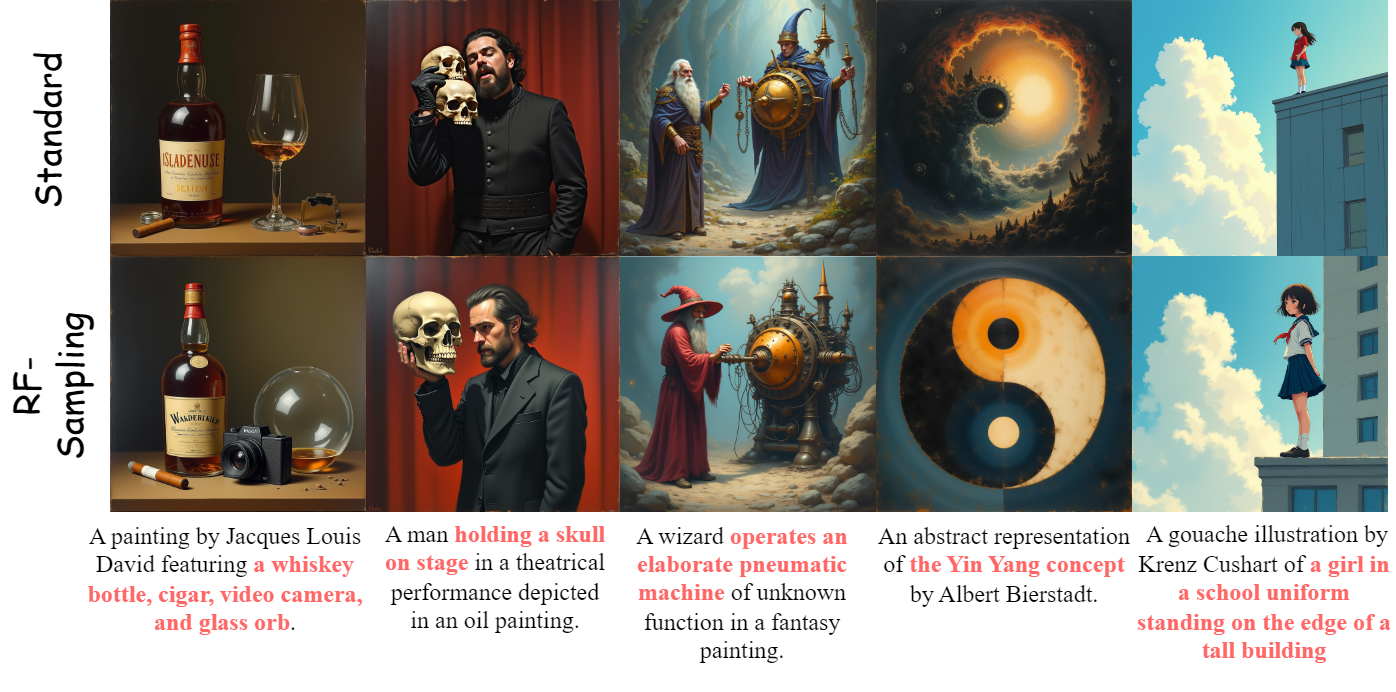}
\caption{\small Synthesized images of FLUX-Lite on painting subset of HPD v2.}
\label{figure:lite-paint}
\end{figure*}

\begin{figure*}[!ht]
\centering
\includegraphics[width=1.\textwidth,trim={0cm 0cm 0cm 0cm},clip]{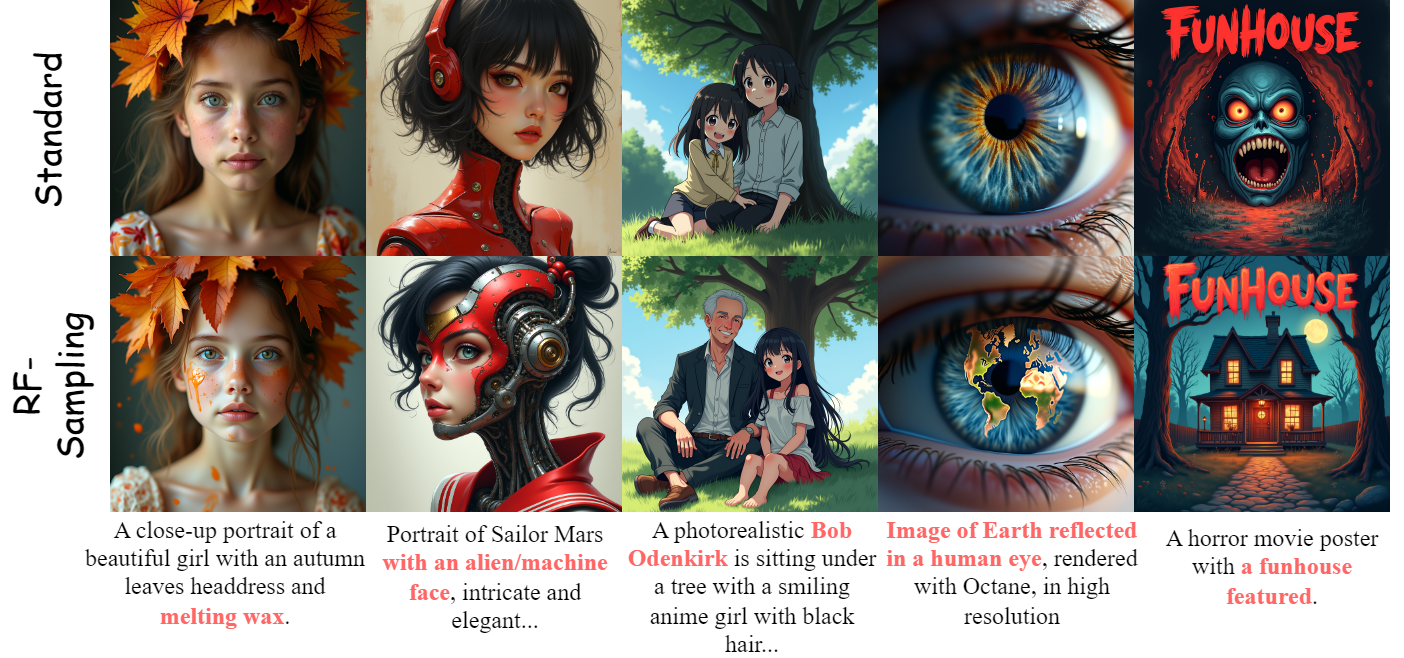}
\caption{\small Synthesized images of FLUX-Lite on concept-art subset of HPD v2.}
\label{figure:lite-concept}
\end{figure*}

\begin{figure*}[!ht]
\centering
\includegraphics[width=1.\textwidth,trim={0cm 0cm 0cm 0cm},clip]{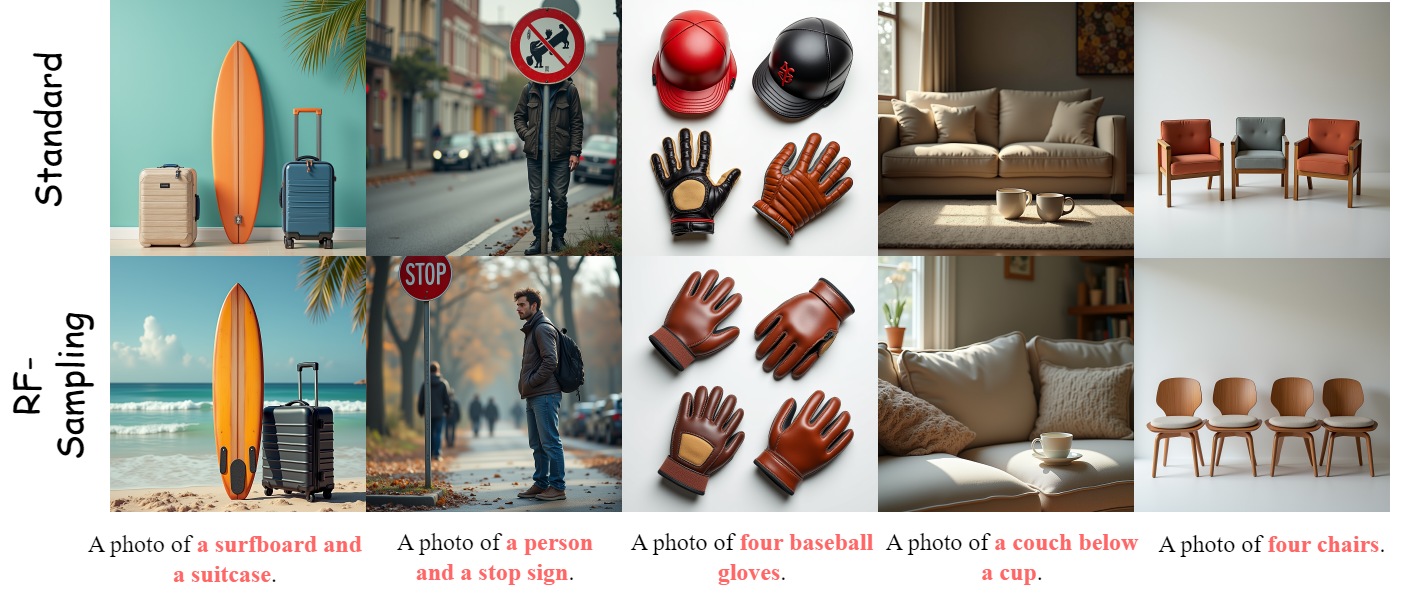}
\caption{\small Synthesized images of FLUX-Lite on GenEval.}
\label{figure:lite-geneval}
\end{figure*}

\begin{figure*}[!ht]
\centering
\includegraphics[width=1.\textwidth,trim={0cm 0cm 0cm 0cm},clip]{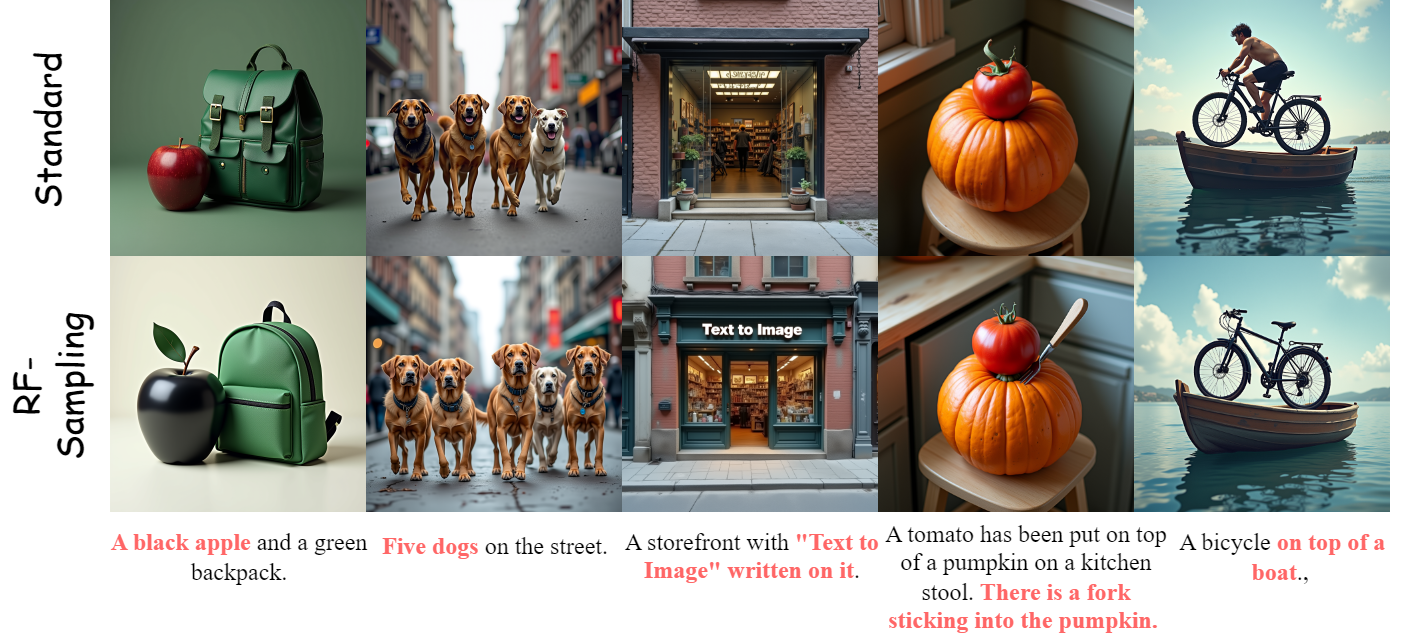}
\caption{\small Synthesized images of FLUX-Lite on Pick-a-Pic and DrawBench.}
\label{figure:lite-draw-pick}
\end{figure*}

\begin{figure*}[!ht]
\centering
\includegraphics[width=1.\textwidth,trim={0cm 0cm 0cm 0cm},clip]{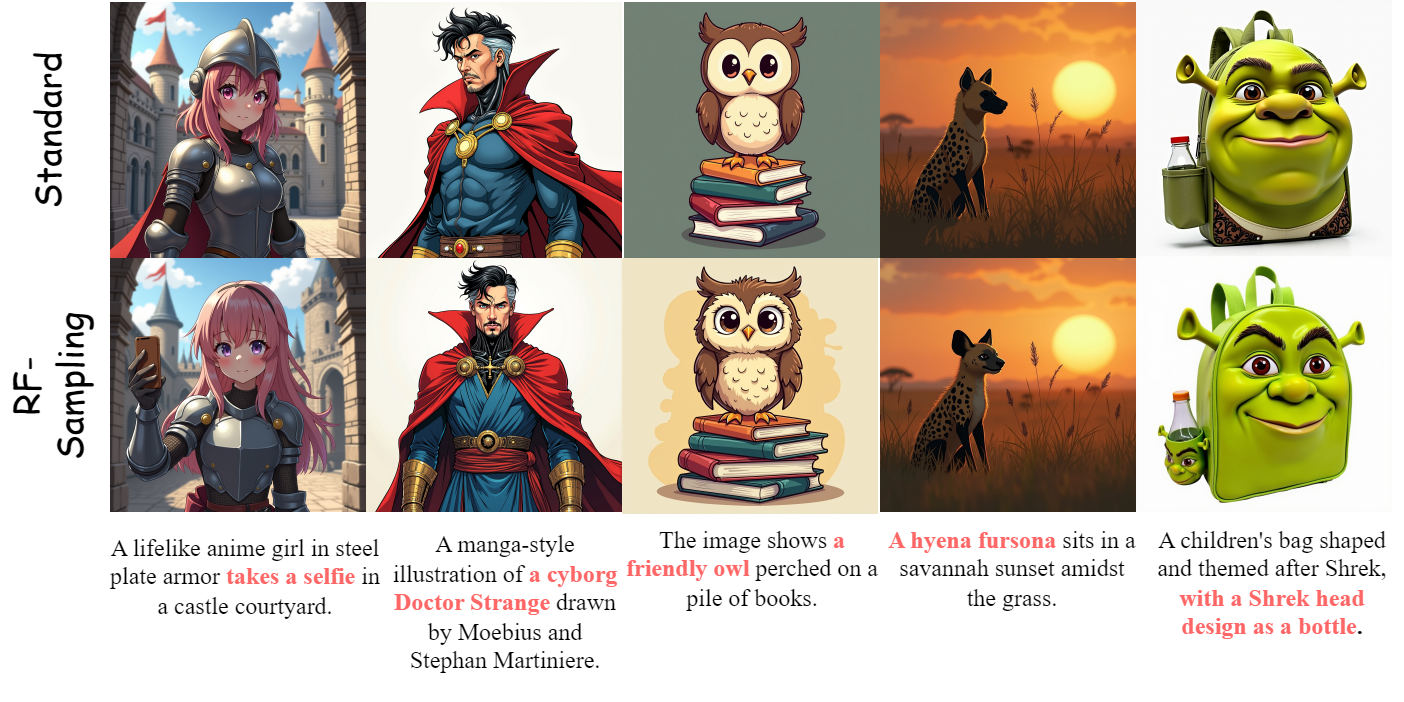}
\caption{\small Synthesized images of FLUX-Dev on anime subset of HPD v2.}
\label{figure:dev-anime}
\end{figure*}

\begin{figure*}[!ht]
\centering
\includegraphics[width=1.\textwidth,trim={0cm 0cm 0cm 0cm},clip]{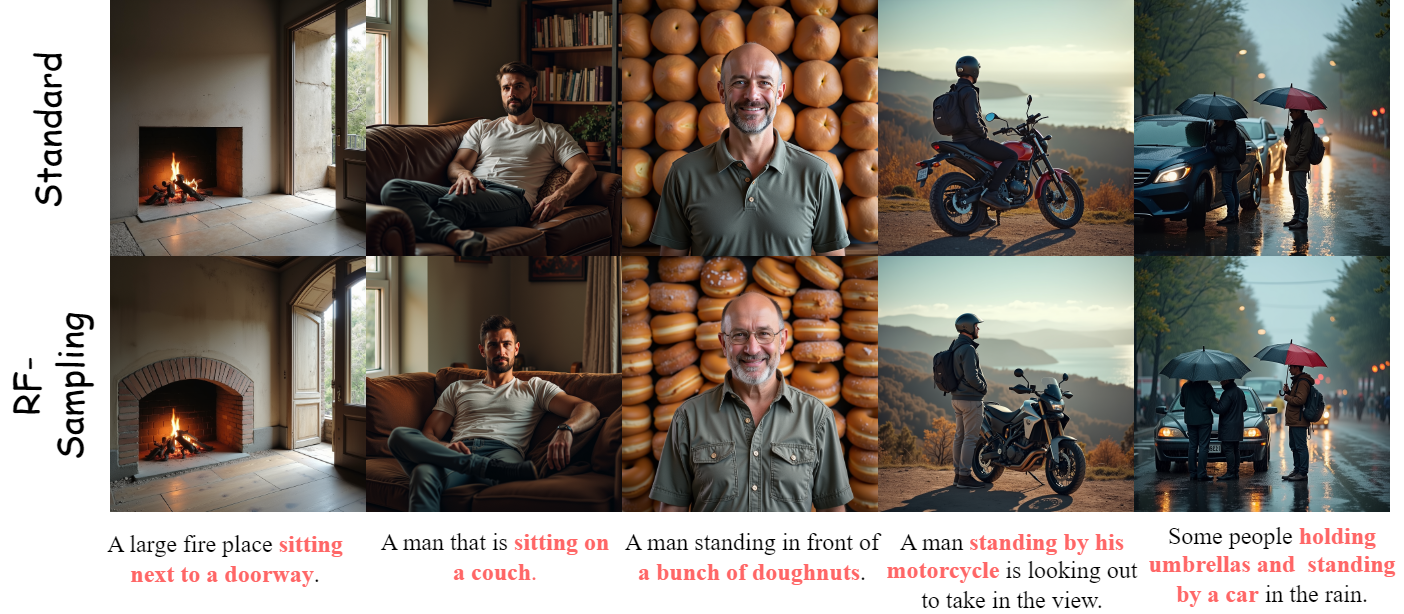}
\caption{\small Synthesized images of FLUX-Dev on photography subset of HPD v2.}
\label{figure:dev-photo}
\end{figure*}

\begin{figure*}[!ht]
\centering
\includegraphics[width=1.\textwidth,trim={0cm 0cm 0cm 0cm},clip]{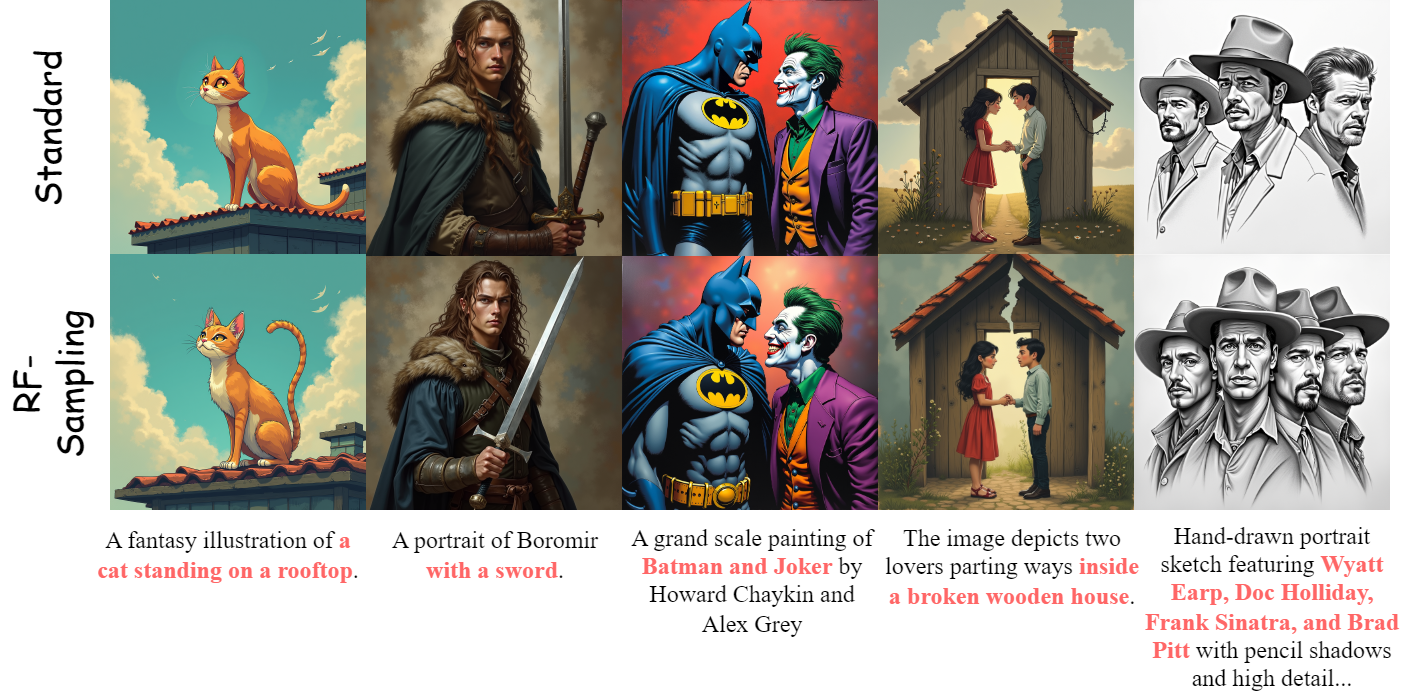}
\caption{\small Synthesized images of FLUX-Dev on painting subset of HPD v2.}
\label{figure:dev-paint}
\end{figure*}

\begin{figure*}[!ht]
\centering
\includegraphics[width=1.\textwidth,trim={0cm 0cm 0cm 0cm},clip]{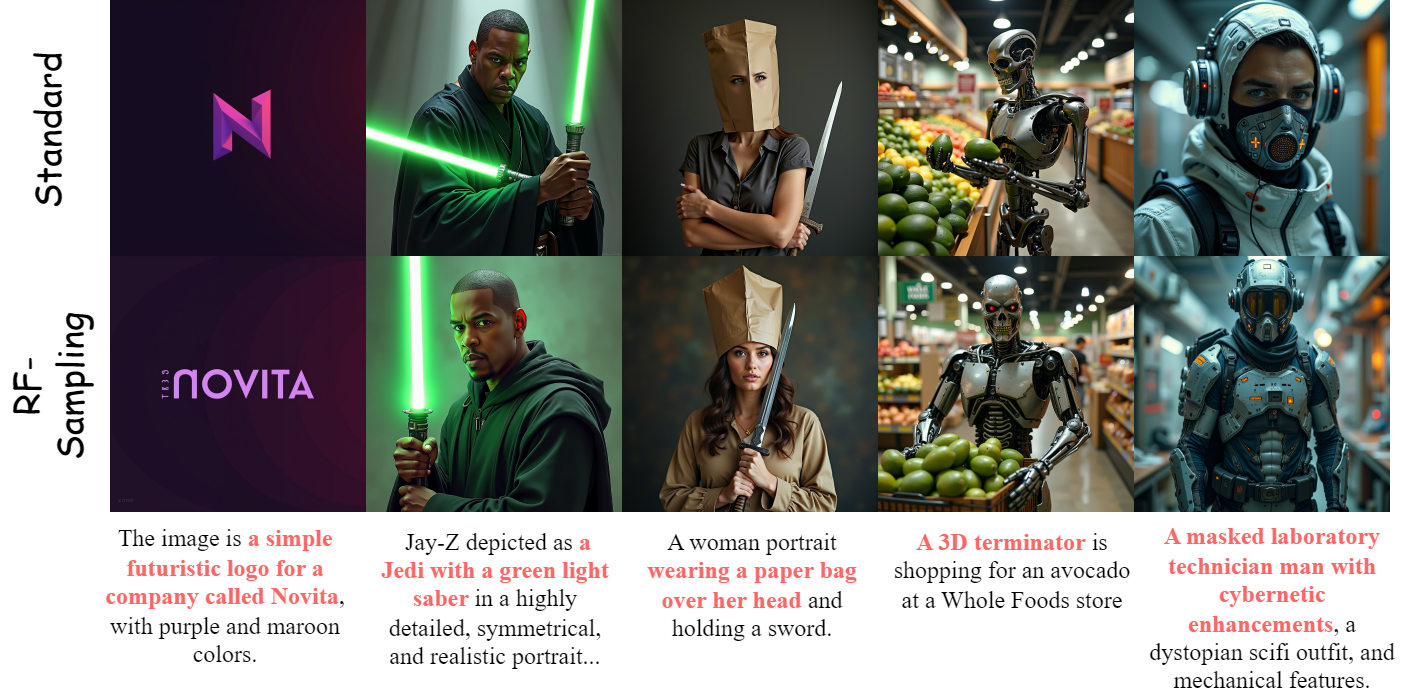}
\caption{\small Synthesized images of FLUX-Dev on concept-art subset of HPD v2.}
\label{figure:dev-concept}
\end{figure*}

\begin{figure*}[!ht]
\centering
\includegraphics[width=1.\textwidth,trim={0cm 0cm 0cm 0cm},clip]{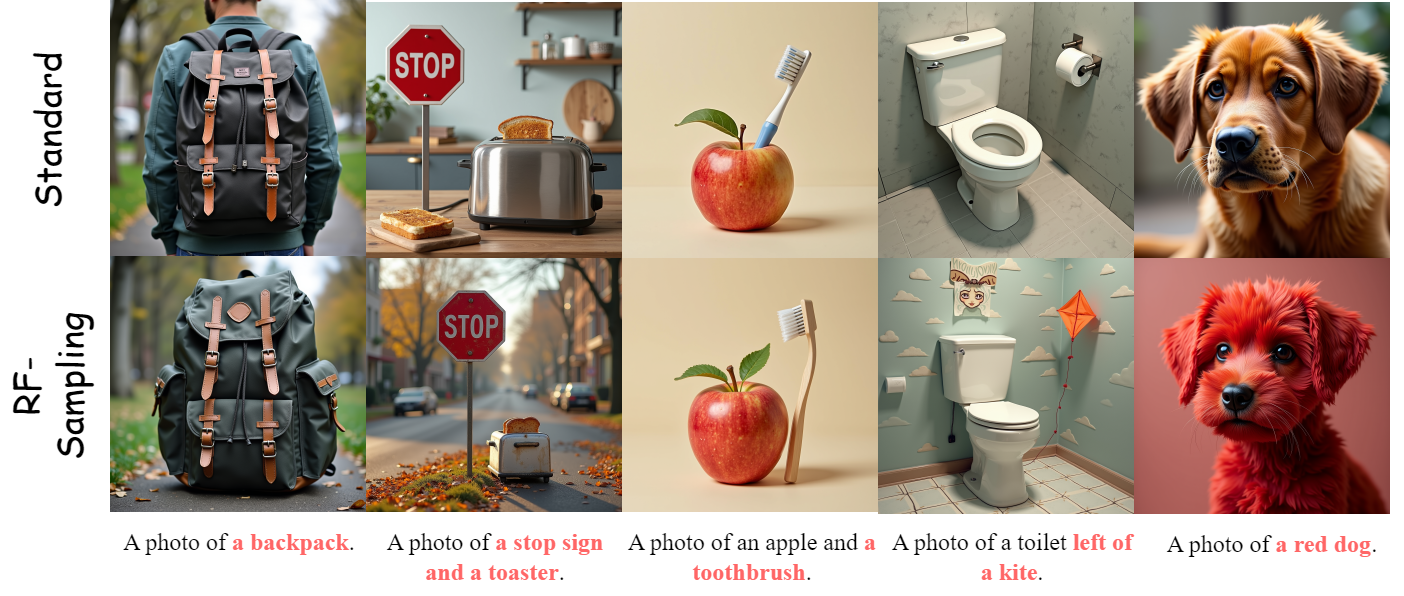}
\caption{\small Synthesized images of FLUX-Dev on GenEval.}
\label{figure:dev-geneval}
\end{figure*}

\begin{figure*}[!ht]
\centering
\includegraphics[width=1.\textwidth,trim={0cm 0cm 0cm 0cm},clip]{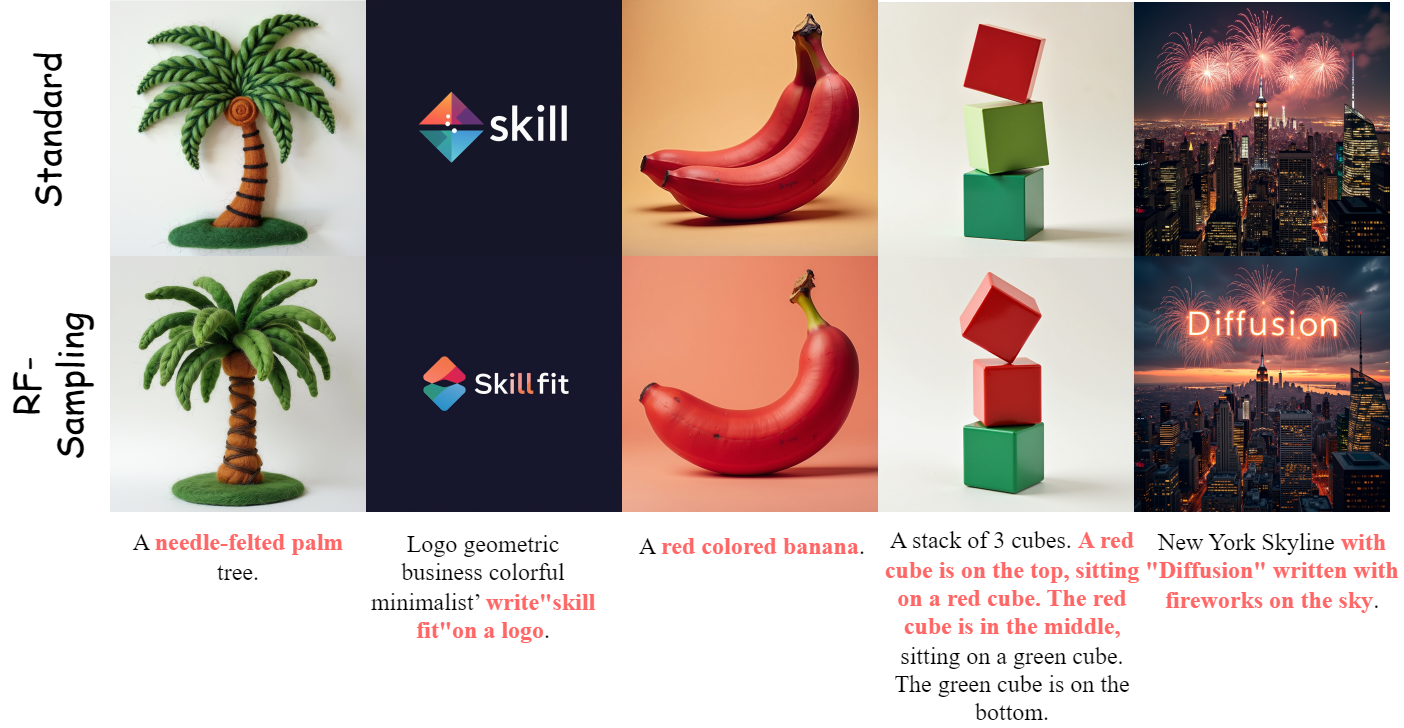}
\caption{\small Synthesized images of FLUX-Dev on Pick-a-Pic and DrawBench.}
\label{figure:dev-draw-pick}
\end{figure*}

\end{document}